\documentclass[letterpaper]{article} % DO NOT CHANGE THIS
\usepackage{aaai25}  % DO NOT CHANGE THIS
\usepackage{times}  % DO NOT CHANGE THIS
\usepackage{helvet}  % DO NOT CHANGE THIS
\usepackage{courier}  % DO NOT CHANGE THIS
\usepackage[hyphens]{url}  % DO NOT CHANGE THIS
\usepackage{graphicx} % DO NOT CHANGE THIS
\usepackage{natbib}  % DO NOT CHANGE THIS AND DO NOT ADD ANY OPTIONS TO IT
\usepackage{caption} % DO NOT CHANGE THIS AND DO NOT ADD ANY OPTIONS TO IT
\usepackage{algorithm}
\usepackage{algorithmic}

\usepackage{newfloat}
\usepackage{listings}
\DeclareCaptionStyle{ruled}{labelfont=normalfont,labelsep=colon,strut=off} % DO NOT CHANGE THIS
\floatstyle{ruled}
\newfloat{listing}{tb}{lst}{}
\floatname{listing}{Listing}
\usepackage{latexsym}
\usepackage{amssymb}
\usepackage{amsmath}
\usepackage{amsthm}
\usepackage{booktabs}
\usepackage{enumitem}
\usepackage{color}

\usepackage{dsfont}
\usepackage{mathtools}
\usepackage{bm}
\usepackage{tikz}
\usepackage{scalerel}
\usepackage{ifthen}
\usepackage{comment}

\newtheorem{theorem}{Theorem}
\newtheorem{lemma}[theorem]{Lemma}
\newtheorem{corollary}[theorem]{Corollary}
\newtheorem{proposition}[theorem]{Proposition}

\newtheorem{definition}{Definition}

\newtheorem{assumption}{Assumption}
\newtheorem{remark}{Remark}

\newcommand{\colme}{ColME}

\newcommand{\belief}{B-ColME}
\newcommand{\consensus}{C-ColME}

\newcommand{\bigO}{\mathcal O}
\newcommand{\x}{\mathbf x}
\newcommand{\xc}{\prescript{}{c}{\x}}
\newcommand{\xav}{\bar{\x}}
\newcommand{\xavc}{{\prescript{}{c}\xav}}
\newcommand{\bb}{\mathbf b}
\newcommand{\mm}{\bm{\mu}}
\newcommand{\mmc}{\prescript{}{c}{\mm}}
\newcommand{\y}{\hat{\mm}}
\newcommand{\yc}{\prescript{}{c}{\y}}
\newcommand{\z}{\mathbf z}
\newcommand{\lambdac}{{\lambda_{2,c}}}
\newcommand{\Pc}{\prescript{}{c}{P}}
\newcommand{\Wc}{\prescript{}{c}{W}}

\newcommand*\diff{\mathop{}\!\mathrm{d}}
\newcommand{\Ei}[1]{\mathrm{Ei}\!\left(#1\right)}
\DeclareMathOperator*{\E}{\mathbb{E}}
\DeclareMathOperator*{\Prob}{\mathbb{P}}

\title{Scalable Decentralized Algorithms for Online Personalized Mean Estimation}
\author {
    Franco Galante\textsuperscript{\rm 1},
    Giovanni Neglia\textsuperscript{\rm 2},
    Emilio Leonardi\textsuperscript{\rm 1}
}
\affiliations {
    \textsuperscript{\rm 1}Politecnico di Torino\\
    \textsuperscript{\rm 2}INRIA\\
    franco.galante@polito.it, giovanni.neglia@inria.fr, emilio.leonardi@polito.it
}

\begin{document}

\maketitle

\begin{abstract}
In numerous settings, agents lack sufficient data to learn a model directly. Collaborating with other agents may help, but introduces a bias-variance trade-off when local data distributions differ.
%While it may reduce model variability, it can introduce bias if local data distributions differ. 
A key challenge is for each agent to identify clients with similar distributions while learning the model, a problem that remains largely unresolved.
This study focuses on a particular instance of the overarching problem, where each agent collects samples from a real-valued distribution over time to estimate its mean. Existing algorithms face impractical per-agent space and time complexities (linear in the number of agents~$|\mathcal{A}|$). % in large-scale systems.
To address scalability challenges, we propose a framework where agents self-organize into a graph, allowing each agent to communicate with only a selected number of peers~$r$. We propose two collaborative mean estimation algorithms: one employs a consensus-based approach, while the other uses a message-passing scheme, with complexity  $\bigO(r)$ and $\bigO(r \cdot \log |\mathcal{A}|)$, respectively. 
We establish conditions for both algorithms to yield asymptotically optimal estimates and we provide a theoretical characterization of their performance. % the asymptotic performance of our proposed algorithms.
\end{abstract}

\begin{links}
\link{Code}{https://github.com/Franco-Galante/scalable-decentralized-algorithms-AAAI25}
\end{links}


% OUR LINKS, looks a bit ugly in the paper
%   \link{Code}{https://github.com/Franco-Galante/scalable-decentralized-algorithms-AAAI25}
%    \link{Extended version}{https://arxiv.org/abs/2402.12812}

\section{Introduction}\label{sec:intro}

Users' devices have become increasingly sophisticated and generate vast amounts of data. This wealth of data has enabled %paved the way for 
the development of accurate and complex models. However, it has also introduced challenges related to security, privacy, real-time processing, and resource management. In response, Federated Learning (FL) has emerged as a key privacy-preserving approach for collaborative model training~\citep{kairouzAdvancesOpenProblems2021,liFederatedLearningChallenges2020}.
While traditional FL methods aim to develop a single model for all clients, the statistical diversity of clients' datasets has led to the development of \emph{personalized} models, designed to better align with the data distributions of individual clients~\citep[e.g.,][]{ghoshEfficientFrameworkClustered2020, personalized, liDitto, marfoq2021federated, dingCollaborativeLearningDetecting2022}.

Many personalized FL strategies group clients into 
%\textcolor{cyan}{similar} %homogeneous
clusters and then tailor a model for each cluster~\citep[e.g.,][]{ghoshEfficientFrameworkClustered2020, sattler, dingCollaborativeLearningDetecting2022}.
% The optimal clustering would ideally group clients whose local optimal models are similar. However, the inherent challenge lies in the fact that these optimal models are unknown a priori. This dilemma results in the two tasks---model learning for the specified task and cluster identification---being deeply interconnected.
Ideally, clustering would group clients with similar local optimal models. However, since the optimal models are unknown a priori, model learning and cluster identification become deeply interconnected tasks.
Various studies have suggested empirical measures of similarity as a workaround~\citep[e.g.,][]{ghoshEfficientFrameworkClustered2020, sattler}, while others rely on presumed knowledge of distances across data distributions~\citep[e.g.][]{dingCollaborativeLearningDetecting2022, evenSampleOptimalityPersonalized2022}. Nonetheless, accurately estimating these distances, especially within an FL framework where clients may possess limited data, proves to be particularly challenging. 
These estimation difficulties are well documented in the literature---e.g., by \citet{evenSampleOptimalityPersonalized2022} [Sec.~6]---highlighting the problem of identifying similar clients for collaborative model learning as a significant yet unresolved issue.

In this paper, we focus on a fundamental aspect of the broader challenge: estimating the mean of an $\mathbb{R}^K$-valued distribution. This problem is often regarded as the archetypal federated learning problem~\cite{dorner2024incentivizing, tsoy2024provable, grimberg2021optimal}, but it also
holds significant practical relevance across various fields, such as smart agriculture, grid management, and healthcare, where multiple sensors collect private, noisy data on identical or related variables~\citep{Adi2020}.

We consider an online, decentralized scenario where, at each time slot, clients receive new samples and exchange information with a limited number of peers. 
To the best of our knowledge, the state-of-the-art method in this setting is the Collaborative Mean Estimation algorithm (\colme) by \citet{colme}. 
Unfortunately, \colme{} faces scalability issues in large systems, as both its per-agent space and time complexities are linear in the number of clients~$|\mathcal A|$. Moreover, in its current form, \colme{} is applicable only to scalar mean estimation problems ($K=1$) and its convergence guarantees only hold for sub-Gaussian data distributions.

We extend the methodology proposed by \citet{colme}
to accommodate multidimensional data drawn from the broader class of distributions with bounded fourth moment.
To address \colme's scalability challenges, we propose that clients self-organize into a network where each client communicates with at most~$r$ neighbors. Over time, this set of neighbors is pruned as clients progressively exclude the less similar ones. In this framework, we introduce two collaborative mean estimation algorithms: one based on consensus and the other on a message-passing scheme. The complexities of these algorithms are \(\bigO(r )\) and \(\bigO(r \cdot \log |\mathcal{A}|)\), respectively. We demonstrate that, despite each client exchanging information with only \(r \ll |\mathcal{A}|\) neighbors, it is possible to achieve a convergence speedup for
mean estimates by a factor of $\Omega(|\mathcal{A}|^{1/2 - \phi})$, where $\phi$ can be made arbitrarily close to~$0$. 

Lastly, we conduct preliminary experiments demonstrating how our algorithms can be adapted to federatedly learn more general machine learning models.

\section{Related Work}\label{sec:related}

% For an overview of the field of personalized federated learning, we refer the reader to the recent survey~\cite{tanPersonalizedFederatedLearning2023}. Here, we limit to mention the most relevant approaches for this paper.
For an overview of personalized federated learning, see the recent survey by~\citet{tanPersonalizedFederatedLearning2023}. Here, we highlight only the most relevant approaches related to this paper.

\citet{ghoshEfficientFrameworkClustered2020} and \citet{sattler} were the first to propose clustered FL algorithms, which divide the clients based on the similarity of their data distributions. Similarity is empirically evaluated by the Euclidean distance between local models and by the cosine similarity of their updates. \citet{dingCollaborativeLearningDetecting2022} study more sophisticated clustering algorithms assuming that clients can efficiently estimate some specific (pseudo)-distances across local distributions (i.e., the integral probability metrics). 

\citet{beaussartWAFFLEWeightedAveraging2021}, \citet{chayti2022linear}, \citet{grimberg2021optimal}, and \citet{evenSampleOptimalityPersonalized2022} consider decentralized approaches, which allow each client to learn a personal model relying on a specific convex combination of information (gradients) from other clients. In particular, \citet{evenSampleOptimalityPersonalized2022} prove that collaboration can at most speed up the convergence time linearly in the number of similar agents and provide algorithms, which, under a priori knowledge of pairwise client distributions' distances, achieve such speedup. The authors recognize the complexity of estimating these distances and provide practical estimation algorithms for linear regression problems, which asymptotically achieve the same speedup, scaling the number of clients but maintaining the number of clusters fixed.
The generalization properties of personalized models obtained by convex combinations of clients' models are studied in~\citet{mansourThreeApproachesPersonalization2020}, while~\citet{donahueModelsharingGamesAnalyzing2021} look at the problem through the lens of game theory.
The work most similar to ours is \citet{colme}, which we describe in detail in the next section.

\section{Model and Background}\label{sec:model}

Table~\ref{tab:notation} lists the most important symbols used throughout the paper. Superscripts are added to variables to indicate whether they pertain to C-ColME (C), B-ColME (B), or both approaches (D).

We consider a set $\mathcal{A}$ of agents (computational units).
At each time instant $t$, an agent $a\in \mathcal{A}  = \{1, 2,\cdots , |\mathcal{A}|\}$ generates  a new sample $\bm{x}_a^t \in \mathbb{R}^K$, with $K\in \mathbb{N}$, drawn i.i.d.~from  
%an assigned individual 
a distribution $D_a$ with expected value $\bm{\mu}_a=\mathbb{E}[\bm{x}_a^t]$.  
Expected values are not necessarily distinct across agents. Indeed, given two agents $a$ and $a'$, and a norm~$||\cdot ||$ in~$\mathbb{R}^K$, 
we denote the gap between the agents' \textit{true} means by $\Delta_{a,a'}:=||\bm{\mu}_{a}-\bm{\mu}_{a'}||$. Let $\mathcal{C}_a$ be
the group  of  agents with the same \textit{true} mean as $a$, (i.e., those for which  $\Delta_{a,a'}=0$). In the following, we will refer to $\mathcal{C}_a$
%with  the term
as the \lq similarity class of~$a$.'  %We assume that nodes are partitioned into relatively few similarity classes.

The goal of each agent $a \in \mathcal{A}$ is to %efficiently and 
estimate its mean~$\bm{\mu}_a$. To this end, at each time~$t$ the agent can compute its \textit{local} mean estimate $\bar{\bm{x}}^t_{a,a}= \frac{1}{t}\sum_{t'=1}^{t}\bm{x}^{t'}_{a}$ over the~$t$ available samples. Additionally, the agent can obtain a more accurate estimate by leveraging information from other agents in $\mathcal{A}$, provided it can identify those who share the same true mean.
% but can also obtain a more accurate estimate taking advantage of information from other agents in its similarity class.

For the scalar case, (i.e., when $K=1$)
\citet{colme} proposed \colme{} as a collaborative algorithm for mean estimation. It relies on two key steps, executed concurrently: i)~the identification of the similarity classes; 
%% ----[old, second-to-last version]---- trying to identify the similarity classes;
%partitioning \lq\lq on the fly" the agents into similarity classes,  %placing in the same class all agents $a'$ with  the same local
% grouping together agents $a'$ with the same \textit{true} mean; 
ii)~the collaboration with agents \textit{believed to belong} to the same class to improve the local estimate.
% ----[old, second-to-last version]---- ii) improving local mean estimates by sharing information with agents believed to belong to the same class.
%We denote with similarity class, $\mathcal{C}$, a maximal set of agents, with the same average, i.e., such that  $\Delta_{l_1,l_2}:=|\mu_{l_1}-\mu_{l_2}|=0$ forall $l_1,l_2\in \mathcal{C}$.   
%Note, tht by construction  the set of agents $\mathcal{A}$ is partitioned in similiarities classes   
%  $\mathcal{C}_i$ i.e.,    $\mathcal{A}=\cup_{i\in[I]} \mathcal{C}_i$. being $[I]$  the set of similarities classes.
%
%{\textcolor{blue}{Ho la sensazione che convenga comunque definire le comunita centrandole sui nodi, altrimenti sorge il problema  che le comunita' possono sparire, venir create dinamicamete e impazziamo con gli indici... }
%
%\textcolor{black}{EL-forse dobbiamo enfatizzare meglio i  contributi nuovi rispetto a \colme{}:L'estensione al caso multdimensionale e ancora di piu' l'estensione al caso generale non sub-gaussiano sono abbastanza nascosti}
% ----[old, second-to-last version]---- Therefore,  to make our results algorithms directly comparable with  those reported in  \cite{colme}, and to ease the notation, in the following  we limit ourselves to present the scalar case; we refer the interested reader to the supplementary material for the analysis  of the more general case.
To ensure direct comparability with the results in \cite{colme} and simplify notation, in what follows, we focus on the scalar case. Readers interested in the general case can find the analysis in \textcolor{black}{Appendix~\ref{app:multidim}}. All the appendices referred to in this work are accessible at~\cite{our_airxiv}. 
%\footnote{\textcolor{blue}{The extended version of this work is available at \url{https://arxiv.org/abs/2402.12812} and the code at {https://github.com/Franco-Galante/scalable-decentralized-algorithms-AAAI25}.}}. 
% In what follows, we use superscripts on variables to indicate whether they pertain to C-ColME (C), B-ColME (B), or both approaches (D).

% \textcolor{blue}{To improve readability, Table~\ref{tab:notation} lists the most important symbols used throughout the paper. The superscript to variables denotes whether a variable refers to \consensus~($C$), \belief~($B$), or either of the two approaches~($D$).}

\begin{table}[h!]
    \centering
    \renewcommand{\arraystretch}{1.15} % add some space between items
    \resizebox{\columnwidth}{!}{ %% UNCOMMENT THIS TO MAKE IT FIT THE COLUMN OR SMALLER (adding a coefficient before \columnwidth)
        \begin{tabular}{c|l}
            \toprule
            \textbf{Symbol} & \textbf{Description}\\
            \midrule
            $\mathcal{A}$      & Set of agents (computational units)       \\ 
            $D_a$              & Distribution of agent $a \in \mathcal{A}$  \\ 
            $\bm{\mu}_a$       & True mean of distribution $D_a$, i.e., $\bm{\mu}_a=\mathbb{E}[\bm{x}_a^t]$ \\
            $\bar{\bm{x}}^t_a$ & Vectorial sample $\bar{\bm{x}}^t_a \in \mathbb{R}^K$ drawn from $D_a$  \\
            $\Delta_{a,a'}$    & Gap between agents $a$ and $a'$ true means \\ % \Delta_{a,a'}:=||\bm{\mu}_{a}-\bm{\mu}_{a'}|| \\
            $\mathcal{C}_a$    & Similarity class of $a$, $\mathcal{C}_a = \{a' \in \mathcal{A} | \Delta_{a,a'}=0\}$ \\ 
            \hline %% divide static given quantities by dynamic ones
            $\mathcal{C}_a^t$  & Estimated similarity class of $a$ at time $t$ \\ 
            $\bar{x}^t_{a,a}$  & Local mean estimate of agent $a$ at $t$ \\ 
            $\bar{x}^t_{a,a'}$ & Local mean estimate of agent $a'$ by agent $a$ at $t$\\ 
            $n^t_{a,a'}$       & Number of samples used to compute $\bar{x}^t_{a,a'}$ at $t$\\
            $\hat{\mu}_a^t$    & Collaborative mean estimate at $t$\\
            $\beta_\gamma (n)$ & Width of the confidence interval \\
            $d_\gamma^t (a,a')$& Optimistic distance between agents $a$ and $a'$\\
            $\zeta_a$          & Time to identify same-class neighbors w.h.p. \\ 
            $\tau_a$           & Time to obtain $(\epsilon, \delta)$ convergence  \\ \hline % divide system and graph related quantities
            $\mathcal{G}$      & Collaborative graph $\mathcal{G}(\mathcal{A}, \mathcal{E})$ \\ 
            $\mathcal{G}^t$    & Pruned collaborative graph $\mathcal{G}^t(\mathcal{A}, \mathcal{E}^t)$ \\
            $\mathcal{N}_a$    & Neighborhood of agent $a$ (or up to distance $d$: $\mathcal{N}_a^d$)\\ %%, i.e., $\mathcal{N}_a = \{a' | (a,a') \in \mathcal{E} \}$ \\
            $r$                & Upper bound on $a$'s neighborhood size $|\mathcal{N}_a|$ \\ 
            $\mathcal{CC}_a$   & Agents in the connected component of the subgraph   \\% s.t. $a' \in \mathcal{CC}_a, a' \in \mathcal{C}_a$ \\ 
            & induced by agents in $\mathcal{C}_a$ to which agent $a$ belongs \\
            $\mathcal{CC}_a^d$ & Same as $\mathcal{CC}_a$ but with agents up to $d$-hops from $a$ \\ 
            \bottomrule
        \end{tabular}
        }
    \caption{Notation Summary}\label{tab:notation}
\end{table}

\paragraph{1) Identifying Similarity Classes.}
We denote $\mathcal{C}_a^t$ as the set of agents that, at time~$t$, agent~$a$ \textit{estimates} to belong to its similarity class~$\mathcal{C}_a$.
% ----[old, second-to-last version]---- Agent~$a$ decides that agent~$a'$ belongs to the same class if their local average estimates are sufficiently close, specifically, 
Specifically, agent~$a$ includes agent~$a'$ in its \textit{estimated} similarity class~$\mathcal{C}_a^t$ if their local mean estimates are sufficiently close.
% \textcolor{cyan}{To be specific, agent~$a$ includes agent~$a'$ in its \textit{estimated} similarity class~$\mathcal{C}^t_a$ if their local mean estimates are sufficiently close, i.e.,}
% if two appropriately defined confidence intervals~$I_{a,a}$ and~$I_{a,a'}$ centered on these estimates overlap.
In general, at time~$t$, agent~$a$ does not have access to the most recent local mean estimate of agent~$a'$ (computed over~$t$ samples), but rather to a \textit{stale} value $\bar{x}_{a,a'}^t$ computed over $n_{a,a'}^t$ samples, where $n_{a,a'}^t$ corresponds to the time when~$a$ and~$a'$ last communicated.
% We denote such value as $\bar{x}_{a,a'}^t$, 
% \textcolor{cyan}{and with $n_{a,a'}^t$ 
% the number of samples over which~$\bar{x}_{a,a'}^t$ has been computed (which coincides with the time of the last information exchange between $a$ and $a'$).}
%% ----[old, second-to-last version]---- and the corresponding number of samples as $n_{a,a'}^t$. \footnote{Note that $n_{a,a'}^t \leqslant t$ coincides with the last time slot at which the two agents have communicated.}
%($\leq t$), which coincides with the last time slot at which the two agents communicated.
Agent $a$ can then estimate its true mean and the true mean of agent $a'$ to belong to the confidence intervals $I_{a,a} =\left[\bar{x}^t_{a,a}-\beta_\gamma(t),\bar{x}^t_{a,a}+\beta_\gamma(t)\right]$, $I_{a,a'}~=~\left[\bar{x}^t_{a,a'}-\beta_\gamma(n^t_{a,a'}),\bar{x}^t_{a,a'}+\beta_\gamma(n^t_{a,a'})\right]$, respectively. As expected, the interval amplitude $\beta_{\gamma}(n)$ depends on the number of samples $n$ on which the empirical average is computed and on the target level of confidence 1-2$\gamma$ associated with the interval. 
Agent $a$ will then %\emph{optimistically} 
% assume that $a'$ belongs to its similarity class $\mathcal C_a$
consider $a'$ to belong to its \textit{estimated} similarity class~$\mathcal{C}^t_a$
if the two intervals overlap, i.e., $I_{a,a}~\cap~ I_{a,a'}\neq \emptyset$, or equivalently if the optimistic distance~$d_\gamma^t (a,a')$ is zero or less: 
\begin{equation}\label{opt-dist-colme}
    %d^t_{a,\gamma}(a'):=|\bar{x}_{a,a}-\bar{x}_{a,a'}|-\beta_\gamma(n^t_{a,a})-\beta_\gamma(n^t_{a,a'})
    d^t_{\gamma}(a,a'):=|\bar{x}_{a,a}-\bar{x}_{a,a'}|-\beta_\gamma(t)-\beta_\gamma(n^t_{a,a'})\leqslant 0.
\end{equation}
%% In order to estimate its similarity class in the most accurate way, 
To achieve the most accurate estimation of its similarity class, agent~$a$ could retrieve from each other peer~$a'$ the most recent estimate~$\bar{x}^t_{a,a'}$ at each time~$t$. % the freshest possible  estimate of $\bar{x}^t_{a,a'}$ computed on all the samples  locally available by time $t$.
However, this approach would result in a per-agent communication burden of~$\bigO(|\mathcal{A}|)$.
% However, this would lead to a %quadratic 
% linear per-agent communication burden 
% $\bigO(\mathcal{|A|})$.
%$\Theta(\mathcal{|A|}^2)$.
%%To limit such a burden, 
To mitigate this effect, node $a$ cyclically queries 
a \textit{single} node 
%only $v$ other nodes 
from~$\mathcal{C}_a^{t-1}$ 
%($v=1$ in \citet{colme}) 
at each time instant $t$, according to a  \textit{Round-Robin} scheme. 
%and uses the newest available estimate $\bar{x}_{a,a'}^{t}$ associated to $n_{a,a'}^{t}$.
This leads to the following update rule for~$\mathcal{C}_a^t$:
\begin{equation}
\label{e:c_update}
\mathcal{C}_a^t=\{ a' \in \mathcal{C}_a^{t-1} : d^t_{\gamma}(a,a')\leqslant 0   \},
\end{equation}
and we observe that by construction $\mathcal{C}_a^t\subseteq \mathcal{C}_a^{t-1}$.
%, thus, when $a$ removes one of other agents  $a'$ from $\mathcal{C}_a^t$ it removes it  forever.
Initially, the agent sets $\mathcal{C}_a^0=\mathcal A$ and progressively makes irreversible decisions to remove agents that are deemed too dissimilar.

\paragraph{2) Estimating the Mean.}
Each node $a$ computes an  \textit{estimate} $\hat{\mu}_a^t$ of its true mean $\mu_a$ combining the available estimates according to a \textit{simple weighting} scheme, where the number of samples $n^t_{a,a'}$ are the weights:
\[
\hat{\mu}_a^t = \sum_{a' \in \mathcal{C}_a^t} \frac{n^{t}_{a,a'}}{\sum_{a' \in \mathcal{C}_a^t} n^{t}_{a,a'}} \bar{x}_{a,a'}^t.
\]

\paragraph{\colme's theoretical guarantees.}
\citet{colme} prove that if data distributions $\{D_a\}_{a \in \mathcal{A}}$ are sub-Gaussian and the amplitude of the confidence intervals is selected as:
\begin{equation}\label{eq:beta_mailard_main}
    \beta_{\gamma}(n)=\sigma \sqrt{\frac{2}{n}\left(1+ \frac{1}{n}\right)\ln(\sqrt{(n+1)}/\gamma )},
\end{equation}
then, with high probability, client $a$ correctly identifies  its similarity class ($\mathcal{C}_a^t = \mathcal C_a$)  and obtains an $\epsilon$-accurate estimate for $t$ larger than two opportune constants $\zeta_a$ and $\tau_a$, respectively (see \textcolor{black}{Appendix~\ref{sec:colme-perf}} for the expressions and the results' statements). %  formal statements of the results).

\paragraph{\colme's limitations.} \citet{colme} acknowledge in their paper the main limitations of \colme{}:
each agent requires a memory footprint and computational complexity proportional to~$|\mathcal A|$. Indeed, each agent~$a$ must store all neighbors' local estimates~$\bar{x}^t_{a,a'}$ and the corresponding sample counts~$n^t_{a,a'}$.
Additionally, as agent~$a$ receives a new sample at each time~$t$, it updates its local estimate $\bar{x}^t_{a,a}$ (and also~$n^t_{a,a}$). This update affects the distance $d_\gamma^t(a,a')$, which thus must be recomputed for all~$a' \in \mathcal{C}_a^t$.
This leads to a
%total time and space complexity of $\bigO(|\mathcal A|^2)$ per time slot, 
per-agent time and space complexity of $\bigO(|\mathcal A|)$ per time slot, 
which becomes impractical in large-scale systems. Moreover, while the \textit{Round-Robin} query scheme reduces the communication burden, it introduces a significant delay in the estimation as~$\zeta_a \in \bigO(|\mathcal A|)$.

\section{Scalable Algorithms over a Graph $\mathcal{G}$}
\label{sec:B_and_C}

Online mean estimation can be made~$\tilde{\bigO}(1)$  with our \textit{scalable} approaches: \consensus\ (Sec.~\ref{sec:consensus}), and \belief\ (Sec.~\ref{sec:b_colme}). 
Both algorithms consider agents~$\mathcal{A}$ organized in a fixed graph~$\mathcal{G}( \mathcal{A}, \mathcal{E})$ and restrict communication to pairs of agents adjacent in~$\mathcal{G}$. 
Let~$\mathcal{N}_a$ and $\mathcal{N}_a^d$ denote the set of neighbors of agent $a$  and the set of agents at distance at most~$d$ from~$a$, respectively.
Let $r$ represent the maximum size of any agent's neighborhood in~$\mathcal{G}$, i.e.,  $|\mathcal{N}_a| \leq r$ for all $a \in \mathcal A$.

Consider the subgraph $\mathcal{G}'$ of $\mathcal{G}$ induced by the agents in~$\mathcal{C}_a$. 
Let $\mathcal{C C}_a$ denote the (initially unknown) set of agents in the connected component of $\mathcal{G}'$ to which~$a$ belongs, and let $\mathcal{C C}_a^d \subset \mathcal{C C}_a$ represent the subset of agents within $\mathcal{C C}_a$  that are at most~$d$ hops away from~$a$. 

Each agent $a$ aims to identify which nodes in its neighborhood $\mathcal{N}_a$ belong to its similarity class $\mathcal{C}_a$.
To achieve this, agent~$a$ receives at time~$t$ an updated local mean estimate $\bar{x}^{t'}_{a,a'}$ from each neighbor $a' \in \mathcal{N}_a$. 
We denote with $\mathcal{C}_a^t \subseteq \mathcal{N}_a$ the set of neighbors that agent~$a$ deems to belong to its own similarity class at time~$t$, initially: $\mathcal{C}_a^0= \mathcal{N}_a$. Similarly to \colme, at each time~$t$, agent~$a$ first  computes the  distance~$d_\gamma^t(a,a')$ for every $a' \in \mathcal{C}_a^{t-1}$ according to~\eqref{opt-dist-colme}
and then updates $\mathcal{C}_a^t $ according to~\eqref{e:c_update}.
As for \colme, $\mathcal{C}_a^t \subseteq \mathcal{C}_a^{t-1} $ and 
as soon as~$a$ removes~$a'$ from $\mathcal{C}_a^t $, it stops communicating with~$a'$. 
Subsequently, communication occur over the pruned graph~$\mathcal G^t = (\mathcal A, \mathcal E^t)$, where~$\mathcal E^t = \{(a,a') \in \mathcal E : a' \in~\mathcal{C}_a^t\}$.

The theoretical guarantees of our algorithms hold under more general settings than those in \cite{colme}. In particular, they apply to any set of distributions $\{D_a\}_{a \in \mathcal A}$ for which the following assumption is satisfied:

\begin{assumption}
\label{ass:beta}
    There exists a positive function $\beta_\gamma(\cdot) \in o(1)$ such that the true mean belongs to all %\textcolor{blue}{(nested)} 
	intervals centered in $\bar{x}_{a,a}^t$ of width $\pm \beta_\gamma(t)$ for $ t \in \mathbb N$ with confidence $1 - 2 \gamma$, namely:
    \begin{equation}
         \label{eq:lower_bound}
         \mathbb{P} \left( \forall t \in \mathbb{N}, \left| \bar{x}_{a,a}^t - \mu \right| {<} \beta_\gamma (t) \right) \geq 1 - 2 \gamma, \forall a \in \mathcal A.
     \end{equation}
\end{assumption}
Assumption~\ref{ass:beta} is satisfied by sub-Gaussian distributions (SGD) with parameter $\sigma^2$, by selecting $\beta_{\gamma}(\cdot)$ as in~\eqref{eq:beta_mailard_main}. In \textcolor{black}{Appendix~\ref{appendix:th1}}, 
we show that the assumption also holds for bounded fourth-moment distributions (BFMD) for $\beta_{\gamma}(\cdot)$ chosen as follows:  
\begin{equation}\label{eq:beta_cheby}
 \beta_\gamma(n)= \left(2\frac{(\kappa + 3 )\sigma^4}{\gamma}\right)^{\frac{1}{4}} \left(\frac{1 + \ln^2 n}{n}\right)^{\frac{1}{4}},
\end{equation}
where $\sigma^2$ bounds the variance of the distributions  $\{D_a, \forall a \in \mathcal A\}$ and $\kappa\sigma^4$ their fourth moment.
When all the variables are identically distributed, $\kappa$ 
corresponds to the kurtosis.
Moreover, for distributions with a larger number of bounded moments, tighter expressions can be derived for $\beta_\gamma(\cdot)$ (see Remark~\ref{rem-moments} in \textcolor{black}{Appendix~\ref{appendix:th1}}).
In what follows, we assume that Assumption~\ref{ass:beta} is always satisfied.

We aim first to determine the time needed for all agents in the connected component $\mathcal{CC}_a$ to identify the subset of neighbors residing in their similarity class, i.e.,  $\mathcal{C}_{a'}^t=\mathcal{C}_{a'} \cap \mathcal{N}_{a'} $, $\forall a' \in \mathcal{CC}_a$. 
Following a similar approach to~\citet[][Theorem 1]{colme} %(reported in Appendix~A as Theorem~\ref{thm:colme_thm1}), 
we can prove that:
\begin{theorem}
    \label{cor:Dcolme_thm1}
   \textup{[Proof in \textcolor{black}{Appendix~\ref{appendix:th1}}]}
    Considering an 
    arbitrarily chosen agent~$a$ in~$\mathcal A$,
    %agent~$a$ picked arbitrarily in~$\mathcal{A}$}, 
    for any $\delta \in (0,1)$, employing either \belief\ or \consensus\, we have:
    \begin{equation}
        \mathbb{P}\left(\exists t>\zeta_a^D,  \exists a' \in \mathcal{CC}_a: \mathcal{C}_{a'}^t  \neq \mathcal{C}_{a'} \cap {\mathcal N_{a'}}\right) \leqslant \frac{\delta}{2}, 
    \end{equation}
    with $\zeta_a^D=n_{\gamma}^{\star}\left(\frac{\Delta_a}{4}\right) + 1 $, $\Delta_a=\underset{a'  \in \mathcal{A} \setminus \mathcal{C}_a }{\min}\Delta_{a,a'}$, $\gamma= \frac{\delta}{4 r | \mathcal{CC}_a|}$.
    $n^{\star}_{\gamma}(x)$ denotes the minimum number of samples that are needed to ensure $\beta_\gamma(n)< x$, i.e.,
$n^{\star}_{\gamma}(x)= \lceil \beta^{-1}_\gamma(x)\rceil$.
\end{theorem}
This result demonstrates that the time required for all agents~$a'$ in~$\mathcal{CC}_a$ (the connected component to which~$a$ belongs) to correctly identify their neighbors within the same similarity class $\mathcal C_a$ is bounded by~$n^{\star}_{\gamma}(\frac{\Delta_a}{4})+1$.
Here, $n^{\star}_{\gamma}(\frac{\Delta_a}{4})$ represents the number of samples needed to distinguish (with confidence $1-2\gamma$) the true mean of agent~$a$ from that of an agent belonging to the `closest' similarity class (i.e., the one with the closest true mean). The additional $1$ accounts for the unit delay in communicating with the neighbors. 

When comparing performance of \belief{} and \consensus{} (Theorem~\ref{cor:Dcolme_thm1}) with  \colme{} \cite{colme} [Theorem~1] (reported in \textcolor{black}{Appendix~\ref{sec:colme-perf}} as Theorem~\ref{thm:colme_thm1} for completeness),
we observe that for large systems, if $r |\mathcal{CC}_a| \in \Theta(|\mathcal A|)$,  $\zeta_a \approx |\mathcal A| + \zeta_a^D$, showing that, as expected, agents can identify much faster similar agents in their neighborhood than in the whole population $\mathcal A$.\footnote{
    For a fairer comparison, we should let \colme{} query $r$ other agents at each time $t$, where $r$ is the average degree of $\mathcal G$. In this case, $\zeta_a \approx |\mathcal A|/r + \zeta_a^D$ and the conclusion does not change.
}  See Sec.~\ref{s:comparison} for a detailed comparison of \colme{}, \consensus{}, and \belief{}. % \textit{Algorithms' Comparison}.

\subsection{Consensus-based Algorithm: \consensus{} }
\label{sec:consensus}

This section introduces the first collaborative mean estimation approach, inspired by consensus algorithms in dynamic settings, as in~\cite{montijanoRobustDiscreteTime2014,franceschelliMultistageDiscreteTime2019}. The basic idea is that each agent maintains two metrics: the empirical average of its local samples~$\bar{x}_{a,a}^t$, and the \lq consensus\rq\ estimate $\hat{\mu}_a^t$. The consensus variable is updated at time~$t$ by computing a convex combination of the local empirical average~$\bar{x}^t_{a,a}$ and a weighted sum of the consensus estimates in its (close) neighborhood~$\{\hat{\mu}_{a'}^{t-1}, a' \in \mathcal C_a^{t-1} \cup \{a\}\}$, see Algorithm~\ref{alg:consensus}.

The dynamics of all estimates are captured by: % the following equation:
\begin{equation}
\label{e:consensus_one_step_dynamics}
{\y}^{t+1}=\left(1-\alpha_t\right) \bar{\x}^{t+1}+\alpha_t W_t {\y}^t,
 \end{equation}
where  $(W_t)_{a, a'}=0$ if $a' \notin \mathcal C_a^t$ and $\alpha_t \in (0,1)$ is the memory parameter. 
Once the agents cease pruning their neighbors, say at time $\tau$,
the matrix $W_t$ does not need to change anymore,  i.e., $W_t = W$ for any $t\geqslant \tau$ with $W_{a,a'}> 0$ if and only if $a'\in \mathcal C_a \cap \mathcal N_a$. In order to achieve consensus, the matrix $W$ needs to be doubly stochastic~\cite{boyd} and we also require it to be symmetric.
By time $\tau$, the original communication graph is split into $C$ connected components, where component $c$ includes $n_c$ agents. By an opportune permutation of the agents,
we can write the matrix $W$ as follows 
\begin{equation}
\label{e:matrix_W}    
W=\left(\begin{array}{cccc}\prescript{}{1}W & 0_{n_1 \times n_2} & \cdots & 0_{n_1 \times n_C} \\ 
0_{n_2 \times n_1} & \prescript{}{2}W & \cdots & 0_{n_2 \times n_C} \\ 
\cdots & \cdots & \cdots & \cdots\\
0_{n_C \times n_1} & 0_{n_C \times n_2} & \cdots & \prescript{}{C}W\end{array}\right),
\end{equation}
where each matrix $\Wc$ is an $n_c \times n_c$ symmetric stochastic matrix.
For $t\geqslant \tau$, the estimates in the different components evolve independently. We can then focus on a given component $c$. All agents in the same component share the same expected value, which we denote by $\mu(c)$. Moreover, let $\mmc=\mu(c)  \mathbf{1}_c$. We denote by $\xc^t$ and $\yc^t$ the $n_c$-dimensional vectors containing the samples' empirical averages and the consensus estimates for the agents in component $c$ 
and by~$\lambda_{2,c}$ the second largest module of the eigenvalues of $\Wc$.

Note that the actual evolution of~$W_t$ is challenging to characterize due to topology changes during the graph pruning phase.
However, our main results (Theorems~\ref{thm:fourth_moment_main_text} and~\ref{thm:consensus_epsilon_delta}) remain applicable to any system where the sequence of matrices $W_t$ for $t \leq \tau$ is arbitrarily set.

\begin{algorithm}[h]
   \caption{\consensus{} over a Time Horizon $H$}
   \label{alg:consensus}
   \begin{algorithmic}
      \STATE {\textbf{Input:}} $\,\,\,\, \mathcal{G} = \left(\mathcal{A},\, \mathcal{E} \right)$, $(D_a)_{a \in \mathcal{A}}$, $\epsilon \in \mathbb R^+, 
    \delta \in (0,1]$
      \STATE {\textbf{Output:}} $\hat{\mu}_a, \, \forall a \in \mathcal{A}$
      with $\mathbb{P}\left( |\hat{\mu}_a - \mu_a| < \epsilon \right) \geqslant 1- \delta$
      \STATE $\mathcal{C}_a^0 \leftarrow \mathcal{N}_a, \forall a \in \mathcal A$
      %\WHILE{$\text{Pr}_N\left\{ |\hat{\mu} - \mu| \leqslant\epsilon\right\} < 1 - \delta$}
      \FOR{time $t$ in $\{1,..,H\}$}
        \STATE In parallel for all nodes $a\in \mathcal A$
            \STATE Draw $x_a^t \sim D_a$
            \STATE $\bar{x}^t_a \leftarrow \frac{t-1}{t} \bar{x}^{t-1}_a + \frac{1}{t} x_a^t$
            \STATE Compute $\beta_\gamma(t)$ with Eq. (\ref{eq:beta_mailard_main}) or Eq. (\ref{eq:beta_cheby})
            \FOR{neighbor $a'$ in $\mathcal{N}_a \cap \mathcal{C}_a^{t-1}$}
                \STATE $d_\gamma^t (a,a') \leftarrow \left| \bar{x}_a^{t} - \bar{x}_{a'}^{t-1} \right|  - \beta_\gamma(t) - \beta_\gamma(t-1)$
            \ENDFOR
            \STATE $\mathcal{C}_a^t \leftarrow \left\{a' \in \mathcal{N}_a \cap \mathcal{C}_a^{t-1} \,\, \text{s.t.} \, d_\gamma^t(a,a') \leqslant0 \right\}$
            \STATE $\hat{\mu}_a^{t} \leftarrow (1-{\alpha_t}) \bar{x}_a^{t} +  \alpha_t \sum_{a' \in \mathcal{C}_a \cup \{a\}} (W_t)_{a,a'} \hat{\mu}^{t-1}_{a'}$
      \ENDFOR
   \end{algorithmic}
\end{algorithm}

\begin{theorem}
\label{thm:fourth_moment_main_text}
\textup{[Proof in \textcolor{black}{Appendix~\ref{sec:c_colme_proofs}}]} Consider a system which evolves according to~\eqref{e:consensus_one_step_dynamics} with $W_t = W$  
in \eqref{e:matrix_W},
for $t\geqslant \tau$.
Let $\Pc= 1/n_c \mathbf{1}_c  \mathbf{1}_c^\intercal$. For $ \alpha_t=\frac{t}{t+1}$, it holds:
\begingroup\makeatletter\def\f@size{8}\check@mathfonts
\def\maketag@@@#1{\hbox{\m@th\large\normalfont#1}}%
\begin{align*}
& \E\left[\lVert {\yc}^{t+1} - \mmc \rVert^4 \right]  \in 
 \bigO\left( \sup_{W_1, \cdots, W_{\zeta_D}}
 \frac{\E\left[\lVert \yc^{\zeta_D} - \mmc \rVert^4 \right]}{(t+1)^4} \right)\nonumber\\ 
 & + \bigO \left( \frac{\left(1 - {1}/{\ln{{\lambdac }}} \right)^2}{(1-\lambdac)^2} \frac{\E\left[\lVert \xc -  \Pc \xc \rVert^4 \right]}{(t+1) ^4}\right)\nonumber\\
 & + 
\bigO \left(\E\left[\lVert \Pc \xc -  \mmc \rVert^4 \right] \left(\frac{1 + \ln t}{ 1 + t }\right)^2 \right).
\end{align*}\endgroup
\end{theorem}

The theorem shows that the error, quantified through the fourth moment, can be decomposed into three terms decreasing over time. The first term depends on the estimates’ error at time~$\tau$.
%$\zeta_D$. 
The second term captures the effect of the consensus averaging, i.e., how effective is the algorithm in bringing the local estimates $\xc$ close to their empirical value $\Pc \xc$ (for example it is minimized if $\lambda_2=0$, which corresponds to $\Wc = \Pc$, the ideal choice for the matrix $\Wc$). Finally, the third term represents the minimum possible error, which would be obtained by averaging the estimates of all agents in the component using the matrix $\Pc$.

Theorem~\ref{thm:consensus_epsilon_delta} shows that \consensus{} achieves a speedup proportional to the size of the connected component~$|\mathcal{CC}_a|$.
\begin{theorem}
\label{thm:consensus_epsilon_delta}
\textup{[Proof in \textcolor{black}{Appendix~\ref{sec:c_colme_proofs}}]} Consider a graph component $c$ and pick uniformly at random an agent $a$ in $c$. 
Let $g(x) \coloneqq x \ln^2(e x)$ and  $\alpha_t=\frac{t}{t+1}$. Under BFMD, it holds: 
\begin{equation*}
    \mathbb{P}\left(\forall t>\tau^{C}_a,\, |\hat \mu_{a}^t-\mu_{a} |<\epsilon \right)\geqslant 1 -\delta
\end{equation*}
  where   $\tau_a^{C} =  \max\left\{ \zeta_a^D, g\!\left(C \frac{\E\left[\lVert \Pc \xc -  \mmc \rVert^4 \right]}{|\mathcal{CC}_a| \epsilon^4 \delta}\right) \in \tilde{\bigO}\left( \frac{\tilde{n}_{\frac{\delta}{2}}(\varepsilon)}{|\mathcal{CC}_a|}\right) \right\}$
  and $\widetilde{n}_{\frac{\delta}{2}}(\varepsilon) =  \Big \lceil \frac{2(\kappa +3)\sigma^4}{\delta \varepsilon^4}\Big \rceil$.
\end{theorem}

The theorem shows that the time to reach an $\epsilon$-accurate estimate with high probability is the maximum of the time for the agents in $\mathcal{CC}_a$ to identify their neighbors in the same similarity class and the time required for those agents to obtain and $\epsilon$-accurate estimate if they could share their own samples.
Indeed, we observe that~$n_{\delta/2}^{\star}(\epsilon)$ is the number of samples sufficient to ensure that $\mathbb{P}(|\hat \mu_a -\mu_a |>\epsilon)<\delta / {2}$ (see details in \textcolor{black}{Appendix~\ref{appendix:th2}}) and that the nodes in $\mathcal{CC}_a$ collectively gather this number of samples by time 
$t= \Big\lceil \frac{n_{\delta/2}^{\star}(\varepsilon)}{|\mathcal{CC}_a|}\Big\rceil$.

\textcolor{black}{Appendix~\ref{sec:c_colme_proofs}}
also presents convergence results for the case $\alpha_t = \alpha$, but they do not enjoy the same speedup factor.

\subsection{Message-passing Algorithm: \belief{}}
\label{sec:b_colme}

In \belief{}, each node $a \in \mathcal{A}$ continuously exchanges \textit{messages} with its direct neighbors~$a' \in \mathcal{C}^{t}_a$. This enables node~$a$ to acquire not only the neighbor's local estimates $\{\bar{x}_{a,a'}^{t}, a' \in \mathcal{C}^{t}_a\}$,
but also aggregated estimates from nodes up to a distance~$d$ in the graph~$\mathcal G^t$ (where~$d$ is a tunable parameter).
Indeed, each neighbor~$a'$ acts as a \textit{forwarder}, granting node~$a$ access to the records from its own neighbors $a'' \in \mathcal{C}^{t}_{a'} \setminus \{a\}$.
Provided each agent correctly identifies all similar nodes in its neighborhood, agent $a$ can potentially access the (delayed) local estimates of all agents in~$\mathcal{CC}_a^d$.

In our message-passing scheme, at time~$t$, agent~$a \in \mathcal{A}$ receives a message $M^{t,a'\rightarrow a}$ from all
neighbors~$a' \in \mathcal{C}^{t}_a$. 
The message $M^{t,a'\rightarrow a}$  is a $d \times 2 \,$ table %/matrix
whose elements $m^{t,a'\rightarrow a}_{h,1}$ contain a sum of samples,
while  $m^{t,a'\rightarrow a}_{h,2}$ indicates the number of samples contributing to this sum.
In particular, at each time $t$, the first row of the table is set as:
$m^{t,a'\rightarrow a}_{1,1}= \sum_{\tau=1}^t x^{\tau}_{a'}$   and $m^{t,a'\rightarrow a}_{1,2}=t$, i.e., the immediate neighbors' sum of local samples.
The remaining entries are computed through the following recursion:
\[
m^{t,a'\rightarrow a}_{h,i}=\sum_{a''\in\mathcal{C}_{a'}^{t}, \; a''\neq a } m^{t-1,a''\rightarrow a'}_{h-1,i},
\]
for $h\in\{2,\dots, d\}$ and $i\in\{1,2\}$. 
This captures information extending beyond immediate neighbors. For additional details on \belief{} see Algorithm~\ref{alg:belief} and \textcolor{black}{Fig.~\ref{fig:sketch_belief} in Appendix~\ref{app:D}}.

If $\mathcal G^t \cap \mathcal N_a^d$ is a tree, then $m^{t,a'\rightarrow a}_{h,1}$ contains the sum of all samples generated within time $t-h+1$ by agents $a'' \in \mathcal{G}^t$  at distance $h-1$ from $a'$  and distance $h$ from $a$, while $m^{t,a'\rightarrow a}_{h,2}$ contains the corresponding number of samples (the proof is by induction on $h$).
Agent $a$ can estimate its mean as:
\begin{equation}
\label{e:estimator_belief}
\hat{\mu}^t_a= \frac{\sum_{\tau=1}^{t}x^{\tau}_{a}+ \sum_{a'\in\mathcal{C}_{a}^{t}}{\sum_{h=1}^d
m^{t,a'\rightarrow a}_{h,1}}}{ t+ \sum_{a'\in\mathcal{C}_{a}^{t}}\sum_{h=1}^d
m^{t,a'\rightarrow a}_{h,2}}.    
\end{equation}
Under the local tree structure assumption, this corresponds to performing an empirical average over all the samples generated 
by all agents in $\mathcal G^t$ at distance $0\leqslant h \leqslant d$ from $a$ up to time~$t-h$.
If $\mathcal G^t \cap \mathcal N_a^d$ is not a tree, samples collected by a given agent~$a''$  may be included in messages received by~$a$ through different parallel paths (from~$a$ to~$a''$). % neighbors in $\mathcal C_a^t$}
As a result, these samples are erroneously counted multiple times in~\eqref{e:estimator_belief}.
The parameter~$d$ must be chosen to prevent this issue with high probability, as discussed in Sec.~\ref{sec:param_setting}.

\begin{algorithm}[h!]
   \caption{\belief{} over a Time Horizon $H$}
   \label{alg:belief}
   \begin{algorithmic}
      \STATE {\textbf{Input:}} $\,\,\,\, \mathcal{G} = \left(\mathcal{A},\, \mathcal{E} \right)$, $(D_a)_{a \in \mathcal{A}}$, $\varepsilon \in \mathbb R^+, 
    \delta \in (0,1]$
      \STATE {\textbf{Output:}} $\hat{\mu}_a, \, \forall a \in \mathcal{A}$
      with $\mathbb{P}\left( |\hat{\mu}_a - \mu_a| < \varepsilon \right) \geqslant 1- \delta$
      \STATE %$\bar{x}_a \leftarrow x_a^0 \sim D_a$,\,\, 
      $\mathcal{C}_a^0 \leftarrow \mathcal{N}_a, \forall a \in \mathcal A$
      % \WHILE{\textcolor{orange}{$\mathbb{P}\left( |\hat{\mu}_a - \mu_a| \leqslant \varepsilon \right) > 1 - \delta$}}
      \FOR{time $t$ in $\{1,..,H\}$} 
         \STATE In parallel for all nodes $a\in \mathcal A$
         \STATE Draw $x_a^t \sim D_a$
         \STATE $\bar{x}^t_a \leftarrow \frac{t-1}{t} \bar{x}^{t-1}_a + \frac{1}{t} x_a^t$
         \STATE Compute $\beta_\gamma(t)$ with Eq.~(\ref{eq:beta_mailard_main}) or Eq.~(\ref{eq:beta_cheby})
         %%%\STATE $\mathcal{C}_a^{t} \leftarrow \mathcal{C}_a^{t-1} $
         \FOR{neighbor $a'$ in $\mathcal{C}_a^{t-1}$}
            \STATE $d_\gamma^t (a,a') \leftarrow \left| \bar{x}^t_a - \bar{x}_{a'}^{t-1} \right| - \beta_\gamma(t) - \beta_\gamma(t-1)$
            \IF{$d_\gamma^t (a,a') > 0$}
                \STATE $\mathcal{C}_a^{t} \leftarrow \mathcal{C}_a^{t} \setminus \left\{ a' \right\}$
            \ENDIF
        \ENDFOR
        \FOR{neighbor $a'$ in $\mathcal{C}_a^{t}$}
            \STATE Compute $M^{t, a \to a'}$ and send it to $a'$
            %%\STATE Send  $M^{t, a \to a'}$ to $a'$
        \ENDFOR
        \STATE Wait for messages $M^{t, a' \to a}\;  \forall a' \in \mathcal C_a^t$ 
        \STATE $\hat{\mu}^t_a \leftarrow \frac{\sum_{\tau=1}^{t}x^{\tau}_{a}+ \sum_{a'\in\mathcal{C}_{a}^{t}}\sum_{h=1}^d m^{t-1,a'\rightarrow a}_{h,1}}{ t+ \sum_{a'\in\mathcal{C}_{a}^{t}}\sum_{h=1}^d m^{t,a'\rightarrow a}_{h,2}}  $
      \ENDFOR
   \end{algorithmic}
\end{algorithm}

Theorem~\ref{thm:Bcolme_thm2} presents the $(\epsilon, \delta)$ convergence result for \belief{}, which enjoys a speedup proportional to~$|\mathcal{CC}^d_a|$.

\begin{theorem}
\label{thm:Bcolme_thm2}
\textup{[Proof in \textcolor{black}{Appendix~\ref{appendix:th2}}]} Provided that $\mathcal{CC}^d_a$ is a tree, for any $\delta \in (0,1)$, employing \belief, 
we have:
\begin{align*}
	& \mathbb{P}\left(\forall t>\tau^B_a, \, |\hat \mu_a^t-\mu_a |<\varepsilon \right)\geqslant 1-\delta
\end{align*}
where $ \tau^B_a=\max\left[ \zeta_a^D+d, \frac{\widetilde{n}_{\frac{\delta}{2}}(\varepsilon)}{|\mathcal{CC}^d_a|}+d \right]$ and
 $\widetilde{n}_{\frac{\delta}{2}}(\varepsilon)= \Big\lceil-\frac{2\sigma^2}{\varepsilon^2} \ln\left( \frac{\delta}{4}(1- \mathrm{e}^{-\frac{ \varepsilon^2 }{\sigma^2}} ) \right) \Big \rceil$ for  SGD and as in Theorem~\ref{thm:consensus_epsilon_delta} for  BFMD.
\end{theorem}
Similar considerations to those for Theorem~\ref{thm:consensus_epsilon_delta} apply. The additional term~$d$ accounts for the delay introduced by the message-passing scheme.

\begin{corollary}
Let $\mathbb{P}(\mathcal{CC}^d_a \text{ is not a tree})=\delta'$, then for any $\delta \in (0,1)$, employing \belief, 
we have:
\[
\mathbb{P}\left(\forall t>\tau^B_a\, :\, |\hat \mu_a^t-\mu_a |<\varepsilon \right)\geqslant 1 -\delta-\delta'.
\]
\end{corollary}
\textcolor{black}{Theorem~\ref{tree-like-Gnr} (Appendix~\ref{app:Gnr})} provides upper bounds for~$\delta'$ when $\mathcal G$ is a random regular graph. In particular, as long we set  $d$ as in Proposition~\ref{prop:gnr} (Sec.~\ref{sec:param_setting}),  $\delta'$ converges to $0$ as the number of agents $|\mathcal{A}|$ increases.

\begin{table*}[ht!]
\centering
\begin{tabular}{l|c|c|c}
\toprule
 & \textbf{Per-agent space/time} & \multicolumn{2}{c}{\textbf{Convergence time}}\\
 & \textbf{complexity} & \textbf{sub-Gaussian} & \textbf{bounded $4$-th moment}\\
\midrule
ColME ({\footnotesize $r$ communications}) & $|\mathcal A|$ & $ \frac{1}{\Delta_a^2} \log \frac{|\mathcal A|}{\Delta_a \delta}  + \frac{|\mathcal A|}{r} + \frac{1}{|\mathcal C_a|} {\frac{1}{\varepsilon^2} \log \frac{1}{\delta \varepsilon^2}}  $ & $\frac{1}{\Delta_a^4} \frac{|\mathcal A| }{\delta} + \frac{|\mathcal A|}{r} +\frac{1}{|\mathcal C_a|}\frac{1}{\delta \varepsilon^4 }  $\\
C-ColME [Thm. \ref{thm:consensus_epsilon_delta}] ($\alpha_t = \frac{t}{t+1}$) & $r$ & $ -$ & $\frac{1}{\Delta_a^4} \frac{ |\mathcal{CC}_a| r}{\delta}  +\frac{1}{|\mathcal{CC}_a|}\frac{1}{\delta \varepsilon^4 }  $\\
B-ColME [Thm. \ref{thm:Bcolme_thm2}]  & $r  d$ & $ \frac{1}{ \Delta_a^2} \log \frac{ |\mathcal{CC}_a| r}{\Delta_a \delta} + d + \frac{1}{|\mathcal{CC}_a^d|} {\frac{1}{\varepsilon^2} \log \frac{1}{\delta \varepsilon^2}}  $ & $\frac{1}{\Delta_a^4} \frac{ |\mathcal{CC}_a| r}{\delta} + d +\frac{1}{|\mathcal{CC}_a^d|}\frac{1}{\delta \varepsilon^4 }  $\\
\bottomrule
\end{tabular}
\caption{Comparison of collaborative estimation algorithms. The convergence time is provided in order sense.}
\label{t:comparison}
\end{table*}

\subsection{Choice of the Graph and other Parameters}\label{sec:param_setting}

The selection of the graph~$\mathcal G(\mathcal{A},\mathcal{E})$ is crucial for the effectiveness of our algorithms.
% plays a pivotal role in the efficacy of our algorithms. 
Here we state the key desirable properties of $\mathcal G(\mathcal{A},\mathcal{E})$.
First,  Theorems~\ref{thm:consensus_epsilon_delta} and~\ref{thm:Bcolme_thm2} show that learning timescales, $\tau^B_a$ and $\tau^C_a$, decrease as the size of the collaborating agent groups, $\mathcal{CC}_a^d$ and $\mathcal{CC}_a$, increase.
Therefore, a highly connected graph is preferred to promote the formation of large clusters of agents belonging to the same similarity class
after the disconnection of inter-class edges. 
Second, %the complexities---both spatial and temporal---of 
the spatial and temporal complexities of \belief{} and \consensus{} are directly proportional to the agents' degree within the graph. Hence, we want the degree to be small and possibly uniform across the agents to balance computation across agents. 
%We observe that these two desiderata can be at odds, 
Note that %there is a trade-off between 
the first two criteria partially conflict, as a higher degree generally leads to larger groups~$\mathcal{CC}_a^d$ and~$\mathcal{CC}_a$, while a smaller degree ensures better spatial and temporal complexities.
A third criterion, specific to \belief{}, is that each agent's neighborhood should have a tree-like structure extending up to~$d$ hops, with~$d$ as large as possible. 

Considering these criteria, we opt for the class of \textit{simple} random regular graphs $\mathcal G_0(N,r)$.
These graphs are sampled uniformly at random from the set of all $r$-regular simple graphs with~$N$ nodes, i.e., graphs without parallel edges or self-loops, and in which every node has exactly~$r$ neighbors. Note that an even product $rN$ guarantees the set is not empty.
The class $\mathcal G_0(N,r)$ exhibits strong connectivity properties for small values of~$r$. Specifically, for any $r\geqslant 3$, the probability that the sampled graph is connected approaches one as~$N$ increases. Moreover, the sampled graph demonstrates a local tree-like structure with high probability (proof in \textcolor{black}{Appendix~\ref{app:Gnr}}). 
The choice of~$r$ (agents' degree) illustrates the trade-off discussed above between reducing complexity (low~$r$) and having large connected components (high~$r$).
A sensible rule is to select~$r$ sufficiently large to guarantee that most agents in the smallest (most critical) class belong to the same connected component. Consider a class including a fraction~$p_{k_a}$ of agents, \textcolor{black}{Table~\ref{tab:extinction_prob} in Appendix~\ref{app:Gnr}} shows the average fraction of agents in this class that is not connected to the main connected component as a function of~$r$. To keep this fraction below e.g.~$10^{-2}$ a good rule of thumb is $r=4/p_{k_a}$. %%{Add some more details on the choice of $r$ [Reviewer~1]}

A final key parameter for \belief{}  is the maximum distance $d$ over which local estimates from agents are propagated. This parameter must be carefully calibrated: it should be small enough to ensure that $\mathcal{CC}_a^d$, for a randomly chosen $a\in \mathcal{A}$,  has a tree-like structure with high probability. However, choosing a $d$ that is too small could unnecessarily restrict the size of $\mathcal{CC}_a^d$, thereby undermining the effectiveness of the estimation process (Theorem \ref{thm:Bcolme_thm2}). A comprehensive analysis of how the parameters~$r$ and~$d$ influence both the structure of $\mathcal{N}_a^d$ and the size of $\mathcal{CC}_a^d$ can be found in \textcolor{black}{Appendix~\ref{app:Gnr}}.
Here, we informally summarize the main result:
\begin{proposition}\label{prop:gnr}
By  selecting $d=\Big \lfloor \frac{1}{2}\log_{r-1} 
\frac{|\mathcal{A}|}{\log_{r-1}|\mathcal|A|} \Big \rfloor$ 
%,  for an arbitrary $\phi_0<1$,
the number of  nodes $a\in \mathcal{A}$, whose $d$-neighborhood is not a tree, is $o(|\mathcal{A}|)$ with a probability tending to 1 as $|\mathcal{A}|$ increases.
For the same $d$ and $r \in \Theta\left(\log(1/\delta)\right)$, $|\mathcal{CC}_a^d|$ %(with $d$ chosen as above) 
is in  $\Omega( |\mathcal{A}|^{\frac{1}{2}-\phi})$ for any arbitrarily small  $\phi>0 $ with probability arbitrarily close to~1. 
\end{proposition}

Finally, for \consensus{}, the consensus matrix $W$ could be chosen to minimize the second largest module $\lambda_{2,c}$ of the eigenvalues of each block $\prescript{}{c}W$ in order to minimize the bound in Theorem~\ref{thm:fourth_moment_main_text}. This optimization problem has been studied by \citet{boyd} and requires in general a centralized solution. In what follows, we consider the following simple, decentralized configuration rule: $(W_t)_{a,b} = \frac{1}{\max\{|\mathcal C_{a}^t|, |\mathcal C_{b}^t|\}+1}, \forall b \in \mathcal C_a^t$ and $(W_t)_{a,a} = 1- \sum\limits_{b \in \mathcal C_a^t} (W_t)_{a,b}$, making~$W$ symmetric and doubly stochastic.

\section{Algorithms' Comparison}
\label{s:comparison}
Table~\ref{t:comparison} presents a comparative analysis of the three algorithms: \colme{}, \consensus{}, and \belief{}. For a fair comparison, we consider a variant of \colme{}, where each agent can communicate with~$r$ agents at each time~$t$, so that all three algorithms incur the same communication overhead.

The second column of Table~\ref{t:comparison} outlines the space and time complexities of the algorithms. Notably, even when~$r$ and~$d$ are allowed to increase logarithmically with the number of agents $|\mathcal A|$, \belief{} retains its efficiency advantage over \colme{}. \consensus{} demonstrates even greater improvements, further reducing the per-agent burden compared to the savings achieved by \belief{}.

The third and fourth columns detail the characteristic times required to achieve $(\epsilon, \delta)$ convergence for the estimates generated by the three algorithms, considering both sub-Gaussian local data distributions and distributions with bounded fourth moment. The characteristic times correspond to $\tau_a$, $\tau^C_a$, and $\tau^B_a$ in % Theorems~\ref{thm:colme_thm2}
\textcolor{black}{Theorem~7 in Appendix~\ref{sec:colme-perf}}, Theorem~\ref{thm:consensus_epsilon_delta}, and~\ref{thm:Bcolme_thm2}, respectively. The table reports their asymptotic behavior as the number of agents $|\mathcal A|$ increases ignoring logarithmic factors. The detailed derivations of these results are provided in \textcolor{black}{Appendix~\ref{app:order_sense}}.

Three factors contribute to the characteristic times. The first factor is the time required to correctly identify potential collaborators. For \colme{}, this involves each agent classifying the other $|\mathcal A|-1$ agents, leading to a term that scales as $\log |\mathcal A|$ or $|\mathcal A|$, depending on the assumed properties of the local distribution. 
For \belief{} and \consensus{}, $|\mathcal A|$ is replaced by $|\mathcal{CC}_a|r$, which represents an upper bound on the number of connections the agents in $\mathcal{CC}_a$ may have initially established with agents from different classes. This substitution may not be immediate, as one might initially expect the relevant scale to be simply~$r$. However, this adjustment accounts for the potential ripple effect of classification errors: a mistake by any agent $a$ can impact the estimates of all agents within the same connected component~$\mathcal{CC}_a$. 

The second factor contributing to the characteristic times is the time each agent needs to collect all relevant information. For \colme{}, this time is proportional to $|\mathcal A|/r$, as an agent queries all other agents. For \belief{}, the time is specifically tied to $d$, the maximum number of hops messages propagate. Notably, this term does not appear for \consensus{}, as it is dominated by the final term.

The final term represents the time needed for accurate mean estimation after collaborators have been identified, highlighting the benefits of collaboration.
In \colme{}, the collaboration's benefit is particularly striking, as all agents within the same class work together to improve their estimates. This collective effort effectively reduces the convergence time by a factor proportional to the size of the collaborating group, $|\mathcal C_a|$.  For \belief{} and \consensus{}, although the speed-up remains proportional to the number of collaborating agents, the actual numbers of collaborators, $|\mathcal{CC}_a^d|$ for \belief{} and $|\mathcal{CC}_a|$ for \consensus{}, are in general smaller.

In conclusion, while \colme{} potentially offers the most accurate estimates, it requires longer convergence times and greater memory and computational resources.
In contrast, \belief{} and \consensus{} present more efficient alternatives, achieving faster convergence with reduced per-agent resource demands. However, this efficiency may come at the expense of the maximum attainable accuracy. The next section quantifies this trade-off experimentally.

\begin{figure*}[ht]
    \centering
    \begin{tabular}{cc}
        \includegraphics[width=0.45\textwidth]{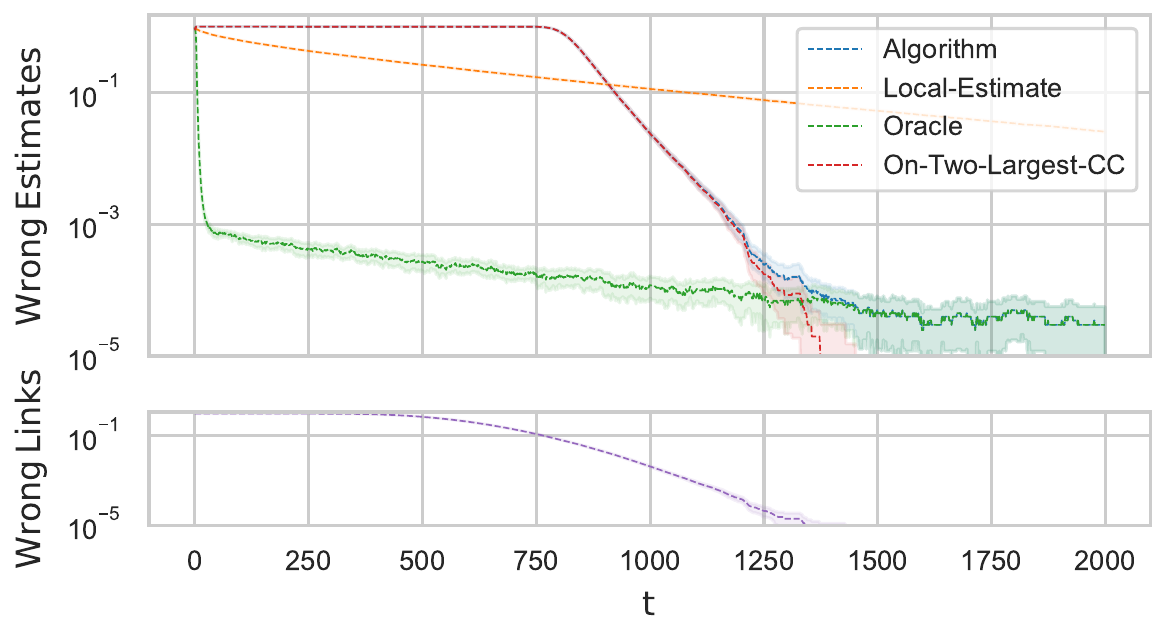} \hspace{2.5mm} & \hspace{2.5mm} \includegraphics[width=0.45\textwidth]{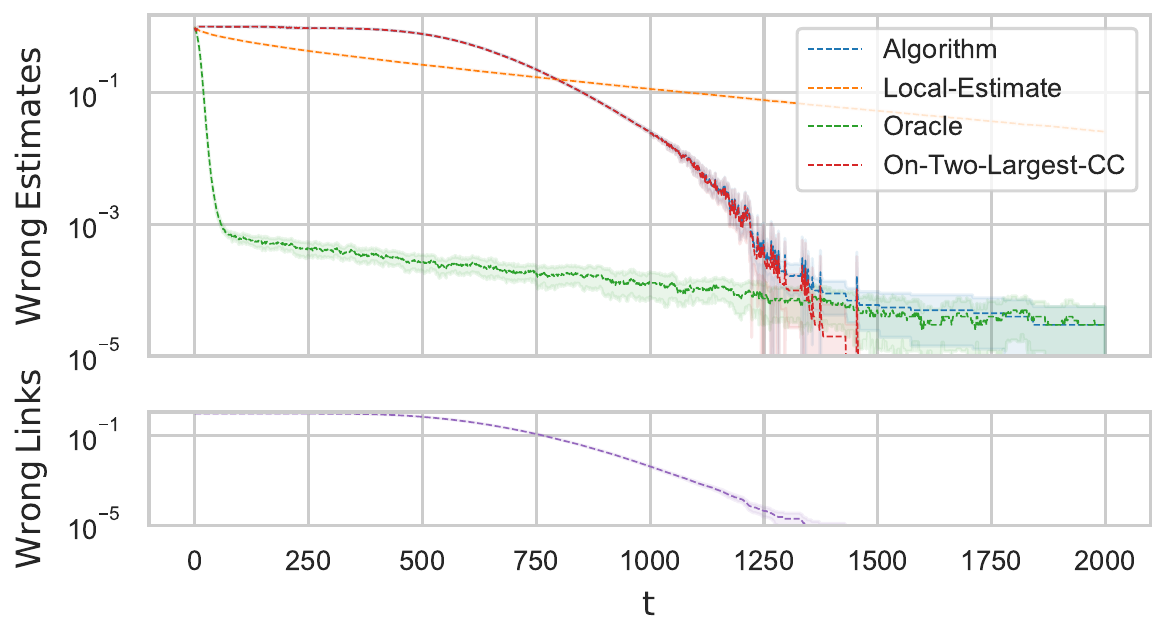}  \\
        \footnotesize (a) \belief{} \hspace{0mm} & \hspace{5mm} \footnotesize (b) \consensus{}
    \end{tabular}
    \caption{%% Fraction of agents whose estimates deviate from the true value by more than $\epsilon$ (top) and a fraction of neighbors that have not yet been identified as belonging to a different class (bottom) for \belief{} (left) and \consensus{} (right). Averages and 95\% confidence intervals were computed over 20 realizations.
%%Fraction of agents with estimates deviating by more than $\epsilon$ from the true value (top) and fraction of \textit{wrong links} (bottom) for \belief{} (left) and \consensus{} (right), averaged over 20 realizations with 95\% confidence intervals.
Fraction of agents with estimate deviates  by more than $\epsilon$ from the true value, i.e., $| \{ a \in \mathcal{A} : |\hat{\mu}_a^t - \mu_a| > \varepsilon \} | / | \mathcal{A} |$ (top)  and fraction of \textit{wrong links} (bottom) for \belief{} (a) and \consensus{} (b), over 20 realizations with 95\% confidence intervals.}       \label{fig:scalable_performance}
\end{figure*}

\section{Numerical Experiments} %over $G_0(N,r)$
\label{sec:pe}

We evaluate the proposed algorithms on the class $\mathcal{G}_0(N,r)$ of simple regular graphs (see Sec.~\ref{sec:param_setting}). In this setting, each agent connects to~$r$ other agents chosen uniformly at random in~$\mathcal{A}$. This setup also provides the tree-like local structure required for \belief{}.
Agents belong to one of two classes, associated with Gaussian distributions $D_1 \sim \mathcal{N}(\mu_1=~0, \sigma^2=4)$ and $D_2 \sim \mathcal{N}(\mu_2=1, \sigma^2=4)$. %Initially, 
Each node is assigned to one of the two classes with equal probability.
Unless otherwise stated, in the experiments $|\mathcal{A}|=N=10000$, $r=10$, $d=4$, $\varepsilon=0.1$, $\delta=0.1$, and  $\beta_\gamma (n)$ as in \eqref{eq:beta_mailard_main}. \textcolor{black}{In Appendix~\ref{app:multidim} and~\ref{sec:add_performance}},
we provide additional experiments for the multidimensional case and varying the system's parameters.

Figure~\ref{fig:scalable_performance} showcases the performance of \belief{} and \consensus{} using two key metrics: the fraction of agents with \textit{incorrect estimates} ($\hat{\mu}_a^t$ more than
$\varepsilon$ away from the true mean~$\mu_a$), and the fraction of \textit{wrong links} still in use (a wrong link connects agents from different classes). 
We compare our algorithms against two benchmarks. 
The first benchmark has each agent independently relying on its \emph{local} estimate~$\bar{x}^t_{a,a}$. In the second benchmark, an \textit{oracle} provides each agent with precise knowledge of which neighbors belong to the same similarity class (i.e., $\mathcal{C}_a^t = \mathcal{C}_a \cap \mathcal{N}_a, \forall a, t$).
The figure reveals that \belief{} has a longer transient phase but then exhibits a slightly steeper convergence than \consensus{}. Notably, \belief{}'s estimates show no apparent improvement until about~90\% of the wrong links have been removed, whereas \consensus{}'s estimates begin to improve as soon as the first edges are eliminated. This phenomenon can be explained as follows. In \belief{}, the estimates at agent~$a$ are not influenced by the removal of some wrong links as long as its $d$-hop neighborhood $\mathcal N_a^d$ remains unchanged. For instance, a given node~$a' \notin \mathcal C_a$ is removed from $\mathcal N_a^d$ only when \emph{all} paths of length at most~$d$ between agent~$a$ and agent~$a'$ are eliminated. In contrast, in \consensus{}, agent~$a'$ contributes to the weighted estimate at agent~$a$ with a weight equal to the sum over all paths between~$a$ and~$a'$ of the product of the consensus weights along the path. As paths are progressively removed, the negative impact of~$a'$ on agent~$a$'s estimate is gradually reduced. However, once all wrong links are removed, \belief{} benefits from its estimates being computed solely on agents belonging to the same class, while \consensus{} requires some additional time for the effect of past estimates to fade away.

\begin{figure}[ht!]
    \centering
    \includegraphics[width=0.85\columnwidth]{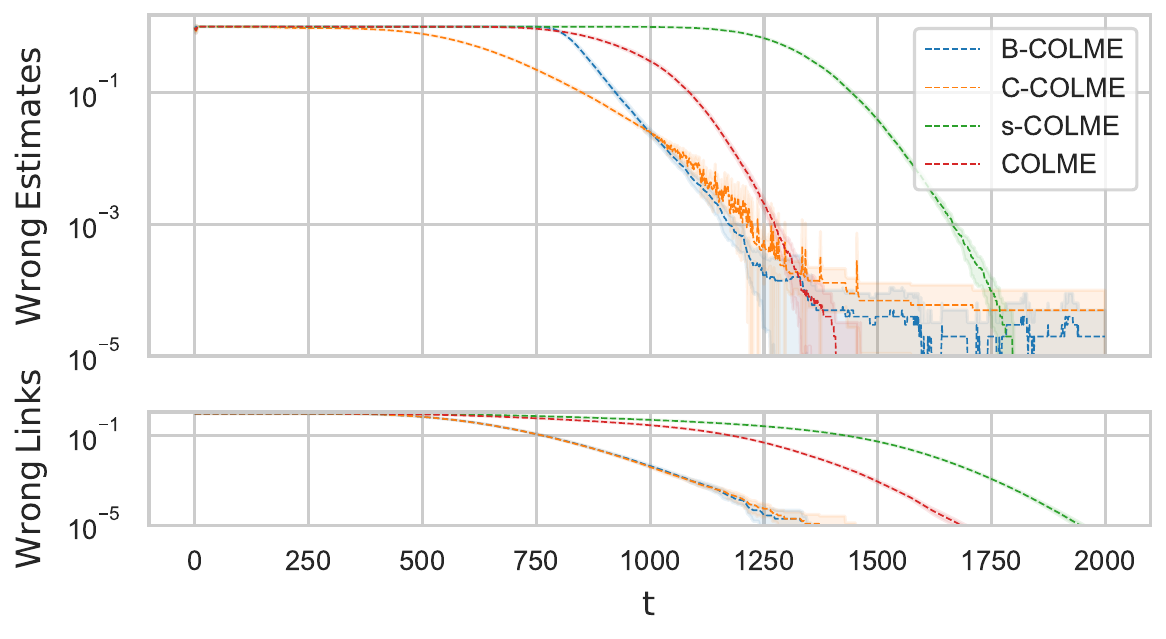}
    \caption{Comparison of our algorithms and two versions of \colme{}, %on a $G_0(N=10000,r=10)$ 
    over 10 realizations. % (\textcolor{cyan}{same format as Fig.~1}).
}
    \label{fig:comparison}
\end{figure}

We also compare the proposed algorithms with \colme{} and a simplified version (s-\colme) where the optimistic distance~$d^t_{\gamma}(a,a')$ is recomputed only for the~$r$ agents queried at time~$t$, achieving an~$\bigO(r)$ per-agent computational cost (the memory cost remains~$\bigO(|\mathcal A |)$). As predicted by the theoretical analysis, \belief{} and \consensus{} are faster than \colme{}, but at the cost of a higher asymptotic error because agent~$a$ collaborates only with the smaller group of nodes in~$\mathcal{CC}_a^d$ for \belief{}, and~$\mathcal{CC}_a$ for \consensus{}. \colme{} pays for this asymptotic improvement with a~$\bigO(|\mathcal A|)$ space-time complexity per agent, impractical for large-scale systems. Note that s-\colme{} improves \colme{}'s complexity at the cost of a much slower discovery of same-class neighbors. %%, significantly affecting the estimation process.

While we focused on %%the fundamental problem of 
online mean estimation, our approach can be adapted to decentralized federated learning. To illustrate this possibility, we adapt the consensus-based decentralized federated learning algorithm, by letting agents progressively exclude neighbors they identify as belonging to a different class. The cosine dissimilarity of agents' updates, the same metric used in {ClusteredFL}~\cite{ghoshEfficientFrameworkClustered2020,chenJointLearningCommunications2021}, replaces the optimistic distance~$d^t_{\gamma}(a,a')$ (details in \textcolor{black}{Appendix~\ref{sec:add_ML_exp}}).
Figure~\ref{fig:ML-exp} shows the performance of our \textit{decentralized FL over a dynamic graph}~(FL-DG) with~$|\mathcal{A}|=100$ agents initially organized over a complete graph.
Two different distributions are obtained from MNIST~\cite{mnist} by swapping/maintaining some labels and each client progressively receives new data samples from one of the two distributions. As the graph is progressively split in two clusters of clients belonging to the same class, each agent's model benefits from cooperating only with similar clients and it achieves a higher accuracy.
  %As soon as the wrong edges (in this case a single edge between agent 1 and agent 4) are identified, each community learns a personalized model suited for the local distribution, achieving a higher accuracy.

  \begin{figure}[h!]
      \centering
      \includegraphics[width=0.85\linewidth]{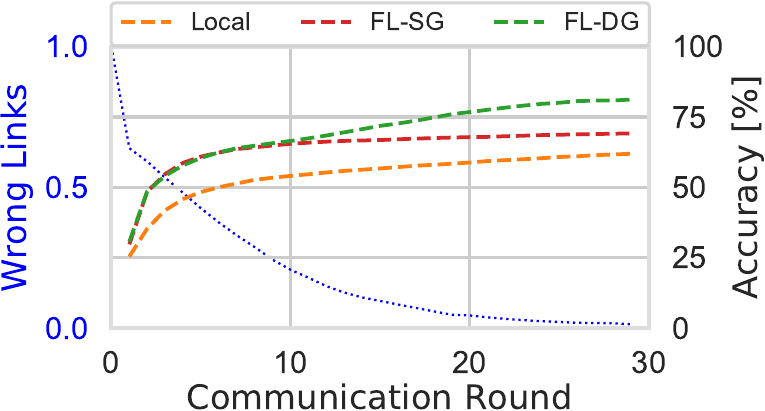}
      \caption{Accuracy of a local model (Local), a decentralized FL over a static graph (FL-SG), and our approach over a dynamic graph (FL-DG). We also show the fraction of links between classes (\textit{wrong links}) over time for FL-DG. %{\color{green} usiamo wrong links qui, incorrect edges nel testo e la formula nelle altre figure. Io sarei per usare una sola espressione testuale (la stessa) dappertutto. Per esempio wrong nodes' estimates, wrong links}
}
      \label{fig:ML-exp}
  \end{figure}

\section{Conclusions}
In this paper, we introduced \belief{} and \consensus{}, two scalable and fully distributed algorithms for collaborative local mean estimation. We thoroughly evaluated their performance through both theoretical and empirical analyses.
Additionally, we adapted our approach for personalized federated learning, applying it to the task of handwritten digit recognition using the MNIST dataset. 

This work points to several future research directions. Here we have allowed agents only to sever existing connections, but not to establish new ones.
Investigating scenarios where agents can rewire their connections to communicate with new agents outside their original neighborhood would be an interesting extension.
%It would be interesting to investigate the case where agents can rewire their connection to new agents that are not in their original neighborhood. 
Additionally, we assumed that agents are partitioned into similarity classes, with agents in the same class generating data with identical true mean. Extending our approach to accommodate more general scenarios, where each agent generates data with potentially different true mean, would be a valuable avenue for further exploration.

\section*{Acknowledgements}
This research was supported in part by the European Network of Excellence dAIEDGE under Grant Agreement Nr. 101120726 and by the Groupe La Poste, sponsor of the Inria Foundation, in the framework of the FedMalin Inria Challenge. It was also funded by the Nokia-Inria challenge LearnNet, and by the French government National Research Agency (ANR) through the UCA JEDI (ANR-15-IDEX-0001), EUR DS4H (ANR-17-EURE-0004), and the 3IA Côte d’Azur Investments in the Future project (ANR-19-P3IA-0002).

Computational resources provided by hpc@polito,which is a project of Academic
Computing within the Department of Control and Computer Engineering at the Politecnico
di Torino (http://hpc.polito.it).

\bibliography{tidy_resources}

\appendix

\appendix
\onecolumn

\section{Appendix A - \colme's Guarantees}
\label{sec:colme-perf}

We provide the two main theoretical results from~\citet{colme} [Theorem~1 and Theorem~2] which were proven under the assumption that the distributions $\{D_a, a \in \mathcal A\}$ are sub-Gaussian. Recall that a distribution $D$ is sub-Gaussian with parameter $\sigma^2$ if  $\forall \lambda\in \mathbb{R}$, $\log \mathbb{E}_{x \sim D} \exp(\lambda(x-\mu))\leqslant\frac{1}{2}\lambda^2\sigma^2$.
The proofs of the Theorems presented in this Appendix can be found in~\cite{colme}. 

Select the amplitude of the confidence intervals $I_{a,a}$ and $I_{a,a'}$, centered around the sample means, as:
\begin{equation}\label{eq:beta_mailard}
    \beta_\gamma(n):= \sigma \sqrt{\frac{2}{n}\left(1+ \frac{1}{n}\right)\ln(\sqrt{(n+1)}/\gamma )}
\end{equation}
and let $n^{\star}_{\gamma}(x)$ denote the minimum number of samples that are needed to ensure $\beta_\gamma(n)< x$, i.e.,
$n^{\star}_{\gamma}(x)= \lceil \beta^{-1}_\gamma(x)\rceil$. We also use $n^{\star}_{\gamma}(a,a')$ for $n^{\star}_{\gamma}(\Delta_{a,a'}/4)$ and $n^{\star}_{\gamma}(a)$ for $n^{\star}_{\gamma}(\Delta_{a}/4)$ where  $\Delta_a: = \min_{a'\in \mathcal{A}\setminus \mathcal{C}_a }\Delta_{a,a'}$; these values denote the minimum number of samples to distinguish (with confidence $1-2 \gamma$) the true mean of agent $a$ from the true mean of agent $a'$, and from the true mean of any other agent, respectively.

The first result (Theorem~\ref{thm:colme_thm1})  provides a bound on the time needed to ensure that a randomly chosen agent $a \in \mathcal{A}$ correctly identifies its similarity class, i.e, $\mathcal{C}_a^t = \mathcal{C}_a$, with high probability.\footnote{In this paper events that occur with a probability larger than $1-\delta$ are said to occur with high probability (w.h.p.).}

\begin{theorem}\label{thm:colme_thm1}
    \citep[Theorem~1]{colme} Assume distributions $\{D_a, \forall a \in \mathcal A\}$ are sub-Gaussians with parameter~$\sigma^2$. For any $\delta \in (0,1)$, and with $\gamma = \frac{\delta}{4|\mathcal{A}|}$, employing \colme, we have:
    \begin{equation}
        \mathbb{P}\left(\exists t>\zeta_a: \mathcal{C}_a^t \neq \mathcal{C}_a \right) \leqslant \frac{\delta}{2},
    \end{equation}
    %\textcolor{cyan}{C'è un piccolo error nella loro derivazione a pagina 16, dove derivano un bound per una metrica senza valore assoluto e poi usano quel bound per la metrica col valore assoluto.}
    \begin{equation*}
        \textrm{with }\zeta_a=n_{\gamma}^{\star}(a)+|\mathcal{A}|-1-\sum_{a' \in\mathcal{A} \backslash \mathcal{C}_a}  \mathds{1}_{\left\{n_{\gamma}^{\star}(a)>n_{\gamma}^{\star}(a,a')+|\mathcal{A}|-1\right\}}.
    \end{equation*} %and $\gamma= \frac{\delta}{\textcolor{black}{4}A} $.
\end{theorem}

Theorem~\ref{thm:colme_thm1} shows that the time $\zeta_a$ required by each agent $a$ to identify (with high probability) which other agents are in the same similarity class can be bounded by the sum of two terms. The first term, $n^\star_{\gamma}(a)$, is an upper-bound for the number of samples needed to 
conclude,  with probability larger or equal than $1-2\gamma$, if the true means of two agents differ by the minimum gap $\Delta_a$.
%distinguish,  with probability larger or equal  than $1-2\gamma$, two agents whose true means differ by the minimum gap $\Delta_a$.
The additional term corresponds to the residual time required to acquire the estimates from other agents. It can be shown that $n^{\star}_{\gamma}(a)$
%G: 24/4/2024 $n^{\star}_{\gamma}(a,a')$ 
grows at most as $\log |\mathcal{A}|$, then $\zeta_{a} \in \bigO(|\mathcal{A}|)$ and for large systems ($|\mathcal{A}| \gg n^{\star}_{\gamma}(a)$), the need to query all agents at least once becomes the dominant factor.

Turning our attention to the estimation error, it holds:
\begin{theorem}
	\label{thm:colme_thm2}
	\citep[Theorem~2]{colme}  Assume distributions $\{D_a, \forall a \in \mathcal A\}$ are sub-Gaussians with parameter~$\sigma^2$. For any $\delta \in (0,1)$, and with $\gamma = \frac{\delta}{4|\mathcal{A}|}$, employing \colme, we have:
	\begin{equation}
		\mathbb{P}\left(\forall t>\tau_a, |\hat \mu_a^t-\mu_a |<\varepsilon \right)\geqslant 1-\delta,\\
	\end{equation}
        with $\tau_a=\max\left[ \zeta_a, \frac{n_{\frac{\delta}{2}}^{\star}(\varepsilon)}{|\mathcal{C}_a|}+ \frac{|\mathcal{C}_a|-1}{2} \right]$.
\end{theorem}
Theorem \ref{thm:colme_thm2} admits a straightforward explanation. Provided that agent $a$ has successfully estimated its similarity class at time~$t$ (i.e., $\mathcal{C}_a^t=\mathcal{C}_a$), the error in the mean estimation will depend only on the available number of samples of agents in $\mathcal{C}_a$, used for the computation of $\hat\mu_a^t$.
Now, a number of samples equal to $n_{\frac{\delta}{2}}^{\star}(\varepsilon)$
is sufficient to ensure that $\mathbb{P}(|\hat \mu_a -\mu_a |>\varepsilon)<\delta / {2}$.
For agent $a$  such  number of samples is surely available at time $t\geqslant t^*_a=\frac{n_{\frac{\delta}{2}}^{\star}(\varepsilon)}{|\mathcal{C}_a|}+ \frac{|\mathcal{C}_a|-1}{2}$ 
where the second term is needed to take into account the effect of the delay introduced by the round-robin scheme. Applying the union bound, we can claim that  whenever  $t\geqslant \max(\zeta_a, t^*_a)$ w.p.~$1-\delta$ both $\mathcal{C}_a^t=\mathcal{C}_a$
and $ |\hat \mu_a^t - \mu_a |\leqslant \varepsilon$ hold.

\newpage

\section{Appendix B - Proof of Theorem \ref{cor:Dcolme_thm1}}
\label{appendix:th1}

In this Appendix %we give a more formal and extended presentation of the results of \textcolor{black}{Section \ref{sec:B_and_C}.} %\ref{sec:scalable_algs}
we provide the proof of Theorem \ref{cor:Dcolme_thm1} %and \textcolor{black}{\ref{thm:Bcolme_thm2}}. 
and, as a side result, we derive Equation \eqref{eq:beta_cheby}.
For the sake of clarity, we repeat the statement of the theorems. 

The first theoretical result provides a bound on the probability that
%the nodes of the graph 
the nodes in the connected component $\mathcal{CC}_a$ of a certain node $a\in \mathcal{A}$
misidentify their \textit{true} neighbors $\mathcal{C}_{a'} \cap \mathcal{N}_{a'}, \, \forall a' \in \mathcal{CC}_a$. We remark that, unlike \colme, the \textit{goodness} of an estimate of our \textit{scalable} algorithms depends not only on the ability of a given node to correctly identify its \textit{true} neighborhood but also on the neighborhood estimates of all other nodes. Communication between nodes (i.e., message passing or consensus mechanism) makes error propagation possible within a connected component. Therefore, when bounding the probability of incorrect neighborhood estimation we have to take a network perspective, which influences the choice of $\gamma$.

\begin{theorem}[Incorrect neighborhood estimation]
    Considering an agent~$a$ picked arbitrarily in~$\mathcal{A}$, for any $\delta \in (0,1)$, employing \belief\ or \consensus\, we have:
    \begin{equation*}
    	\mathbb{P}\left(\exists t>\zeta_a^D,  \exists a' \, \in \mathcal{CC}_a: \mathcal{C}_{a'}^t  \neq \mathcal{C}_{a'} \cap {\mathcal N_{a'}}\right) \leqslant \frac{\delta}{2},
    \end{equation*}
    %\text{ with } \zeta_D=n^{\star}(\gamma)+1 \text{ and }  \gamma= \frac{\delta}{8| \mathcal{E}|}.
    with $\zeta_a^D=n^{\star}_\gamma(a)+1$ and  $\gamma= \frac{\delta}{4 r | \mathcal{CC}_a|}$.
\end{theorem}

\begin{proof}

The proof involves establishing a series of intermediate results that finally enable us to prove the theorem. We outline the steps of the proof below:
\begin{itemize}
    \item First, we show that under the two proposed confidence interval parametrizations $\beta_\gamma(n)$ (Eq. (\ref{eq:beta_mailard}) and (\ref{eq:beta_cheby})) and considering the sample mean $\bar{x}$ computed over $n$ samples, the probability that the \textit{true} mean~$\mu$, falls within the confidence interval $\left[ \bar{x} -\beta_\gamma(n), \bar{x}+\beta_\gamma(n) \right]$ for every $n\in\mathbb{N}$, is at least $1-2\gamma$. This is completely equivalent to saying that the probability that the \textit{true} mean value is outside the confidence interval for some $n\in\mathbb{N}$  is bounded above by $2\gamma$ (Lemma~\ref{lemma:ci} and Proposition~\ref{cheby}).
    
    \item Second, we remark that Lemma \ref{lemma:ci} and Proposition \ref{cheby} consider a \textit{local} perspective, taking one particular estimate $\bar{x} $ (i.e., $ \bar{x}_{a,a'}^t$), together with its number of samples $n $ (i.e., $ n_{a,a'}^t$). We extend this result by proving that the true mean falls within the confidence interval $I_{a,a'}$ {(with high probability)} for all the nodes $a' \in \mathcal{CC}_a$ in the connected component, for all records retrieved locally from the neighbors ($a'' \in \mathcal{N}_a$), and for every discrete time instant $t\in \mathbb{N}$ (Lemma \ref{lemma:beta_confidence}). We will refer to this event as $E$. This result provides a \textit{global} perspective over the entire connected component $\mathcal{CC}_a$
    %of the entire collaborative network.
    It is important to observe that only when the \textit{true} value is in $I_{a,a'}$, we can provide guarantees about the correct estimation of the similarity class.
    
    \item Then, we consider the \textit{optimistic distance} $d_\gamma^t(a,a'')$ for which we show that, conditionally over $E$,
    whenever it takes on strictly positive values (i.e., $d_\gamma^t(a,a'') > 0$), the neighbor $a''$ does not belong to the same similarity class $\mathcal{C}_a$ as agent $a$ (Lemma \ref{lemma:class_membership}). As a byproduct, we also derive the minimal number of samples $n^\star_\gamma(a)$ needed to correctly decide whether a neighboring node belongs to the same equivalence class.
    
    \item At last, by combining previous results we can easily obtain the claim.
\end{itemize}

\begin{lemma}[Interval parametrization]
\label{lemma:ci}
    For any $\gamma \in (0,1)$,
    setting $\beta_\gamma (n) = \sigma \sqrt{\frac{2}{n}\left( 1 + \frac{1}{n} \right) \ln{\frac{\sqrt{n+1}}{\gamma}}}$ (if the random variables~$x_t$ are sub-Gaussian with parameter $\sigma$) or
    $\beta_\gamma(n) = Hn^{-\alpha}, \; \text{with} \; \alpha<\frac{1}{4} \; \text{and} \;  H=\sqrt[4]{\frac{(\kappa+3)\sigma^4\zeta(2-4\alpha)}{\gamma}} $ (if $x_t$ have the first 4 polynomial moments bounded), it holds:
    \begin{equation}
    \label{eq:param_bound}
        %% \mathbb{P} \left( \forall n \in \mathbb{N}, \left| \bar{x} - \mu \right| \textcolor{black}{<} \beta_\gamma (n) \right) \geq 1 - 2 \gamma % (version 'complementata' che definisce un lower bound per la probabilità di error)
        \mathbb{P} \left( \exists n \in \mathbb{N}, \left| \bar{x} - \mu \right| \geq \beta_\gamma (n) \right) \leqslant 2 \gamma
    \end{equation}
\end{lemma}
\begin{proof}

    \textbf{($\sigma$-sub-Gaussian $x_t$)} We start from the theoretical guarantees on the parametrization (Eq. \ref{eq:beta_mailard}) of the confidence intervals from \cite{maillard} [\textit{Lemma 2.7}]: %\textit{Lemma 2.7 (Time-uniform sub-Gaussian concentration)} :
    Let $\{x_{t}\}_{t\in \mathbb{N}}$ a sequence of independent real-valued random variables, where for each $t$, $x_{t}$ has mean $\mu_{t}$ and is $\sigma_{t}$-sub-Gaussian. Then, for all $\gamma \in (0,1)$ it holds:
    \begin{equation}
    \begin{split}
        \mathbb{P}\left( \exists n \in \mathbb{N}, \, \sum_{t=1}^{n} (x_{t} - \mu_{t}) \geq \sqrt{2 \sum_{t=1}^{n} {\sigma_{t}^2} \left( 1 + \frac{1}{n}\right) \ln{\frac{\sqrt{n+1}}{\gamma}} } \right) \leqslant \gamma\\
        \mathbb{P}\left( \exists n \in \mathbb{N}, \, \sum_{t=1}^{n} (\mu_{t} - x_{t}) \geq \sqrt{2 \sum_{t=1}^{n} {\sigma_{t}^2} \left( 1 + \frac{1}{n}\right) \ln{\frac{\sqrt{n+1}}{\gamma}} } \right) \leqslant \gamma\\
    \end{split}
    \end{equation}

    Our sequence of random variables can i) either correspond to the samples $x_t$ each node $a \in \mathcal{A}$ generates at each discrete time instant $t$, which is i.i.d., with mean $\mu_a$ and $\sigma$-sub-Gaussian (by assumption), or ii) the truncated sequence up to $n_{a,a'}$ the node learns by querying its neighbors, possessing the same properties. Indeed, recall that for each \textit{locally} available estimate $\bar{x}_{a,a'}$, each node keeps also the number of samples over which that estimate has been computed $n^t_{a,a'}$. For ease of notation, we will drop the subscripts and superscripts, which for the sake of the lemma are superfluous. Being $\mu_{n'}=\mu$ and $\sigma_{n'}=\sigma$ constant in our case, it is immediate to write (considering just the first inequality for compactness):

    \begin{equation*}
        \mathbb{P}\left( \exists n \in \mathbb{N}, \, \sum_{t=1}^n (x_{t}) - n \mu \geq \sqrt{2 n\sigma^2 \left( 1 + \frac{1}{n}\right) \ln{\frac{\sqrt{n+1}}{\gamma}} } \right) \leqslant \gamma
    \end{equation*}

    Dividing by $n$ both sides we obtain the sample mean $\bar{x} = \frac{1}{n} \sum_{t}^{n} x_{t}$ (in place of the summation) and the parametrization of confidence interval $\beta_\gamma (n)$, as we have introduced in Sec.~\ref{sec:model} (which we restate here for completeness):

    \begin{equation} \label{eq:beta}
        \beta_\gamma (n) := \sigma \sqrt{\frac{2}{n}\left( 1 + \frac{1}{n} \right) \ln{\frac{\sqrt{n+1}}{\gamma}}} 
    \end{equation}
    By a simple substitution, we get:
    \begin{equation} 
        \mathbb{P}\left( \exists n \in \mathbb{N}, \, \bar{x} - \mu \geq \beta_\gamma (n) \right) \leqslant \gamma 
    \end{equation}

    Lastly, recall we are interested in the probability of the true mean being in the bilateral interval bounded by the given parametrization $\beta_\gamma$. Hence we can bound this probability with $2 \gamma$:

    \begin{equation*}
        \mathbb{P} \left( \exists n \in \mathbb{N}, \left| \bar{x} - \mu \right| \geq \beta_\gamma (n) \right) \leqslant 2 \gamma
    \end{equation*}

     This proves the first part of the lemma. Moreover, considering the complementary event and noting that $\mathbb{P}(e) \leqslant 2\gamma \iff \mathbb{P}(e^c) \geq 1 - 2\gamma$, where $e$ is a generic event and $e^c$ its complementary, it is immediate to obtain the lower bound (complementary to Eq. \ref{eq:param_bound}):
     \begin{equation}
         \label{eq:lower_bound_app}
         \mathbb{P} \left( \forall n \in \mathbb{N}, \left| \bar{x} - \mu \right| {<} \beta_\gamma (n) \right) \geq 1 - 2 \gamma
     \end{equation}
     
     Note also that the probabilistic confidence level $\gamma$ can be considered as a function of $\delta$. By choosing appropriately the function $\gamma = f(\delta)$ it is possible to provide the desired level $\delta$ for the PAC-convergence of a given algorithm.

\medskip
\medskip

    \textbf{($x^t$ with the first 4 bounded polynomial moments)}.
    Now we release the assumption that $x_t$ are extracted from sub-Gaussian distributions. 
   We only assume that  $\mathbb{E}[(x_{a}^t-\mu_a)^4]\leqslant \mu_4$ for any $a\in \mathcal{A}$.

%\subsection{Derivation of \eqref{eq:beta_cheby}} 

   We start recalling the class of concentration inequalities which generalize the classical Chebyshev inequality:

   \begin{proposition}\label{cheby}
   Given a random variable $X$ with average $\mu<\infty$ and finite $2i$-central  moment
    $\mathbb{E}[(X-\mu)^{2i}]= \mu_{(2i)}(X)$   for any $b>0$ we have:
    \[
      \mathbb{P}(|X-\mu|>b)< \frac{\mu_{(2i)}(X)}{(b)^{2i}}   
    \]
    Moreover:
  \[
      \mathbb{P}(X>b)< \frac{\mathbb{E}[Y^{2i}]}{(b)^{2i}}.
    \]
    
    \end{proposition}
    Applying previous inequality to our estimate $\bar{x}=\sum_{t=1}^n x^t$  for  the case $i=2$ we get:
    \[
     \mathbb{P}(|\bar{x}-\mu|>\beta(n) )< \frac{\mu_{(4)}(\bar{x}) }{(\beta_\gamma(n))^4}.   
    \]
   Now  observe that 
     \begin{align}
          \mu_{(4)}(\bar{x}):= & \mathbb{E}\left[\frac{1}{n^4}\sum_{t=1}^n (x^t-\mu) \sum_{\tau=1}^n (x^\tau-\mu)
         \sum_{\theta=1}^n (x^\theta-\mu)  \sum_{\phi=1}^n (x^\phi-\mu) \right] \nonumber \\
       = & \frac{1}{n^4} \sum_{t=1}^n \sum_{\tau=1}^n\sum_{\theta=1}^n\sum_{\phi=1}^n \mathbb{E}[(x^t-\mu)(x^\tau-\mu) (x^\theta-\mu) (x^\phi-\mu)]\nonumber \\
       \leq & \frac{1}{n^4} [n \kappa \sigma^4 +3n(n-1)\sigma^4]
       \label{bound-4-mom}
     \end{align}
     observe, indeed, that from the independence of  samples descends that whenever $t\not \in \{ \tau, \theta, \phi\}$ 
   \begin{equation*}
     \mathbb{E}[(x^t-\mu)(x^\tau-\mu)(x^\theta -\mu)(x^\phi -\mu)]=\mathbb{E}[(x^t-\mu)]\mathbb{E}[(x^\tau-\mu) (x^\theta-\mu)(x^\phi-\mu)]=0
   \end{equation*}  
   while whenever $t=\tau\neq \theta=\phi$
   \begin{equation*}
     \mathbb{E}[(x^t-\mu)(x^\tau-\mu)(x^\theta -\mu)(x^\phi -\mu)]=\mathbb{E}[(x^t-\mu)^2]
     \mathbb{E}[(x^\theta-\mu)^2]\leqslant \sigma^4.
   \end{equation*}  
    Therefore
\[
 \mathbb{P}(|\bar{x}-\mu|>\beta_\gamma(n) )<  \frac{ \kappa \sigma^4 }{(n^3 \beta(n))^4}+ \frac{ 3\sigma^4 }{n^2(\beta_\gamma(n))^4}.   
\]
Now by sub-additivity of probability, we get:
\[
 \mathbb{P}(\exists n \in \mathbb{N},  |\bar{x}-\mu|>\beta_\gamma(n) )=\mathbb{P}(\cup_n\{ |\bar{x}-\mu|>\beta_\gamma(n)\} )
 < \sum_n \frac{ \kappa\sigma^4 }{n^3(\beta_\gamma(n))^4}+ \sum_n \frac{ 3\sigma^4 }{n^2(\beta_\gamma(n))^4}.   
\]
Now observe that $\{\beta_\gamma(n)\}_{n\in\mathbb{N}} $  on the one hand should be chosen as small as possible and in  particular
we should enforce $\beta_\gamma(n)\to 0$ as $n$ grows large; on the other hand, however 
the choice of $\{\beta_\gamma(n)\}_{n\in\mathbb{N}}$ must guarantee that:
\[
\sum_n \frac{ \kappa \sigma^4 }{n^3(\beta_\gamma(n))^4}+ \sum_n \frac{ \sigma^4 }{n^2(\beta_\gamma(n))^4}\leqslant 2\gamma
\]
This is possible if by choosing $\beta(n)=H n^{-\alpha} $  for an  $\alpha<\frac{1}{4}$ arbitrarily close to $1/4$ and a properly chosen $H$.
Indeed  with this choice of $\beta(n)$ we have:
\begin{align*} 
\sum_n \frac{ \kappa\sigma^4 } {n^3(\beta_\gamma(n))^4}+ \sum_n \frac{ 3\sigma^4 }{n^2(\beta_\gamma(n))^4}=&
\sum_n \frac{\kappa\sigma^4}{H^4n^{3-4\alpha}}+ \sum_n \frac{ 3\sigma^4 }{H^4n^{2-4\alpha}}
= \frac{(\kappa+3)\sigma^4}{H^4}\sum_n \frac{1}{n^{2-4\alpha}}\left(1+\frac{1}{n}\right)\\
\leqslant & \frac{2(\kappa+3)\sigma^4}{H^4}\sum_n \frac{1}{n^{2-4\alpha}}= \frac{2(\kappa+3)\sigma^4}{H^4}\zeta(2-4\alpha)
\end{align*}
where $\zeta(z):=\sum_n \frac{1}{n^z}$, with $z\in \mathbb{C}$,  denotes the $\zeta$-Riemann function. We recall  that $\zeta(x)<\infty$ for any real $x>1$. Therefore
by selecting 
\[
\beta_\gamma(n)=Hn^{-\alpha} \qquad \text{with} \quad \alpha<\frac{1}{4} \quad \text{and} \quad  H=\sqrt[4]{\frac{(\kappa+3)\sigma^4\zeta(2-4\alpha)}{\gamma}}
\]
we guarantee that 
\[
 \mathbb{P}(\exists n \in \mathbb{N},  |\bar{x}-\mu|>\beta_\gamma(n) )<2\gamma.
\]
A tighter expression can be obtained as follows. Set $\beta(n)= \left(\frac{H (1+\ln^2 n)}{n}\right)^{\frac{1}{4}}$:
\begin{align*} 
\sum_{n=1}^{\infty} \frac{ \kappa\sigma^4 } {n^3(\beta_\gamma(n))^4}+ \sum_{n=1}^{\infty} \frac{ 3\sigma^4 }{n^2(\beta(n))^4}& \leqslant
\sum_{n=1}^{\infty} \frac{(\kappa+3)\sigma^4 }{n^2(\beta(n))^4}
= \sum_{n=1}^{\infty} \frac{(\kappa+3)\sigma^4 }{n (1+\ln^2 n) }\\
& \leqslant \frac{(\kappa +3) \sigma^4}{H}  \left( 1 + \sum_{n=2}^{\infty} \frac{1}{n (1+\ln^2 n) } \right)\\
& \leqslant \frac{(\kappa +3) \sigma^4}{H}  \left( 1 + \sum_{n=2}^{\infty} \frac{1}{n \ln^2 n } \right)\\
& \leqslant \frac{(\kappa +3) \sigma^4}{H}  \left( 1 + \frac{1}{2 \ln^2 2}+ \int_{2}^{\infty} \frac{1}{x \ln^2 x } \mathrm{d} x\right)\\
& = \frac{(\kappa +3) \sigma^4}{H}  \left( 1 + \frac{1}{2 \ln^2 2}+ \frac{1}{\ln 2}\right)\\
& \leqslant 4 \frac{(\kappa +3) \sigma^4}{H}. 
\end{align*}
Imposing that this is smaller than $2 \gamma$, we can conclude that 
by selecting 
\[
\beta_\gamma(n)= \left(2\frac{(\kappa + 3) \sigma^4}{\gamma}\right)^{\frac{1}{4}} \left(\frac{1 + \ln^2 n}{n}\right)^{\frac{1}{4}},
\]
we guarantee that 
\[
 \mathbb{P}(\exists n \in \mathbb{N},  |\bar{x}-\mu|>\beta_\gamma(n) )<2\gamma.
\]

\begin{remark}\label{rem-moments}
When $x^t$ exhibits  a larger number of finite moments we can refine our approach by employing Proposition 
\eqref{cheby} for a different (larger) choice of $i$.
Doing so, we obtain a more favorable behavior for $\beta_\gamma(n)$. In particular, we will get that 
\[
\beta_\gamma(n)= O( n^{-\alpha})  \qquad \text{with} \qquad \alpha< \frac{1}{2}\left(1-\frac{1}{i}\right) 
\]
At last, we wish to emphasize that $\beta_\gamma(n)$ can not be properly defined when distribution $D_a$ exhibits less than four bounded polynomial moments. The application 
of Chebyshev inequality ($i=1$), indeed, would  lead to a too the following weak upper bound:
\[
 \mathbb{P}(\exists n \in \mathbb{N},  |\bar{x}-\mu|>\beta-\gamma(n) )<\sum_n \frac{\sigma^2}{n(\beta_\gamma(n))^2}.
\]
Observe, indeed, that since $\zeta(1)=\sum_n \frac{1}{n}$  diverges, it is impossible the find of suitable 
expression for $\{\beta(n)\}$ which jointly satisfy: $\lim_{n\to\infty} \beta_\gamma(n)=0$ and $ \sum_n \frac{\sigma^2}{n(\beta_\gamma(n))^2}<\infty$.
\end{remark}

\end{proof}

This result is the first fundamental building block to define a notion of distance (which uses the estimates $\bar{x}$ and the parametrization $\beta_\gamma$) for which it is possible to provide guarantees about the class membership.

We have bounded the probability of not having the \textit{true} mean within the $\beta_\gamma$ confidence interval given a certain estimate $\bar{x}$ and the corresponding number of samples $n$. We now have to take a global perspective, so we consider the event 
% $E := \bigcap\limits_{a \in \mathcal{A}} \bigcap\limits_{t \in \mathbb{N}} \bigcap\limits_{a' \in \mathcal{N}_a} \left| \bar{x}_{a,a'} - \mu_{a'} \right|$,
$E := \left\{\forall a' \in \mathcal{CC}_a, \forall t \in \mathbb{N}, \forall a'' \in \mathcal{N}_{a'}, \, \left| \bar{x}_{a',a''}^t - \mu_{a''} \right| < \, \beta_\gamma (n) \right\}$,
which is equivalent to say that, for every node $a' \in \mathcal{CC}_a$ in the connected component of node $a \in \mathcal{A}$ and for every instant $t \in \mathbb{N}$, the \textit{true} mean value of each of the neighbors $a''$ of node $a'$, given the info $a'$ is able to collect (neighbor's sample mean and the number of samples), is within the confidence interval $I_{a',a''}$. We show that this holds with high probability with an appropriate choice of $\gamma$:

\begin{lemma}[Confidence of $\beta_\gamma$ interval]
\label{lemma:beta_confidence}
    Considering the interval parametrization $\beta_\gamma (n)$ (Eq. (\ref{eq:beta_mailard}) or (\ref{eq:beta_cheby})), setting $\gamma(\delta) = \frac{\delta}{4 r |\mathcal{CC}_a|}$, it holds:
    \begin{equation}
        \mathbb{P}\left(\forall a' \in \mathcal{CC}_a, \forall t \in \mathbb{N}, \forall a'' \in \mathcal{N}_{a'},  \, \left| \bar{x}_{a',a''}^t - \mu_{a''} \right| < \beta_\gamma (n_{a',a''}^t) \right) \geq 1 - \frac{\delta}{2}
    \end{equation}
\end{lemma}

\begin{proof}
    We have introduced the event $E = \left\{\forall a' \in \mathcal{CC}_a, \forall t \in \mathbb{N}, \forall a'' \in \mathcal{N}_{a'}, \, \left| \bar{x}_{a',a''}^t - \mu_{a''} \right| < \, \beta_\gamma (n^t_{a',a''}) \right\}$, it is more convenient to work with the complementary event: % so that to leverage Lemma \ref{lemma:ci}:
    \begin{align*}
        \mathbb{P} (E) = 1 - \mathbb{P} (E^c) = 1 - \mathbb{P} \left( \exists a' \in \mathcal{CC}_a, \exists t \in \mathbb{N}, \exists a'' \in \mathcal{N}_{a'} \, : \, \left| \bar{x}_{a',a''}^t - \mu_{a''} \right| > \beta_\gamma(n_{a',a''}^t) \right)
    \end{align*}
    
    Applying a union bound with respect to the nodes $a' \in \mathcal{CC}_a$ and neighbors $a'' \in \mathcal{N}_{a'}$, and using Lemma \ref{lemma:ci} ($\mathbb{P}(\exists n \in \mathbb{N}, |\bar{x} -  \mu| \geq \beta_\gamma(n)) \leq 2\gamma$), we can immediately obtain a lower bound on the probability of the event $E$:
    \begin{align*}
        \mathbb{P}(E) & \geq 1 - \sum_{a' \in \mathcal{CC}_a} \sum_{a'' \in \mathcal{N}_{a'}} \mathbb{P} \left( \exists t \in \mathbb{N} : \left| \bar{x}^t_{a',a''} -\mu_{a''} \right| \geq \, \beta_\gamma (n_{a',a''}) \right)\\
        & \geq 1 - r |\mathcal{CC}_a| (2 \gamma) = 1 - 2 r |\mathcal{CC}_a| \gamma
    \end{align*}

    Now, we set $\gamma = \frac{\delta}{4 r |\mathcal{CC}_a|}$ and thus we immediately obtain $\mathbb{P}(E) \geq 1 - \frac{\delta}{2}$. This explains the value of the constant $\gamma$ in the theorem. The above bound would then be used in the $(\varepsilon - \delta)$-convergence of the \belief\ and \consensus\ algorithm.
\end{proof}

This result provides a probabilistic bound for the situation in which the true value is not within the confidence interval and for which we cannot provide theoretical guarantees.

At this point, assuming that event $E$ holds (with high probability due to Lemma \ref{lemma:beta_confidence}), we need to show that the \textit{optimistic} distance $d^t_\gamma(a,a')$ allows an agent to discriminate whether one of its neighbors belongs to the same similarity class $\mathcal{C}_a$.

\begin{lemma}[Class membership rule]
\label{lemma:class_membership}
Conditionally over the event $E$,
we have 
    \begin{equation}
        d_{\gamma}^t (a,a') > 0 \iff a' \notin \mathcal{N}_a \cap \mathcal{C}_a
    \end{equation}   
\end{lemma}

\begin{proof}
 Defined  the \textit{optimistic} distance\footnote{For ease of notation we will use $a, a'$, instead of $a',a''$ as we did in the previous Lemma.} as $d_{\gamma}^t(a, a') := \left| \bar{x}_{a,a}^t - \bar{x}_{a,a'}^t \right| - \beta_\gamma (n_{a,a}^t) - \beta_\gamma (n_{a,a'}^t) $ and denoted with $\Delta_{a,a'} = \left| \mu_a - \mu_{a'} \right|$
 the gaps between the \textit{true} mean of agents belonging to different similarity classes, by summing and subtracting $(\mu_{a} -\mu_{a'})$ inside the absolute value, it is immediate to obtain:

    \begin{equation}
        \label{eq:distance_rewrite}
        d_{\gamma}^t (a,a') = \left| (\mu_a - \mu_{a'}) + (\bar{x}_{a,a}^t - \mu_a) - (\bar{x}_{a,a'}^t -\mu_{a'}) \right| - \beta_\gamma (n_{a,a}^t) - \beta_\gamma (n_{a,a'}^t)
    \end{equation}

    Now, we show that conditionally over the event $E$, $d^t_\gamma(a,a')$ satisfies two inequalities 
    which  allow us to determine whether two nodes belong to the same similarity class by looking at the sign of $d^t_\gamma(a,a')$. 
    %Indeed, this will justify our choice of considering nodes of the same class whenever $d^t_\gamma(a,a') \leqslant 0$.

    \textbf{(Forward implication)} First, let us apply the triangular inequality on the absolute value in Eq. (\ref{eq:distance_rewrite}) and bound it with the sum of the absolute values of the addends:
    \begin{equation*}
        d^t_\gamma (a, a') \leqslant \Delta_{a,a'} + \left| \bar{x}_{a,a}^t - \mu_a \right| + \left| \bar{x}_{a,a'}^t -\mu_{a'} \right| - \beta_\gamma (n_{a,a}^t) - \beta_\gamma (n_{a,a'}^t)
    \end{equation*}

Now, conditionally over $E$, we have that $\left| \bar{x}_{a,a}^t - \mu_a \right| \leq \, \beta_\gamma (n_{a,a}^t)$ and $\left| \bar{x}_{a,a'}^t - \mu_{a'} \right| \leq \, \beta_\gamma (n_{a,a'}^t)$. 
Therefore whenever $a' \in \mathcal{C}_a \cup \mathcal{N}_a$, i.e., $\Delta_{a,a'}=0$,  we have:
    \begin{equation*}
        d^t_\gamma (a, a') \leqslant \left| \bar{x}_{a,a}^t - \mu_a \right| - \beta_\gamma (n_{a,a}^t) + \left| \bar{x}_{a,a'}^t -\mu_{a'} \right| - \beta_\gamma (n_{a,a'}^t) \leqslant 0
    \end{equation*}

    So,  conditionally over $E$, and if $a' \in \mathcal{C}_a \cup \mathcal{N}_a$ the optimistic distance $d_\gamma^t (a,a')$ is smaller or equal than 0, i.e., $a' \in \mathcal{C}_a \cup \mathcal{N}_a \implies d_\gamma^t (a,a') \leqslant 0$. Considering the contrapositive statement, we immediately prove the \textbf{forward implication} of the lemma, namely:
    \begin{equation*}
        d_\gamma^t (a,a') > 0 \implies a' \notin \mathcal{C}_a \cup \mathcal{N}_a.
    \end{equation*}

    \textbf{(Backward implication)} At this point, %we need to prove the backward implication, i.e., 
    we need to show that conditionally over the event $E$, whenever two nodes do not belong to the same similarity class,  then the optimistic distance is positive (i.e., $a' \notin \mathcal{C}_a \cup \mathcal{N}_a \implies d_\gamma^t (a,a') > 0$). To do so, we start from Eq. (\ref{eq:distance_rewrite}), aiming at deriving a lower bound for $d^t_\gamma (a,a')$. By applying the reverse triangular inequality ($|a-b| > ||a|-|b|| \implies |a|-|b| \leqslant |a-b|$):
    %% |a-b+c| < |c| + |a-b| < |a| + |b| + |c|
    \begin{equation*}
        d^t_\gamma (a, a') \geq \left| \Delta_{a,a'} + \left( \bar{x}_{a,a}^t - \mu_a \right) \right| - \left| \bar{x}_{a,a'}^t -\mu_{a'} \right| - \beta_\gamma (n_{a,a}^t) - \beta_\gamma (n_{a,a'}^t)
    \end{equation*}
    And then recalling that $|a+b|\ge |a|-|b|$, we get:
    \begin{equation*}
        d^t_\gamma (a, a') \geq \Delta_{a,a'} - \left| \left( \bar{x}_{a,a}^t - \mu_a \right) \right| - \left| \bar{x}_{a,a'}^t -\mu_{a'} \right| - \beta_\gamma (n_{a,a}^t) - \beta_\gamma (n_{a,a'}^t)
    \end{equation*}
    Again, conditionally over $E$, and we have $\left| \bar{x}_{a,a}^t - \mu_a \right| \leqslant \beta_\gamma (n_{a,a}^t)$ and $\left| \bar{x}_{a,a'}^t - \mu_{a'} \right| \leq \, \beta_\gamma (n_{a,a'}^t)$. Therefore we can write:
    \begin{align*}
        d^t_\gamma (a, a') & \geq \Delta_{a,a'} - \left| \left( \bar{x}_{a,a}^t - \mu_a \right) \right| - \left| \bar{x}_{a,a'}^t -\mu_{a'} \right| - \beta_\gamma (n_{a,a}^t) - \beta_\gamma (n_{a,a'}^t) \\
        & \geq \Delta_{a,a'} - 2 \beta_\gamma (n_{a,a}) - 2 \beta_\gamma (n_{a,a'}^t)
    \end{align*}
    Moreover, consider that by definition\footnote{As a matter of fact, for our scalable algorithms, the inequality is always strict as $n^t_{a,a} = t$ and $n^t_{a,a'} = t-1$.} we have $n^t_{a,a} \geq n^t_{a,a'}$, thus $\beta_\gamma (n^t_{a,a}) \leqslant \beta_\gamma (n^t_{a,a'})$, so we can write:
    \begin{equation*}
         d^t_\gamma (a, a') \geq \Delta_{a,a'} - 4 \beta_\gamma (n^t_{a,a'})
    \end{equation*}
    {Now we need to observe that, whereas conditionally over $E$ in the previous case ($a' \in \mathcal{C}_a \cup \mathcal{N}_a$) the optimistic distance always keeps negative  (simply take in mind that by definition $\beta_\gamma (0) = + \infty$). When neighbor $a'$  belongs to a different similarity class of $a$
    (i.e., $a' \notin \mathcal{C}_a \cup \mathcal{N}_a$), the optimistic distance $d_\gamma^t (a,a')$ will become positive (thus signaling $a$ and $a'$ belong to different similarity classes), as soon as the collected number of samples $n^t_{a,a'}$ becomes sufficiently large to guarantee 
     $\beta_\gamma(n^t_{a,a'})<\frac{ \Delta_{a,a'}}{4}$.}

Now denoted with $\beta_\gamma ^{-1}(x)$ the inverse function of $\beta_\gamma (n)$,  and defined: 
    \begin{equation}
    \label{eq:def_n_star_aa}
        n_{a,a'}^\star = \left\lceil \beta_\gamma^{-1} \left( \frac{\Delta_{a,a'}}{4} \right) \right\rceil,
    \end{equation} 
conditionally over $E$, $ \forall \, n_{a,a'} \geq n_{a,a'}^\star$,  we have $a' \notin \mathcal{C}_a \cup \mathcal{N}_a \implies d_\gamma^t (a,a') > 0$. And this proves the \textbf{backward} implication, which concludes the proof of the lemma.
\end{proof}

To conclude our proof, observe that according to our scalable algorithms, at each time instant $t\in \mathbb{N}$ each node $a \in \mathcal{A}$ queries all the nodes that were in its \textit{estimated} similarity class at the previous step $\mathcal{C}_a^{t-1}$ (for $t=0$ all the neighbors $a' \in \mathcal{N}_a$ are contacted). Therefore, all received estimates $\bar{x}^t_{a,a'}$ suffer for a delay of 1 time instant, i.e., $n_{a,a'} = t-1$.
Now, Lemma \ref{lemma:beta_confidence}
 ensures that, by choosing $\gamma = \frac{\delta}{4 r|\mathcal{CC}_a|}$,  we have $\mathbb{P}(E)\ge 1- \frac{\delta}{2} $.   
Moreover, considering: 
   \begin{equation}
    \label{eq:def_n_star}
        \qquad  n_\gamma^{\star} (a) = \left\lceil \beta_\gamma^{-1} \left( \frac{\Delta_a}{4} \right) \right\rceil
    \end{equation}
where recall that $\Delta_a=\min_{ a'  \in {\mathcal{A}} \setminus \mathcal{C}_a}\Delta_{a,a'}$.
%$\Delta=\min_{a} \min_{ a'  \in \mathcal{N}_a\setminus \mathcal{N}_a\cap \mathcal{C}_a}\Delta_{a,a'}$;

By Lemma \ref{lemma:class_membership}  we have that, conditionally over $E$, as soon as  $t-1 \geq n^\star_{a,a'}$ we have
$d^t_\gamma (a,a')>0$ for all pairs of neighboring nodes $(a,a')$  belonging to different similarity classes, while 
$d^t_\gamma (a,a')\leqslant 0$  for all pairs of neighboring nodes $(a,a')$  belonging to the same similarity class. Therefore, whenever $t-1 \geq n_\gamma^{\star} (a)$ this holds for all the pairs in the connected component $\mathcal{CC}_a$, as $\Delta_a \geq \min_{ a'  \in {\mathcal{CC}_a}\setminus \mathcal{C}_a}\Delta_{a,a'}$.
\end{proof}

\newpage

\section{Appendix C - Extension to the Multidimensional Case}
\label{app:multidim}

In this Appendix, we show briefly how the previous approach can be generalized to the multidimensional case.
We assume that at each instant $t$ every agent $a$
generates a new sample $\bm{x}_a^t \in \mathbb{R}^K$  drawn i.i.d.~from  
%an assigned individual 
distribution $D_a$ with expected value $\bm{\mu}_a=\mathbb{E}[\bm{x}_a^t]$.
In particular, we show how the definition of confidence intervals $\beta_\gamma(n)$
can be extended to the multidimensional case.  Let $\Delta_{a,a'}:=||\bm{\mu}_a-\bm{\mu}_{a'}  ||$.

\subsection{Sub-Gaussian case}
\label{app:multi_sub}

In this case, we assume distributions $D_a$ to be multidimensional sub-Gaussians.
Recalling the definition of multidimensional sub-Gaussian random variables, we have:
\begin{definition}
A random vector $\bm{x} \in \mathbb{R}^K$ is said to be sub-Gaussian with variance proxy $\sigma^2 $ 
if it is centered and for any $\bm{u} \in \mathbb{R}^K$ such that $|| u||  = 1$, the real random variable $\bm{u}^T \bm{x}$  is sub-Gaussian with variance proxy $\sigma^2$. 
\end{definition}

Therefore, we can choose an orthonormal basis $\bm{u}_i$ $\forall 1\le i\le K$ in $\mathbb{R}^K$ 
and consider the uni-dimensional projections of samples along directions induced by  $\bm{u}_i$, $\forall i$.
 Since projected samples are uni-dimensional sub-Gaussian,  the theory developed in the previous sections applies to every projection $i$. In particular, we can compute  the {\it optimistic} distance along axes~$i$ as:
 \[
d_{i, \gamma/K}^t(a, a') := \left| \bm{u}_i^T( \bar{\bm{x}}_{a,a}^t - \bar{\bm{x}}_{a,a'}^t) \right| - \beta_{\gamma/K} (n_{a,a}^t) - \beta_{\gamma/K} (n_{a,a'}^t)
\]
 and decide that   $a'$ is maintained in $\mathcal{C}^t_a$, as long as $d_{i,\gamma/K}^t(a, a') \le 0$ $\forall\; 1\le i\le d$.
 At last  observe that  $\Delta^{(i)}_{a,a'}= \bm{u}_i^T(\bm{\mu}_a - \bm{\mu}_{a'} )$
 satisfies:
 \[
 \max_i \Delta^{(i)}_{a,a'}\ge \frac{1} {\sqrt{K}} \Delta_{a,a'}
\]
as a result of simple geometrical arguments.

\subsection{Fourth bounded moment}
We start assuming that that  $\mathbb{E}[||(\bm{x}_{a}^t-\bm{\mu}_a)||^4]<\infty$ for any $a\in \mathcal{A}$. 
Let  $\bar{\bm{x}}_a=\frac{1}{n}\sum_{t=1}^n \bm{x}^t$  and 
$Y_a:= ||\bar{\bm{x}}_a -\bm{\mu}_a ||\ge 0$, by  applying 
Proposition \ref{cheby} to $Y$
we have:
\[
\mathbb{P}(|Y_a|> \beta_\gamma(n))\le \frac{
	\mathbb{E}[Y_a^4]
}{\beta_\gamma(n)^4}
\]
Now defined with ${\mu}_{i,a}$ the expectation of the $i$-th component of samples $\bm{x}^t_a$, i.e:
\[
\mu_{i,a}:= 
\mathbb{E}[x^t_{i,a}] 
\]
considering the Euclidean norm, we have: 
\begin{align*}
\mathbb{E}[Y_a^4]&=\mathbb{E}\left[ \left(\sum_{i=1}^K (\bar{x}_{i,a}-{\mu}_{i,a})^2 \right)^2\right] =
\sum_{i=1}^K \sum_{h=1}^K  \mathbb{E}[ (\bar{{x}}_{i,a} -\mu_{i,a} )^2(  {\bar{x}}_{h,a} -\mu_{h,a} )^2]  \\
&\le   \sum_{i=1}^K \sum_{h=1}^K  \sqrt{ \mathbb{E} [(\bar{{x}}_{i,a} -\mu_{i,a} )^4 ]  \mathbb{E}[(\bar{{x}}_{h,a} -\mu_{h,a} )^4]}
\le K^2 \max_i \mathbb{E}[(\bar{{x}}_{i,a} -\mu_{i,a} )^4]
\end{align*}
where the second to last inequality follows by the application of the Cauchy-Schwarz inequality.  

Now proceeding exactly as in \eqref{bound-4-mom} we obtain that for every $i$
\[
\mathbb{E}[(\bar{{x}}_{i,a} -\mu_{i,a} )^4]\le \frac{1}{n^4} [n \kappa_i \sigma_i^4 +3n(n-1)\sigma_i^4]
\]
Now by sub-additivity of probability we get, recalling that $Y_a = || \bar{\bm{x}}_{a} - \bm{\mu}_{a} ||$:
\[
\mathbb{P}(\exists n \in \mathbb{N},  Y_a>\beta_\gamma(n) )=\mathbb{P}(\cup_n\{ Y_a>\beta_\gamma(n)\} )
< K^2 \max_i \left(\sum_n \frac{ \kappa_i\sigma^4_i }{n^3(\beta_\gamma(n))^4}+ \sum_n \frac{ 3\sigma_i^4 }{n^2(\beta_\gamma(n))^4}\right).   
\]
Then proceeding as in the proof of Theorem \ref{cor:Dcolme_thm1}, (i.e. forcing the r.h.s to be less than $2 \gamma$) we obtain the following expression for~$\beta_\gamma(n)$:
\[
\beta_\gamma(n)= \max_i \sqrt{K}
 \left(2\frac{(\kappa_i + 3) \sigma_i^4}{\gamma}\right)^{\frac{1}{4}} \left(\frac{1 + \ln^2 n}{n}\right)^{\frac{1}{4}}.
\]
At last, we have to generalize the definition of  \textit{optimistic} distance  as  follows:
\[
d_{\gamma}^t(a, a') := \left|\left| \bar{\bm{x}}_{a,a}^t - \bar{\bm{x}}_{a,a'}^t \right|\right| - \beta_\gamma (n_{a,a}^t) - \beta_\gamma (n_{a,a'}^t). \]
Then, proceeding exactly as in the previous section, 
we can extend Lemma 	\ref{lemma:class_membership} statement to  the multidimensional case
(we recall that by triangular inequality
%\begin{align}
%	\label{eq:distance_rewrite-multidim}
$	\left|\left| \bar{\bm{x}}_{a,a}^t - \bar{\bm{x}}_{a,a'}^t 	 \right|\right| = \left|\left| (\bm{\mu}_a - \bm{\mu}_{a'}) + 
(\bar{\bm{x}}_{a,a}^t - \bm{\mu}_a) - (\bar{\bm{x}}_{a,a'}^t -\bm{\mu}_{a'}) \right|\right|   \le \Delta_{a,a'}+ \left|\left|   \bar{\bm{x}}_{a,a}^t - \bm{\mu}_a   \right|\right|+ \left|\left|   \bar{\bm{x}}_{a,a'}^t -\bm{\mu}_{a'} \right|\right| $,
as well as by reverse triangular inequality
$
	\left|\left| \bar{\bm{x}}_{a,a}^t - \bar{\bm{x}}_{a,a'}^t 	 \right|\right|  \geq \Delta_{a,a'} - \left|\left|  \bar{\bm{x}}_{a,a}^t - \bm{\mu}_a  \right|\right| - 
	\left| \left| \bar{\bm{x}}_{a,a'}^t -\bm{\mu}_{a'} \right| \right|
$).

\subsection{Numerical Experiments}

We experimentally assess the \lq dicovery\rq\ phase, where nodes attempt to find peers in their neighborhood $\mathcal{N}_a$ that belong to the same similarity class (note that the guarantees we established for the estimate process, whether through the message-passing scheme or consensus, hold for each component). 

We consider a setup similar to Section~\ref{sec:pe}, with $N=10000$ agents, belonging to one of two classes: the first characterized by an all-zero vector mean $\bm{\mu_1}=\mathbf{0}$ and the second by $\bm{\mu_2}=\frac{1}{\sqrt K}\mathbf{1}$. Samples are drawn from a multivariate Gaussian distribution, where the common variance (the diagonal element in the covariance matrix) is $\sigma^2=4.0$.
We consider the canonical basis , i.e., $\{ \mathbf{e}_i \}_{i \in \{ 1, ..., K\}}$ \footnote{The vector $\mathbf{e}_i = (0,...,1,...,0)$ is the all-zero vector with a a $1$ in its i-th component.}, as the vectors for projecting our vectorial samples. The covariance matrix can be arbitrary, provided that the diagonal elements (the variances) are $\sigma^2$. We consider $\varepsilon=0.1$ and $\delta=0.1$, averaging over $10$ different realizations of the process.

\begin{figure*}[h!]
\centerline{\includegraphics[width=0.65\textwidth]{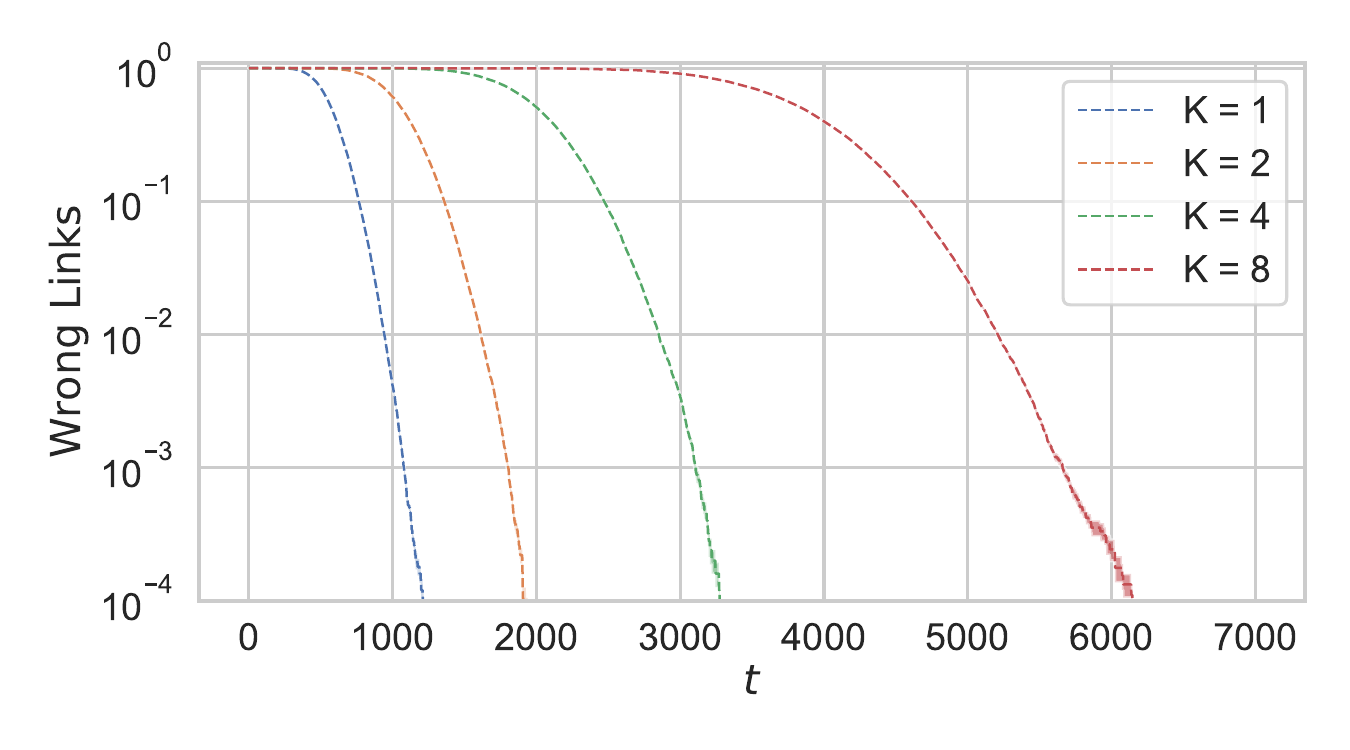}}
\caption{Fraction of the \textit{wrong} links (log scale) over time, as a function of the dimension $K$ of the mean values, $\bm{x} \in \mathbb{R}^K$.} \label{fig:multidim}
\end{figure*}

As the number of dimensions $K$ increases, the time required to discover peers in the same class also increases. However, doubling the number of dimensions results in less than a doubling of the discovery time.
This slowdown occurs because the bounds on each component need to be more stringent (note the factor $\frac{\gamma}{K}$ in Sec.~\ref{app:multi_sub}) to guarantee the same bound on the probability of incorrectly hindering a connection to a same-class peer.
% However, note also that on the flip side, it is sufficient that any of the~$K$ components of the distance~$K$ becomes positive, that the wrong link is removed, and this contributes to the sublinear increase of the discovery time.
However, it's also important to note that if any of the $K$ components of the distance~$d_{\gamma / K}^t(a,a')$ becomes positive, it is a condition sufficient to remove the incorrect link. This mechanism contributes to the sublinear increase in discovery time with respect to~$K$.

\newpage

\section{Appendix D - Schematic Representation for \belief{} }
\label{app:D}

We provide a sketch for the functioning of the \belief{} which, together with Algorithm~\ref{alg:belief} (in the main text) provides a detailed explanation of the proposed algorithm.

\begin{figure*}[h!]
\centerline{\includegraphics[width=0.95\textwidth]{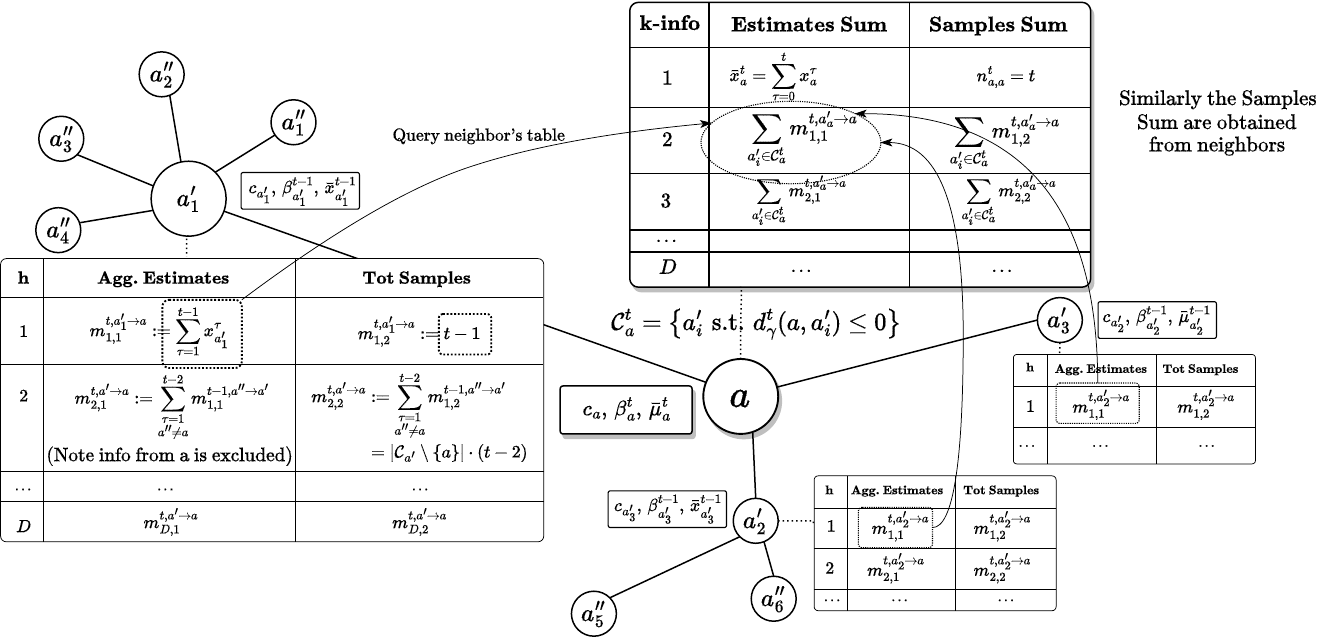}}
\caption{Simplified illustration of the functioning of the \belief{} algorithm from the point of view of node $a$ (all the quantities are already aggregated in the messages $m_{h,i}^{t, a' \rightarrow a}$, so that to exclude the \lq\lq self\rq\rq\ info sent by $a$).} \label{fig:sketch_belief}
\end{figure*}

\newpage

\section{Appendix E - Proof for \belief{} - Theorem \ref{thm:Bcolme_thm2} }
\label{appendix:th2}
%\subsection{PAC Convergence for \belief\ }

First, we recall the notion of $(\varepsilon,\delta)$-convergence (also referred to as PAC-convergence), 
that we use to assess theoretically the performance of the estimation algorithms. The definition is as follows:

\begin{definition}[PAC-convergence]
	An estimation procedure for agent $a$ is called $(\varepsilon,\delta)$-convergent if there exist $\tau_a \in \mathbb{N}$ such that:
	\begin{equation*}
		\mathbb{P} \left( \forall t > \tau_a , \left| \hat{\mu}_a^t - \mu_a \right| \leq \varepsilon \right) > 1 - \delta
	\end{equation*}
\end{definition}

% While the previous result holds for both our \textit{scalable} algorithms, here we focus on \belief{}. We derive a convergence result in the sense of PAC convergence (Definition 1).
Here we provide a convergence result in the sense of the above definition for \belief{}, we will derive a similar result for \consensus{} in Appendix \ref{sec:c_colme_proofs}.

\begin{theorem}
    Provided that $\mathcal{CC}^d_a$ has a tree structure,  for any $\delta \in (0,1)$, employing \belief (d), 
    we have:
    \begin{align*}
    	& \mathbb{P}\left(\forall t>\tau^B_a, |\hat \mu_a^t-\mu_a |<\varepsilon \right)\ge 1-\delta\\\
    	& \qquad \text{with}\quad \tau^B_a=\max\left[ \zeta_D+d, \frac{\widetilde{n}_{\frac{\delta}{2}}}{|\mathcal{CC}^d_a|} +d \right]     
    \end{align*}
where 
\[
\widetilde{n}_{\frac{\delta}{2}}(\varepsilon):=\left\{\begin{array} {ll}

 \min_{n}\left\{n : \sum_{n+1}^\infty 2Q\left(\frac{\sqrt{n}\varepsilon}{\sigma}\right)<\frac{\delta}{2}\right\}  & \text{ Gaussian distribitions} \\
\min_{n}\left\{n : \sum_{n+1}^\infty 2\exp\left(\frac{n\varepsilon^2}{2\sigma^2}\right)<\frac{\delta}{2}\right\}  & \text{ sub-Gaussian distributions } \\
\min_{n}\left\{n : \sum_{n+1}^\infty  \frac{\mu_4 +2(\sigma^2)^2}{(\varepsilon n)^2}  <\frac{\delta}{2}\right\}   &    \text{distrib. with bounded fourth moment}   \\
\end{array} \right.
\]
with 
\[
\widetilde{n}_{\frac{\delta}{2}}(\varepsilon)\le \left\{\begin{array} {ll}
\Big\lceil-\frac{\sigma^2}{\varepsilon^2} \log\left( \frac{\delta}{2}\Big(1- \mathrm{e}^{-\frac{ \varepsilon^2 }{\sigma^2}} \Big) \right) \Big \rceil    &  \text{Gaussian distributions}\\
\Big\lceil-\frac{\sigma^2}{\varepsilon^2} \log\left( \frac{\sqrt{2\pi}\varepsilon\delta}{2\sigma}\Big(1- \mathrm{e}^{-\frac{ \varepsilon^2 }{\sigma^2}} \Big) \right) \Big \rceil    &  \text{ sub-Gaussian distributions}\\
\Big \lceil \frac{2(\kappa +3)\sigma^4}{\delta \varepsilon^4}\Big \rceil     &    \text{distrib. with bounded fourth moment} 
\end{array} \right.
\]
\end{theorem}
where we recall that Gaussian distributions belong to the class of sub-Gaussians, and therefore the bound derived for sub-Gaussian distributions, can be applied also to Gaussian.

\begin{proof}
We start considering the case in which the distribution of samples for every  agent $a$  is  Gaussian (normal) with the same standard deviation $\sigma$, i.e. $D_a= \mathcal{N}(\mu_a, \sigma )$, $\forall a$. 
In such a case, if we consider an empirical average $\bar{x}(n)= \frac{1}{n}\sum_1^n x_t$ where $x_t$ are i.i.d  extracted from  $D_a$, for some $a\in \mathcal{A}$,
we have 
\[
\mathbb{P}(|\bar{x}(n)-\mu_a |>\varepsilon)= 2Q\left(\frac{\sqrt{n}\varepsilon}{\sigma}\right)
\]
indeed observe that $\bar{x}(n)$, as an immediate consequence of the elementary properties of normal random variables, is distributed as a Gaussian, with zero mean and standard deviation equal to $\frac{\sigma}{\sqrt{n}}$.

Then  given an arbitrary $n_0\in \mathbb{N}$, 
\[
\mathbb{P}(\exists n > n_0 \,:\, |\bar{x}(n)-\mu_a |>\varepsilon)\le \sum_{n_0+1}^\infty \mathbb{P}( |\bar{x}(n)-\mu_a |>\varepsilon)=
\sum_{n_0+1}^\infty 2Q\left(\frac{\sqrt{n}\varepsilon}{\sigma}\right)
\]
Observe that $\sum_{1}^\infty 2Q\left(\frac{\sqrt{n}\varepsilon}{\sigma}\right)$  converges, therefore we can safely  define 
\[
\widetilde{n}_{\frac{\delta}{2}}(\varepsilon):=\min_{n_0}\left\{n_0 : \sum_{n_0+1}^\infty 2Q\left(\frac{\sqrt{n}\varepsilon}{\sigma}\right)<\frac{\delta}{2}\right\}.
\]
Now, to conclude the proof, observe that conditionally over the fact that $\mathcal{C}_a^{t}= \mathcal{C}_a\cap \mathcal{N}_a$  for every time $t\ge t_0-d$, agent $a$  computes its average $\hat{\mu}_a$ as an average of samples collected by all 
agents in $\mathcal{CC}_a^d$. Such samples are i.i.d. and follow $D_a$; moreover, their number is easily lower bounded by 
$|\mathcal{CC}_a^d|(t_0 -d)$. Therefore whenever $|\mathcal{CC}_a^d|(t_0 -d)\ge \widetilde{n}_{\frac{\delta}{2}}$ by the definition of 
$\widetilde{n}_{\frac{\delta}{2}}$ we have 
\[
\mathbb{P}\left(\exists t>t_0 \,:\, |\hat \mu_a^t-\mu_a |>\varepsilon\; \Big| \; |\mathcal{CC}_a^d|(t_0 -d)\ge \widetilde{n}_{\frac{\delta}{2}},\,
 \mathcal{C}_a^{t_0-d}= \mathcal{C}_a\cap \mathcal{N}_a\, \forall a \in \mathcal{A}  \right)<\frac{\delta}{2}.
\]
The claim descends from the definition of $\zeta_D$ in Theorem \ref{cor:Dcolme_thm1}.

Now consider the more general case in which  the distribution of samples at nodes are Gaussian with possibly different standard deviations, uniformly bounded by $\sigma$, i.e.,
$D_a=\mathcal{N}(\mu_a, \sigma_a)$   with $\sigma_a\le \sigma$, $\forall a\in \mathcal{A}$.
In such a case if we consider an empirical average $\bar{x}(n)= \frac{1}{n}\sum_1^n x_t$  of independent  samples $x_t$ 
extracted from  Gaussian distributions with the same average $\mu_a$, but with possibly different standard deviations $\sigma_a\le \sigma$,
we have that $\bar{x}(n)$ is distributed as a Gaussian with zero mean and standard deviation smaller or equal than $\frac{\sigma}{\sqrt{n}}$.
Therefore:
\[
\mathbb{P}(|\bar{x}(n)-\mu_a |>\varepsilon)\le 2Q\left(\frac{\sqrt{n}\varepsilon}{\sigma}\right)
\]
and we can proceed exactly as in the previous case.

Now, the previous approach rather immediately extends to the case  in which  $D_a$ are sub-Gaussian  with parameter $\sigma$,
since in this case $\bar{x}(n)= \frac{1}{n}\sum_1^n x_t$  of independent  samples $x_t$ 
extracted from sub-Gaussian distributions with parameter $\sigma$ having the same average $\mu_a$ is sub-Gaussian with parameter $\sigma/\sqrt{n}$ and zero mean, and therefore by definition of sub-Gaussian (see \cite{subgaussians}  for a discussion about  sub-Gaussian distributions and their properties):
%\[
%\mathbb{P}(|\bar{x}(n)-\mu_a |>\varepsilon)\le 2Q\left(\frac{\sqrt{n}\varepsilon}{\sigma}\right).
%\]

\[
\mathbb{P}(|\bar{x}(n)-\mu_a |>\varepsilon)\le 2  \exp\left(- \frac{n\varepsilon^2}{2\sigma^2}\right) % Q\left(\frac{\sqrt{n}\varepsilon}{\sigma}\right).
\]

Now since  $\sum_{1}^\infty \exp\left(- \frac{n\varepsilon^2}{2\sigma^2}\right)$  is a converging geometrical series,  we can safely  define 
\[
\widetilde{n}_{\frac{\delta}{2}}(\varepsilon):=\min_{n_0}\left\{ n_0 : \sum_{n_0+1}^\infty 2\exp\left(- \frac{n\varepsilon^2}{2\sigma^2} \right)<\frac{\delta}{2} \right\}.
\] 
and proceed as in the previous case.

At last, consider the case in which $\mu_a$ has a fourth central moment uniformly bounded by $\mu_4$, and a variance uniformly bounded by $\sigma^2$. In this case  considering the empirical  $\bar{x}(n)= \frac{1}{n}\sum_1^n x_t$  of independent samples $x_t$ 
extracted from arbitrary distributions with the same average $\mu_a$, following the same approach as in the proof of Proposition \ref{cheby}
we can bound its fourth moment as:
\[
\mathbb{E}[(\bar{x}(n) -\mu_a)^4]\le \frac{1}{n^4} [n \kappa +3n(n-1)]\sigma^4
\]
Therefore applying Chebyshev inequality (see Proposition  \ref{cheby}) we have:
\[
\mathbb{P}(|\bar{x}{(n)}-\mu_a|\ge \varepsilon)<\frac{\mathbb{E}[(\bar{x}(n) -\mu_a)^4]}{\varepsilon^4}\le \frac{1}{(\varepsilon n)^4} [n \kappa +3n(n-1)]\sigma^4
\]

Then  given an arbitrary $n_0\in \mathbb{N}$
\[
\mathbb{P}(\exists n > n_0, \, |\bar{x}(n)-\mu_a |>\varepsilon)\le \sum_{n_0+1}^\infty \mathbb{P}( |\bar{x}(n)-\mu_a |>\varepsilon)=
\sum_{n_0+1}^\infty  \frac{1}{(\varepsilon n)^4} [[n \kappa +3n(n-1)]\sigma^4]\le 
 \frac{(\kappa +3)\sigma^4}{\varepsilon^4 n^2} 
\]
again  $\sum_{1}^\infty  \frac{(\kappa +3)\sigma^4}{\varepsilon^4 n^2}  $  converges, therefore we can  define 
\[
\widetilde{n}_{\frac{\delta}{2}}(\varepsilon)=\min_{n_0}\left\{n_0 : \sum_{n_0+1}^\infty  \frac{(\kappa +3)\sigma^4}{\varepsilon^4 n^2} <\frac{\delta}{2} \right\}
\]
Then we can proceed exactly as in previous cases.

As the last step, now we derive easy upper bounds for $\widetilde{n}_{\frac{\delta}{2}}$.  We start from the last case.
We have 
\[
\sum_{n_0+1}^\infty \frac{1}{n^2}\le \int_{n_0+1}^\infty \frac{1}{(x-1)^2} \diff x= \frac{1}{n_0}
\]
and 
\[
\sum_{n_0+1}^\infty  \frac{(\kappa +3)\sigma^4)}{\varepsilon^4 n^2}\le  \frac{(\kappa +3)\sigma^4}{\varepsilon^4n_0} \qquad n_0>1
\]
From which  we obtain that: 
\[
 \widetilde{n}_{\frac{\delta}{2}} \le \Big \lceil \frac{2(\kappa +3)\sigma^4}{\delta \varepsilon^4}\Big \rceil
\]
Now considering the Gaussian case, in this case, we can exploit the following well-known bounds for $Q(x)$:   
\[
\frac{x}{1+x^2} \theta(x) <Q(x)< \frac{1}{x}\theta(x)  \qquad \forall x\ge 0
\]
with $\theta(x)= \frac{\mathrm{e}^{-\frac{x^2}{2}}}{\sqrt{2\pi}}. $
Therefore 
\[
\sum_{n_0+1}^\infty Q\left(\frac{\sqrt{n}\varepsilon}{\sigma}\right)\le \sum_{n_0+1}^\infty 
\frac{\sigma \mathrm{e}^{-\frac{ n\varepsilon^2 }{2\sigma^2}}}{\sqrt{2\pi n}\varepsilon }\le  \sum_{n_0+1}^\infty
\frac{\sigma \mathrm{e}^{-\frac{ n\varepsilon^2 }{2\sigma^2}}}{\sqrt{2\pi }\varepsilon }=
\frac{\sigma}{\sqrt{2\pi }\varepsilon}  \mathrm{e}^{-\frac{ n_0\varepsilon^2 }{2\sigma^2}} 
\sum_{n=1}^\infty   \mathrm{e}^{-\frac{ n\varepsilon^2 }{2\sigma^2}}
\]
with 
\[
\sum_{n=1}^\infty   \mathrm{e}^{-\frac{ n\varepsilon^2 }{2\sigma^2}}=\frac{1}{1- \mathrm{e}^{-\frac{ \varepsilon^2 }{2\sigma^2}}}
\]
Therefore imposing 
\[
2\frac{\sigma}{\sqrt{2\pi }\varepsilon}  \mathrm{e}^{-\frac{ \widetilde{n}_{\frac{\delta}{2}}\varepsilon^2 }{2\sigma^2}} \frac{1}{1- \mathrm{e}^{-\frac{ \varepsilon^2 }{2\sigma^2}}}\le \frac{\delta}{2}
\]
we obtain the following upper bound on $\widetilde{n}_{\frac{\delta}{2}}$
\[
\widetilde{n}_{\frac{\delta}{2}}\le \Big \lceil-\frac{2\sigma^2}{\varepsilon^2} \ln\left( \frac{\sqrt{2\pi}\varepsilon\delta}{4\sigma}\Big(1- \mathrm{e}^{-\frac{ \varepsilon^2 }{2\sigma^2}} \Big) \right)\Big \rceil
\]
%Observe that previous bound applies (tightly)  to the case of sub-Gaussian distributions with parameter $\sigma$.
At last, consider the sub-Gaussian case:
\[
\sum_{n_0+1}^\infty \exp\left(- \frac{n\varepsilon^2}{2\sigma^2} \right)= 
\exp\left(- \frac{n_0\varepsilon^2}{2\sigma^2}\right) \sum_{n=1}^\infty\exp\left(- \frac{n\varepsilon^2}{2\sigma^2}\right)
= \frac{ \exp \left(- \frac{n_0\varepsilon^2}{2\sigma^2} \right)   }{1- \exp\left(- \frac{\varepsilon^2}{2\sigma^2}\right)}
\]
from which we obtain:
\[
\widetilde{n}_{\frac{\delta}{2}}\le \Big \lceil-\frac{2\sigma^2}{\varepsilon^2} \ln\left( \frac{\delta}{4} \Big(1- \mathrm{e}^{-\frac{ \varepsilon^2 }{2\sigma^2}} \Big) \right)\Big \rceil
\]
\end{proof}

We observe that $\widetilde{n}_{\frac{\delta}{2}}(\varepsilon)$ is smaller than the corresponding term  $\beta_{\frac{\delta}{2}}^{-1}(\varepsilon)$ appearing in Theorem \ref{thm:colme_thm2} for \colme. Our proofs can readily be adapted to \colme, enabling to substitute  $n_{\frac{\delta}{2}}^\star(\varepsilon)$ in \cite{colme} [Theorem~2] with the smaller term~$\widetilde{n}_{\frac{\delta}{2}}(\varepsilon)$.

\newpage

\section{Appendix F - Proofs for  \consensus{} - Theorem \ref{thm:fourth_moment_main_text} and Theorem \ref{thm:consensus_epsilon_delta}}
\label{sec:c_colme_proofs}

It is convenient to consider an auxiliary system for which the consensus matrix can be arbitrary until time $\zeta_d$ and then switches to a situation where agents only communicate with their neighbors belonging to the same class, i.e., $W_t = W$ for any $t\geqslant \zeta_d$ and $W_{a,a'}> 0$ if and only if $a'\in \mathcal C_a \cap \mathcal N_a$. 

We will derive some bounds for the auxiliary system for any choice of the matrices $W_1, W_2, \dots W_{\gamma_D}$. With probability $1- \frac{\delta}{2}$ these bounds apply also to the original system under study because with such probability each agent correctly detects the neighbors in the same class for any $t > \zeta_D$. From now on, we refer then to the auxiliary system.

\subsection{Preliminaries}
After time $\zeta_D$ the original graph has been then split into $C$ connected components, where component $c$ includes $n_c$ agents. By an opportune permutation of the agents, we can write the matrix $W$ as follows 
\[W=\left(\begin{array}{cccc}\prescript{}{1}W & 0_{n_1 \times n_2} & \cdots & 0_{n_1 \times n_C} \\ 
0_{n_2 \times n_1} & \prescript{}{2}W & \cdots & 0_{n_2 \times n_C} \\ 
\cdots & \cdots & \cdots & \cdots\\
0_{n_C \times n_1} & 0_{n_C \times n_2} & \cdots & \prescript{}{C}W\end{array}\right),\]
where $0_{n \times m}$ denotes an $n\times m$ matrix with 0 elements.

We focus on a given component $c$ with $n_c$ agents. All agents in the same component share the same expected value, which we denote by $\mu(c)$. Moreover, let $\mmc=\mu(c)  \mathbf{1}_c$. We denote by $\xc^t$ and $\yc^t$ the $n_c$-dimensional vectors containing the samples' empirical averages and the estimates for the agents in component $c$. 

For $t > \zeta_D$, the estimates in component $c$ evolve independently from the other components and we can write:
\[
{\yc}^{t+1}=\left(1-\alpha_t\right) \xavc^{t+1} +\alpha_t \Wc {\y}^t(c).
\]
It is then easy to prove by recurrence that 
\begin{align}
\label{e:first_recurrence}
{\yc}^{t+1} - \mmc=\alpha_{\zeta_D,t} \Wc^{t+1 - \zeta_D}(\y^{\zeta_D}(c) - \mmc) + \sum_{\tau=\zeta_D}^t (1-\alpha_\tau) \alpha_{\tau+1,t} {\Wc}^{t-\tau} (\xavc^{\tau+1} - \mmc),
\end{align}
where $\alpha_{i,j} \triangleq \prod_{\ell= i}^j \alpha_{\ell}$, with the usual convention that  $\alpha_{i,j}=1$ if $j< i$.

Let $\Pc \triangleq \frac{1}{n_c}\mathbf 1 \mathbf 1^{\top}$.
It is easy to check that the doubly-stochasticity  of $W$ implies that $(\Wc - \Pc)^t = \Wc^t - P$. From which it follows 
{\allowdisplaybreaks
\begin{align}
\label{e:second_recurrence}
{\yc}^{t+1} - \mmc& =\alpha_{\zeta_D,t} \Wc^{t+1 - \zeta_D}(\y^{\zeta_D}(c) - \mmc) + \sum_{\tau=\zeta_D}^t (1-\alpha_\tau) \alpha_{\tau+1,t} {\Wc}^{t-\tau} (\xavc^{\tau+1} - \mmc)\\
& =\alpha_{\zeta_D,t} \Wc^{t+1 - \zeta_D}(\y^{\zeta_D}(c) - \mmc) \nonumber\\
& \phantom{=} + \sum_{\tau=\zeta_D}^t (1-\alpha_\tau) \alpha_{\tau+1,t} {\Wc}^{t-\tau} (\xavc^{\tau+1} - \Pc \xavc^{\tau+1})\nonumber \\
& \phantom{=} + \sum_{\tau=\zeta_D}^t (1-\alpha_\tau) \alpha_{\tau+1,t} (\Pc \xavc^{\tau+1} - \mmc)\\
& =\alpha_{\zeta_D,t} \Wc^{t+1 - \zeta_D}(\y^{\zeta_D}(c) - \mmc) \nonumber\\
& \phantom{=} + \sum_{\tau=\zeta_D}^t (1-\alpha_\tau) \alpha_{\tau+1,t} {\Wc}^{t-\tau} (\xavc^{\tau+1} - \Pc \xavc^{\tau+1})\nonumber \\
& \phantom{=} + \sum_{\tau=\zeta_D}^t (1-\alpha_\tau) \alpha_{\tau+1,t} (\Pc \xavc^{\tau+1} - \mmc)\\
& =\alpha_{\zeta_D,t} \Wc^{t+1 - \zeta_D}(\y^{\zeta_D}(c) - \mmc) \nonumber\\
& \phantom{=} + \sum_{\tau=\zeta_D}^t (1-\alpha_\tau) \alpha_{\tau+1,t} {(\Wc - \Pc)}^{t-\tau} (\xavc^{\tau+1} - \Pc \xavc^{\tau+1})\nonumber \\
& \phantom{=} + \sum_{\tau=\zeta_D}^t (1-\alpha_\tau) \alpha_{\tau+1,t} (\Pc \xavc^{\tau+1} - \mmc). \label{e:estimation_error}
\end{align}
}

\subsection{Technical Results}
\begin{lemma}
\label{l:convolution}
For $0<\beta<1$
\begin{align}
\sum_{\tau=t_0}^t \frac{\beta^{t-\tau}}{\tau+1} & \in \bigO\left(\left(1 + \frac{1}{\ln{\frac{1}{\beta}}} \right)\frac{1}{t+1}\right)
\end{align}
\end{lemma}
\begin{proof}
\begin{align}
\label{e:convolution1}
 \sum_{\tau=t_0}^t \frac{\beta^{t-\tau}}{\tau+1}
 =\beta^{t+1} \sum_{\tau=t_0+1}^{t+1}\frac{\beta^{-\tau}}{\tau} 
 \end{align}

Let  $t^{\prime}=\max \left\{\left\lceil\frac{1}{ \ln \frac{1}{\beta}}\right\rceil, t_0+1\right\}$. For  $\tau_2 \geqslant \tau_1 \geqslant t^{\prime}$, 
$\frac{\beta^{-\tau_1}}{\tau_1} \leqslant \frac{\beta^{-\tau_2}}{\tau_2}$.

{\allowdisplaybreaks
\begin{align}
\sum_{\tau=t_0+1}^{t+1} \frac{\beta^{-\tau}}{\tau} 
& =\underbrace{\sum_{\tau=t_0}^{t^{\prime}-1} \frac{\beta^{-\tau}}{\tau}}_{C}+\sum_{\tau=t^{\prime}}^t \frac{\beta^{-\tau}}{\tau} +\frac{\beta^{-t-1}}{t+1}\\
& \leqslant C+\frac{\beta^{-t-1}}{t+1}+\int_{t^\prime}^{t+1} \frac{\beta^{-\tau}}{\tau} \diff \tau \\
& =C +\frac{\beta^{-t-1}}{t+1} +\int_{t^{\prime}}^{t+1} \frac{e^{\ln \frac{1}{\beta} \tau}}{\tau} \diff \tau \\
%& \sqrt{\ln \frac{1}{\beta}} \sqrt{\tau}=x \quad d x=\frac{\sqrt{\ln \frac{1}{\beta}}}{2} \frac{1}{\sqrt{\tau}} \diff \tau \\
 & =C+\frac{\beta^{-t-1}}{t+1}+ \int_{{t^\prime} \ln \frac{1}{\beta}} ^{(t+1){\ln \frac{1}{\beta}} } \frac{e^{x}}{x} \diff x \\
% & =\textcolor{magenta}{C'}+\frac{\beta^{-t-1}}{t+1}
% + \ln \frac{1}{\beta} \left(\int_{{t^\prime} \ln \frac{1}{\beta}} ^{(t+1){\ln \frac{1}{\beta}} } \frac{e^{x}}{x} \diff x -  \frac{e^{x}}{x}\Bigg\rvert_{{t^\prime} \ln \frac{1}{\beta}} ^{(t+1){\ln \frac{1}{\beta}} }\right)\\
& \leqslant C+\frac{\beta^{-t-1}}{t+1}  
+ 
\Ei{(t+1){\ln \frac{1}{\beta}}} - { \Ei{t'\,{\ln \frac{1}{\beta}}}}\\
& \leqslant C+\frac{\beta^{-t-1}}{t+1}  
+ 
\Ei{(t+1){\ln \frac{1}{\beta}}} \\
%& \leqslant C+\frac{\beta^{-t-1}}{t+1}
%+ \ln \frac{1}{\beta} \left(
% \Ei{(t+1){\ln \frac{1}{\beta}}} - \frac{e^{(t+1)\ln{\frac{1}{\beta}}}}{(t+1)\ln{\frac{1}{\beta}}} 
%  \right)\\
& \leqslant C+\frac{\beta^{-t-1}}{t+1}
+ 
\frac{e^{(t+1)\ln{\frac{1}{\beta}}}}{(t+1)\ln{\frac{1}{\beta}}} \left(1+ \frac{3}{(t+1)\ln{\frac{1}{\beta}}}\right) \\
% & = C+\frac{\beta^{-t-1}}{t+1}
% + \frac{3}{\ln{\frac{1}{\beta}}}
% \frac{e^{(t+1)\ln{\frac{1}{\beta}}}}{(t+1)^2}\\
 & = C+
\left(1 + \frac{1}{\ln{\frac{1}{\beta}}} + \frac{3}{(t+1) \ln{\frac{1}{\beta}}} \right)\frac{\beta^{-t-1}}{t+1}, \label{e:convolution2}
\end{align}
}
where $\Ei{t}\triangleq \int_{-\infty}^t \frac{e^x}{x} \diff x$ is the exponential integral. The third inequality follows  from the series representation 
\[\Ei{t}= \frac{e^t}{t} \left( \sum_{k=0}^n \frac{k !}{t^k} + e_n(t) \right),\]
where $e_n(t) \triangleq (n+1)! t e^{-t} \int_{-\infty}^t \frac{e^x}{x^{n+2}} \diff x$ for $n=0$. The remainder $e_n(t)$ can be bounded by the $n+1$-th term times the factor $1+ \sqrt{\pi}\frac{\Gamma(n/2+3/2)}{\Gamma(n/2+1)}$~\cite{dlmf}.
%(thanks also to \cite{mathstackexchange} for the pointer). 
This factor is smaller than $3$ for $n=0$.

Finally, from \eqref{e:convolution1} and \eqref{e:convolution2}, we obtain
\begin{align}
 \sum_{\tau=t_0}^t \frac{\beta^{t-\tau}}{\tau+1}
 \leqslant \beta^{t+1} C+
\left(1 + \frac{1}{\ln{\frac{1}{\beta}}} + \frac{3}{(t+1) \ln{\frac{1}{\beta}}} \right)\frac{1}{t+1}
\in \bigO\left(\left(1 + \frac{1}{\ln{\frac{1}{\beta}}} \right)\frac{1}{t+1}\right).
\end{align}
\end{proof}

\begin{lemma}
\label{l:sums}
Let $(\z_1, \z_2, \dots, \z_t, \dots)$ be a sequence of i.i.d. vectorial random variables in $\mathbb R^n$ with expected value $\bf 0$ and finite $4$-th moment $\E[\lVert z_i \rVert^4]$, and
$\{A(t_1, t_2), (t_1,t_2)\in \mathbb N^2\}$ a set of $n \times n$ matrices with bounded norms.
%, and $\{\gamma(t_1,t_2), (t_1,t_2)\in \mathbb N^2\} $ a set of bounded coefficients. 
Let $\bb_\tau \triangleq \frac{1}{\tau}\sum_{t=1}^\tau \z_t$. It holds:

\begin{align}
\E\left[\sum_{\tau_1, \tau_2, \tau_3, \tau_4=t_0}^t \bb_{\tau_1}^\top A(\tau_1, \tau_2) \bb_{\tau_2} \bb_{\tau_3}^\top A(\tau_3, \tau_4) \bb_{\tau_4}\right] 
%\leqslant 3 \E\left[\lVert \z \rVert^4 \right] \sum_{\tau_1, \tau_2, \tau_3, \tau_4=1}^t \frac{\lVert A(\tau_1, \tau_2) \rVert \cdot \lVert A(\tau_3, \tau_4) \rVert }{ \tau_2 \tau_3 } \]
\leqslant 2 \E\left[\lVert \z \rVert^4 \right] \sum_{\tau_1, \tau_2, \tau_3, \tau_4=t_0}^t & \left(\frac{\lVert A(\tau_1, \tau_2) \rVert \cdot \lVert A(\tau_3, \tau_4) \rVert }{ \tau_2 \tau_3 } \right. \nonumber\\
& \left. + \frac{\lVert A(\tau_1, \tau_2) \rVert \cdot \lVert A(\tau_3, \tau_4) \rVert }{ \tau_2 \tau_4 } \right) \end{align}
\end{lemma}

\begin{proof}
We will omit the indices when they run from $1$ to $t$ and denote by $\z$ and $\z'$ two generic independent random variables distributed as $\z_\tau$. 
{\allowdisplaybreaks
\begin{align}
& \E\left[\sum_{\tau_1, \tau_2, \tau_3, \tau_4=t_0}^t \bb_{\tau_1}^\top A(\tau_1, \tau_2) \bb_{\tau_2} \bb_{\tau_3}^\top A(\tau_3, \tau_4) \bb_{\tau_4}\right] \\
 & = \sum_{\tau_1, \tau_2, \tau_3, \tau_4=t_0}^t \frac{1}{\tau_1 \tau_2 \tau_3 \tau_4} \sum_{t_1=1}^{\tau_1}\sum_{t_2=1}^{\tau_2}\sum_{t_3=1}^{\tau_3}\sum_{t_4=1}^{\tau_4}
 \E\left[\z_{t_1}^\top A(\tau_1, \tau_2) \z_{t_2} \z_{t_3}^\top A(\tau_3, \tau_4) \z_{t_4}\right]\\
 %& = \sum_{\tau_1, \tau_2, \tau_3, \tau_4=t_0}^t \frac{1}{\tau_1 \tau_2 \tau_3 \tau_4} \sum_{t_1, t_2, t_3, t_4} 
 %\bm{1}_{t_1\leqslant \tau_1} \bm{1}_{t_2 \leqslant \tau_2} \bm{1}_{t_3\leqslant \tau_3} \bm{1}_{t_4\leqslant \tau_4}
 %\E\left[\z_{t_1}^\top A(\tau_1, \tau_2) \z_{t_2} \z_{t_3}^\top A(\tau_3, \tau_4) \z_{t_4}\right]\\
 & = \sum_{\tau_1, \tau_2, \tau_3, \tau_4=t_0}^t \frac{1}{\tau_1 \tau_2 \tau_3 \tau_4} \sum_{t_1, t_2, t_3, t_4} 
 \bm{1}_{t_1\leqslant \tau_1} \bm{1}_{t_2 \leqslant \tau_2} \bm{1}_{t_3\leqslant \tau_3} \bm{1}_{t_4\leqslant \tau_4}
 \E\left[\z_{t_1}^\top A(\tau_1, \tau_2) \z_{t_2} \z_{t_3}^\top A(\tau_3, \tau_4) \z_{t_4}\right]\\
  & = \sum_{\tau_1, \tau_2, \tau_3, \tau_4=t_0}^t \frac{1}{\tau_1 \tau_2 \tau_3 \tau_4}  \sum_{t_1, t_2, t_3, t_4} 
 \bm{1}_{t_1\leqslant \tau_1, t_2 \leqslant \tau_2, t_3\leqslant \tau_3, t_4\leqslant \tau_4}  \times \Big( \bm{1}_{t_1=t_2=t_3=t_4}
 \E\left[\z_{t_1}^\top A(\tau_1, \tau_2) \z_{t_2} \z_{t_3}^\top A(\tau_3, \tau_4) \z_{t_4}\right]\nonumber\\
 & \phantom{= \sum}+ \bm{1}_{t_1=t_2, t_3=t_4, t_1\neq t_3}
 \E\left[\z_{t_1}^\top A(\tau_1, \tau_2) \z_{t_2} \z_{t_3}^\top A(\tau_3, \tau_4) \z_{t_4}\right]\nonumber\\ 
 & \phantom{= \sum}+ \bm{1}_{t_1=t_3, t_2=t_4, t_1\neq t_2}
 \E\left[\z_{t_1}^\top A(\tau_1, \tau_2) \z_{t_2} \z_{t_3}^\top A(\tau_3, \tau_4) \z_{t_4}\right] \nonumber \\
  &{ \phantom{= \sum}+\bm{1}_{t_1=t_4, t_2=t_3, t_1\neq t_2}
 \E\left[\z_{t_1}^\top A(\tau_1, \tau_2) \z_{t_2} \z_{t_3}^\top A(\tau_3, \tau_4) \z_{t_4}\right]\Big) } \label{e:sums_expected}\\
 & = \sum_{\tau_1, \tau_2, \tau_3, \tau_4=t_0}^t \frac{1}{\tau_1 \tau_2 \tau_3 \tau_4}  \sum_{t_1} 
 \bm{1}_{t_1\leqslant \tau_1, t_1 \leqslant \tau_2, t_1\leqslant \tau_3, t_1\leqslant \tau_4}
  \E\left[\z_{t_1}^\top A(\tau_1, \tau_2) \z_{t_1} \z_{t_1}^\top A(\tau_3, \tau_4) \z_{t_1}\right]\nonumber\\
 & \phantom{=} + \sum_{\tau_1, \tau_2, \tau_3, \tau_4=t_0}^t 
 \frac{1}{\tau_1 \tau_2 \tau_3 \tau_4}  \sum_{t_1, t_3} 
 \bm{1}_{t_1\leqslant \tau_1, t_1 \leqslant \tau_2, t_3\leqslant \tau_3, t_3\leqslant \tau_4}\bm{1}_{t_1\neq t_3}
 \E\left[\z_{t_1}^\top A(\tau_1, \tau_2) \z_{t_1} \z_{t_3}^\top A(\tau_3, \tau_4) \z_{t_3}\right]\nonumber\\ 
 & \phantom{=} + \sum_{\tau_1, \tau_2, \tau_3, \tau_4=t_0}^t \frac{1}{\tau_1 \tau_2 \tau_3 \tau_4}  \sum_{t_1, t_2} 
 \bm{1}_{t_1\leqslant \tau_1, t_2 \leqslant \tau_2, t_1\leqslant \tau_3, t_2\leqslant \tau_4}\bm{1}_{t_1\neq t_2}
 \E\left[\z_{t_1}^\top A(\tau_1, \tau_2) \z_{t_2} \z_{t_1}^\top A(\tau_3, \tau_4) \z_{t_2}\right] \nonumber \\
 & {\phantom{=} + \sum_{\tau_1, \tau_2, \tau_3, \tau_4=t_0}^t \frac{1}{\tau_1 \tau_2 \tau_3 \tau_4}  \sum_{t_1, t_2} 
 \bm{1}_{t_1\leqslant \tau_1, t_2 \leqslant \tau_2, t_2\leqslant \tau_3, t_1\leqslant \tau_4}\bm{1}_{t_1\neq t_2}
 \E\left[\z_{t_1}^\top A(\tau_1, \tau_2) \z_{t_2} \z_{t_2}^\top A(\tau_3, \tau_4) \z_{t_1}\right]} \\
 & = \sum_{\tau_1, \tau_2, \tau_3, \tau_4=t_0}^t \frac{1}{\tau_1 \tau_2 \tau_3 \tau_4}  \sum_{t_1} 
 \bm{1}_{t_1\leqslant \tau_1, t_1 \leqslant \tau_2, t_1\leqslant \tau_3, t_1\leqslant \tau_4}
  \E\left[\z^\top A(\tau_1, \tau_2) \z \z^\top A(\tau_3, \tau_4) \z\right]\nonumber\\
 & \phantom{=}+ \sum_{\tau_1, \tau_2, \tau_3, \tau_4=t_0}^t 
 \frac{1}{\tau_1 \tau_2 \tau_3 \tau_4}  \sum_{t_1, t_3} 
 \bm{1}_{t_1\leqslant \tau_1, t_1 \leqslant \tau_2, t_3\leqslant \tau_3, t_3\leqslant \tau_4}\bm{1}_{t_1\neq t_3}
 \E\left[\z^\top A(\tau_1, \tau_2) \z\right] \E\left[ {\z'}^\top A(\tau_3, \tau_4) {\z'}\right]\nonumber\\ 
 & \phantom{=} + \sum_{\tau_1, \tau_2, \tau_3, \tau_4=t_0}^t \frac{1}{\tau_1 \tau_2 \tau_3 \tau_4}  \sum_{t_1, t_2} 
 \bm{1}_{t_1\leqslant \tau_1, t_2 \leqslant \tau_2, t_1\leqslant \tau_3, t_2\leqslant \tau_4}\bm{1}_{t_1\neq t_2}
 \E\left[\z^\top A(\tau_1, \tau_2) \z' \z^\top A(\tau_3, \tau_4) \z'\right] \nonumber\\
 & {\phantom{=} + \sum_{\tau_1, \tau_2, \tau_3, \tau_4=t_0}^t \frac{1}{\tau_1 \tau_2 \tau_3 \tau_4}  \sum_{t_1, t_2} 
 \bm{1}_{t_1\leqslant \tau_1, t_2 \leqslant \tau_2, t_2\leqslant \tau_3, t_1\leqslant \tau_4}\bm{1}_{t_1\neq t_2}
 \E\left[\z^\top A(\tau_1, \tau_2) \z' \z'^\top A(\tau_3, \tau_4) \z\right] }\\ 
 & = \sum_{\tau_1, \tau_2, \tau_3, \tau_4=t_0}^t \frac{\min\{\tau_1, \tau_2, \tau_3, \tau_4\}}{\tau_1 \tau_2 \tau_3 \tau_4}  
  \E\left[\z^\top A(\tau_1, \tau_2) \z \z^\top A(\tau_3, \tau_4) \z\right]\nonumber\\
 & \phantom{=}+ \sum_{\tau_1, \tau_2, \tau_3, \tau_4=t_0}^t 
 \frac{\min\{\tau_1, \tau_2\}(\min\{\tau_3, \tau_4\}-1)}{\tau_1 \tau_2 \tau_3 \tau_4}  
 \E\left[\z^\top A(\tau_1, \tau_2) \z\right] \E\left[ {\z'}^\top A(\tau_3, \tau_4) {\z'}\right]\nonumber\\ 
 & \phantom{=}+ \sum_{\tau_1, \tau_2, \tau_3, \tau_4=t_0}^t \frac{\min\{\tau_1, \tau_3\}(\min\{\tau_2, \tau_4\}-1)}{\tau_1 \tau_2 \tau_3 \tau_4}  
 \E\left[\z^\top A(\tau_1, \tau_2) \z' \z^\top A(\tau_3, \tau_4) \z'\right] \nonumber\\
  & {\phantom{=}+ \sum_{\tau_1, \tau_2, \tau_3, \tau_4=t_0}^t \frac{\min\{\tau_1, \tau_4\}(\min\{\tau_2, \tau_3\}-1)}{\tau_1 \tau_2 \tau_3 \tau_4}  
 \E\left[\z^\top A(\tau_1, \tau_2) \z' (\z')^\top A(\tau_3, \tau_4) \z\right] }
 \label{e:sums_three_terms},
 \end{align}
 }
 where \eqref{e:sums_expected} follows from the independence of the variables $\{\z_t\}_{t \in \mathbb N}$ and the fact that $\E[\z_t] = \bm{0}$ for any $t$.
 
 Now we can upperbound the three terms in \eqref{e:sums_three_terms}.
 \begin{align}
 & \E\left[\sum_{\tau_1, \tau_2, \tau_3, \tau_4=t_0}^t \bb_{\tau_1}^\top A(\tau_1, \tau_2) \bb_{\tau_2} \bb_{\tau_3}^\top A(\tau_3, \tau_4) \bb_{\tau_4}\right]\\
 & \leqslant \sum_{\tau_1, \tau_2, \tau_3, \tau_4=t_0}^t \frac{1}{ \tau_2 \tau_3 \tau_4}  
  \E\left[\z^\top A(\tau_1, \tau_2) \z \z^\top A(\tau_3, \tau_4) \z\right]\nonumber\\
 & \phantom{\leqslant}+ \sum_{\tau_1, \tau_2, \tau_3, \tau_4=t_0}^t 
 \frac{1}{\tau_2 \tau_3}  
 \E\left[\z^\top A(\tau_1, \tau_2) \z\right] \E\left[ {\z'}^\top A(\tau_3, \tau_4) {\z'}\right]\nonumber\\ 
 & \phantom{\leqslant} + \sum_{\tau_1, \tau_2, \tau_3, \tau_4=t_0}^t \frac{1}{\tau_2 \tau_3}  
 \E\left[\z^\top A(\tau_1, \tau_2) \z' \z^\top A(\tau_3, \tau_4) \z'\right] \nonumber\\
& {\phantom{\leqslant} + \sum_{\tau_1, \tau_2, \tau_3, \tau_4=t_0}^t \frac{1}{\tau_2 \tau_4}  
 \E\left[\z^\top A(\tau_1, \tau_2) \z' \z'^\top A(\tau_3, \tau_4) \z\right]}\\
 & \leqslant \sum_{\tau_1, \tau_2, \tau_3, \tau_4=t_0}^t \frac{1}{ \tau_2 \tau_3 \tau_4}  
  \lVert A(\tau_1, \tau_2) \rVert \cdot \lVert A(\tau_3, \tau_4) \rVert \cdot \E\left[\lVert \z \rVert^4 \right] \nonumber\\
 & \phantom{\leqslant}+ \sum_{\tau_1, \tau_2, \tau_3, \tau_4=t_0}^t 
 \frac{1}{\tau_2 \tau_3}  
 \lVert A(\tau_1, \tau_2) \rVert \cdot \lVert A(\tau_3, \tau_4) \rVert \cdot \E\left[\lVert \z \rVert^2 \right]^2\nonumber\\ 
 & \phantom{\leqslant} + \sum_{\tau_1, \tau_2, \tau_3, \tau_4=t_0}^t \frac{1}{\tau_2 \tau_3}  \lVert A(\tau_1, \tau_2) \rVert \cdot \lVert A(\tau_3, \tau_4) \rVert \cdot
 \E\left[\lVert \z \rVert^2 \cdot \lVert \z' \rVert^2\right] \nonumber\\
 & {\phantom{\leqslant} + \sum_{\tau_1, \tau_2, \tau_3, \tau_4=t_0}^t \frac{1}{\tau_2 \tau_4}  \lVert A(\tau_1, \tau_3) \rVert \cdot \lVert A(\tau_3, \tau_4) \rVert \cdot
 \E\left[\lVert \z \rVert^2 \cdot \lVert \z' \rVert^2\right] }\\ 
 & {\leqslant 2 \E\left[\lVert \z \rVert^4 \right] \sum_{\tau_1, \tau_2, \tau_3, \tau_4=t_0}^t \left(\frac{\lVert A(\tau_1, \tau_2) \rVert \cdot \lVert A(\tau_3, \tau_4) \rVert }{ \tau_2 \tau_3 } + \frac{\lVert A(\tau_1, \tau_2) \rVert \cdot \lVert A(\tau_3, \tau_4) \rVert }{ \tau_2 \tau_4 } \right)}. \nonumber
%\textcolor{magenta}{ & \phantom{\leqslant}   + \E\left[\lVert \z \rVert^4 \right] \sum_{\tau_1, \tau_2, \tau_3, \tau_4=t_0}^t \frac{\lVert A(\tau_1, \tau_2) \rVert \cdot \lVert A(\tau_3, \tau_4) \rVert }{ \tau_1 \tau_3 } }\\
 % & \textcolor{magenta}{  
 % \leqslant 4 \E\left[\lVert \z \rVert^4 \right] 
 % \max \left( 
 % \sum_{\tau_1, \tau_2, \tau_3, \tau_4=1}^t \frac{\lVert A(\tau_1, \tau_2) \rVert \cdot \lVert A(\tau_3, \tau_4) \rVert }{ \tau_2 \tau_3 },  
 % \sum_{\tau_1, \tau_2, \tau_3, \tau_4=1}^t \frac{\lVert A(\tau_1, \tau_3) \rVert \cdot \lVert A(\tau_3, \tau_4) \rVert }{ \tau_1 \tau_3 }
 % \right) }. 
\end{align}
\end{proof}

We now particularize the result of Lemma~\ref{l:sums} to the two cases of interest for what follows.

\begin{corollary}
\label{cor:sums_stochastic_matrix}
Let $(\z_1, \z_2, \dots, \z_t, \dots)$ be a sequence of i.i.d. vectorial random variables in $\mathbb R^n$ with expected value $\bf 0$ and finite $4$-th moment $\E[\lVert z_i \rVert^4]$ and $\{A(t_1, t_2), (t_1, t_2) \in \mathbb N^2\}$ a set of $n \times n$ symmetric stochastic matrices. Let $\bb_\tau \triangleq \frac{1}{\tau}\sum_{t=1}^\tau \z_t$. It holds:

\[\E\left[\sum_{\tau_1, \tau_2, \tau_3, \tau_4=t_0}^t \bb_{\tau_1}^\top A(\tau_1, \tau_2) \bb_{\tau_2} \bb_{\tau_3}^\top A(\tau_3, \tau_4) \bb_{\tau_4}\right] 
\leqslant 4 \E\left[\lVert \z \rVert^4 \right] t^2 (1+ \ln t)^2. \]
\end{corollary}
\begin{proof}
It is sufficient to observe that $\lVert A(\tau_1, \tau_2)\rVert= 1$ and that $\sum_{\tau=1}^t \frac{1}{\tau} \leqslant 1 + \ln t$.
\end{proof}

\begin{corollary}
\label{cor:sums_stochastic_matrix_minus_projector}
Let $(\z_1, \z_2, \dots, \z_t, \dots)$ be a sequence of i.i.d. vectorial random variables in $\mathbb R^n$ with expected value $\bf 0$ and finite $4$-th moment $\E[\lVert z_i \rVert^4]$ and $\{A(t_1, t_2), (t_1, t_2) \in \mathbb N^2\}= \beta^{2t-t_1-t_2} B^{2t-t_1-t_2}$, where $\beta \in [0,1]$, B a symmetric matrix.
%is an irreducible symmetric stochastic matrix with non-negative diagonal elements, 
%and $P = \frac{1}{n} \mathbf{1}\mathbf{1}^\top$. 
Let  $\rho(B)$ denote the spectral norm of $B$ and $\bb_\tau \triangleq \frac{1}{\tau}\sum_{t=1}^\tau \z_t$. If $\beta \rho(B)<1$, then it holds:
%Let  $\lambda_2(B)$ denote the module of the second largest eigenvalue of $B$ in module and $\bb_\tau \triangleq \frac{1}{\tau}\sum_{t=1}^\tau \z_t$. If $\beta \lambda_2(B)<1$, then  holds:

\[\E\left[\sum_{\tau_1, \tau_2, \tau_3, \tau_4=t_0}^t \bb_{\tau_1}^\top A(\tau_1, \tau_2) \bb_{\tau_2} \bb_{\tau_3}^\top A(\tau_3, \tau_4) \bb_{\tau_4}\right] 
\in \bigO \left( \E\left[\lVert \z \rVert^4 \right] \frac{1}{(1-\beta \rho(B))^2} \left(1 + \frac{1}{\ln{\frac{1}{\beta \rho(B)}}} \right)^2\frac{1}{(t+1) ^2}\right). \]
\end{corollary}
\begin{proof}
% We observe that the largest module of the eigenvalues of $B$ is equal to $1$, because $B$ is symmetric and stochastic. It follows that $\rho(B)\le 1$. The condition $\beta \rho(B) < 1$ excludes then that both values may be equal to 1.
As $B$ is symmetric, then $\lVert B \rVert = \rho(B)$.
%As $B$ is symmetric and stochastic, then the module of its largest eigenvalue is equal to $1$. Moreover, $B$ is  irreducible with non-negative elements on the diagonal, then it is primitive, i.e., $1$ is the only eigenvalue on the unit circle. 
From Lemma~\ref{l:sums}, we obtain
\begin{align}
  & \E\left[\sum_{\tau_1, \tau_2, \tau_3, \tau_4=t_0}^t \bb_{\tau_1}^\top A(\tau_1, \tau_2) \bb_{\tau_2} \bb_{\tau_3}^\top A(\tau_3, \tau_4) \bb_{\tau_4}\right] \nonumber\\ 
 & \leqslant 2 \E\left[\lVert \z \rVert^4 \right] \left(\sum_{\tau_1, \tau_2, \tau_3, \tau_4=1}^t \frac{(\beta \lambda_2(B))^{4 t- \tau_1-\tau_2-\tau_3-\tau_4}}{ \tau_2 \tau_3 }
 + \sum_{\tau_1, \tau_2, \tau_3, \tau_4=1}^t \frac{(\beta \lambda_2(B))^{4 t- \tau_1-\tau_2-\tau_3-\tau_4}}{ \tau_2 \tau_4 }
 \right)
 \\   
 & =4 \E\left[\lVert \z \rVert^4 \right] \sum_{\tau_1, \tau_2, \tau_3, \tau_4=1}^t \frac{(\beta \lambda_2(B))^{4 t- \tau_1-\tau_2-\tau_3-\tau_4}}{ \tau_2 \tau_3 }
 \\   
& = 4 \E\left[\lVert \z \rVert^4 \right] \left(\sum_{\tau_1=1}^t (\beta \lambda_2(B))^{ t- \tau_1}\right)^2 \left(\sum_{\tau_2=1}^t \frac{(\beta \lambda_2(B))^{ t- \tau_2}}{\tau_2}\right)^2\\
& \le 4 \E\left[\lVert \z \rVert^4 \right] \frac{1}{(1-\beta \lambda_2(B))^2} \left(\sum_{\tau_2=1}^t \frac{(\beta \lambda_2(B))^{ t- \tau_2}}{\tau_2}\right)^2\\
& \in \bigO \left( \E\left[\lVert \z \rVert^4 \right] \frac{1}{(1-\beta \lambda_2(B))^2} \left(1 + \frac{1}{\ln{\frac{1}{\beta \lambda_2(B)}}} \right)^2\frac{1}{(t+1) ^2}\right),
\end{align}
 where in the last step we used Lemma~\ref{l:convolution}.
\end{proof}

\begin{lemma}
    \label{l:second_eigenvalue}
    Let $W$ be an $n \times n$ symmetric, stochastic, and irreducible matrix and $P = \frac{1}{n} \mathbf 1_n \mathbf 1_n^\top$, where $\mathbf{1}_n$ is an $n$-dimensional vector whose elements are all equal to $1$, then $\rho(W-P) = \lambda_2(W)<1$.
\end{lemma}
\begin{proof}
As $W$ is symmetric and stochastic, then the module of its largest eigenvalue is equal to $1$.
The vector $\mathbf 1_n$ is both a left and right eigenvector of $W$ relative to the simple eigenvalue $1$. Then, $P$ is the projector onto the null space of $W- I$ along the range of $W- I$ \cite{meyer01}[p.518]. The spectral theorem leads us to conclude that 
the eigenvalues of $W-P$ (counted with their multiplicity) are then all eigenvalues of $W$ except $1$ and with the addition of $0$. We can then conclude that $\lVert W-P \rVert = \rho(W-P) = \lambda_2(W)$. 
$W$ is irreducible with non-negative elements on the diagonal, then it is primitive~\cite{meyer01}[Example 8.3.3], i.e., $1$ is the only eigenvalue on the unit circle. 
    
\end{proof}

\begin{lemma}
    \label{l:nasty_inequality}
    For any  $a> 0 $, if $x \ge \max\{a \ln^2 a,1\}$ then $x\geqslant \frac{a}{4} \ln^2 x$.    
\end{lemma}
\begin{proof}
    We consider first $a\ge 1$.
    \begin{align}
        x \geqslant a \ln^2 a & \implies \sqrt{x} \geqslant \sqrt{a} |\ln a | \underset{a \ge 1}{=}  \sqrt{a} \ln a\\
        & \implies \sqrt{x} \geqslant \sqrt{a} \ln \sqrt{x} = \frac{\sqrt{a}}{2} \ln x \label{e:use_shalev}\\
        & \implies x \geqslant \frac{a}{4} \ln^2 x,
    \end{align}
    where \eqref{e:use_shalev} follows from Lemma~A.1 in \cite{shalev_understanding}.
    For $a < 1$, it is easy to check that for $x \geqslant 1$, $x \geqslant\frac{a}{4} \ln^2 x$ holds unconditionally.
\end{proof}

\subsection{Bounding the $4$-th Moment of the Estimation Error.} 
\label{app:fourth_moment}
For convenience, we omit from now on the dependence on the specific clustered component $c$.

\begin{theorem}
    \label{thm:fourth_moment}
Let $\lambda_2(W)$ denote the module of the second largest eigenvalue in module of $W$. It holds:
\begin{align}
\E\left[\lVert {\y}^{t+1} - \mm \rVert^4 \right]  & \in  \bigO\left( \sup_{W_1, W_2, \cdots, W_{\zeta_D}}\E\left[\lVert \y^{\zeta_D} - \mm \rVert^4 \right]
\alpha^{4 t}\right) \nonumber\\ 
&  + \bigO\left(\E\left[\lVert \x -  \mm \rVert^4 \right] \frac{(1-\alpha)^4}{(1-\alpha \lambda_2(W))^2} \left(1 + \frac{1}{\ln{\frac{1}{\alpha \lambda_2(W)}}} \right)^2\frac{1}{(t+1) ^2}\right)\nonumber\\
& + \bigO\left(\E\left[\lVert P \x -  \mm \rVert^4 \right] {(1-\alpha)^2} \left(1 + \frac{1}{\ln{\frac{1}{\alpha }}} \right)^2\frac{1}{(t+1) ^2}\right), & \textrm{if  } \alpha_t=\alpha,
\label{e:fourth_moment_alpha_const}
\end{align}

\begin{align}
\E\left[\lVert {\y}^{t+1} - \mm \rVert^4 \right] & \in  \bigO\left( \sup_{W_1, W_2, \cdots, W_{\zeta_D}}\E\left[\lVert \y^{\zeta_D} - \mm \rVert^4 \right]
\frac{1}{(t+1)^4} \right)\nonumber\\ 
&  + \bigO \left(\E\left[\lVert \x -  \mm \rVert^4 \right] \frac{1}{(1-\lambda_2(W))^2} \left(1 + \frac{1}{\ln{\frac{1}{\lambda_2(W) }}} \right)^2\frac{1}{(t+1) ^4}\right)\nonumber\\
& + \bigO \left(\E\left[\lVert P \x -  \mm \rVert^4 \right] \left(\frac{1 + \ln t}{ 1 + t }\right)^2 \right), & \textrm{if  } \alpha_t=\frac{t}{t+1}.
\label{e:fourth_moment_alpha_t}
\end{align}
\end{theorem}

\begin{proof}
$W$ is  irreducible (the graph component is connected and $W_{i,j}>0$ for each link), then by Lemma~\ref{l:second_eigenvalue}, $\lambda_2(W)< 1$.
For $\alpha_t=\alpha<1$, it would be sufficient to observe that $\lambda_2(W)\leqslant \rho(W)=1$, but for $\alpha_t=\frac{t}{t+1}$, we need the strict inequality.

Our starting point is \eqref{e:estimation_error}, which we repeat here (omitting the dependence on the specific clustered component $c$):
\begin{align}
{\y}^{t+1} - \mm&  =\alpha_{\zeta_D,t} W^{t+1 - \zeta_D}(\y^{\zeta_D} - \mm) \nonumber\\
& \phantom{=} + \sum_{\tau=\zeta_D}^t (1-\alpha_\tau) \alpha_{\tau+1,t} {(W - P)}^{t-\tau} (\xav^{\tau+1} - P\xav^{\tau+1})\nonumber \\
& \phantom{=} + \sum_{\tau=\zeta_D}^t (1-\alpha_\tau) \alpha_{\tau+1,t} (P \xav^{\tau+1} - \mm). 
\end{align}

Applying twice $(\sum_{i=1}^n a_i)^2 \le n \sum_{i=1}^n a_i^2$ with $n=3$, and applying the expectation we obtain:
{\allowdisplaybreaks
\begin{align}
\E\left[\lVert {\y}^{t+1} - \mm \rVert^4 \right] &   \le 27 \cdot \alpha_{\zeta_D,t}^4 \cdot \E\left[\lVert W\rVert^{4(t+1  - \zeta_D)} \cdot \lVert \y^{\zeta_D} - \mm \rVert^4 \right]\nonumber\\
& \phantom{=} + 27 \E\left[\left\lVert \sum_{\tau=\zeta_D}^t (1-\alpha_\tau) \alpha_{\tau+1,t} {(W - P)}^{t-\tau} (\xav^{\tau+1} - P\xav^{\tau+1})\right\rVert^4 \right] \nonumber \\
& \phantom{=} + 27 \E\left[\left\lVert \sum_{\tau=\zeta_D}^t (1-\alpha_\tau) \alpha_{\tau+1,t} (P \xav^{\tau+1} - \mm) \right\rVert^4\right]\\
& \le 27 \cdot \underbrace{\alpha_{\zeta_D,t}^4 \cdot  \E\left[\lVert \y^{\zeta_D} - \mm \rVert^4 \right]}_{C_1} \nonumber\\
& \phantom{=} + 27 \underbrace{\E\left[\left\lVert \sum_{\tau=\zeta_D}^t (1-\alpha_\tau) \alpha_{\tau+1,t} {(W - P)}^{t-\tau} (\xav^{\tau+1} - P\xav^{\tau+1})\right\rVert^4 \right]}_{C_2} \nonumber \\
& \phantom{=} + 27 \underbrace{\E\left[\left\lVert \sum_{\tau=\zeta_D}^t (1-\alpha_\tau) \alpha_{\tau+1,t} (P \xav^{\tau+1} - \mm) \right\rVert^4 \right]}_{C_3}, 
\end{align}
}
where in the last step we took advantage of the fact that $W$ is doubly stochastic and symmetric and then $\lVert W \rVert=1$.

We now move to bound the three terms $C_1$, $C_2$, and $C_3$. We observe that 
\begin{align}
\alpha_{t_0+1,t} = \begin{cases}
    \alpha^{t-t_0}, & \text{if }\alpha_t = \alpha,\\
    \frac{t_0+1}{t} & \text{if }\alpha_t = \frac{t}{t+1},
    \end{cases}
\end{align}
and
\begin{align}
(1-\alpha_{t_0})\alpha_{t_0+1,t} = \begin{cases}
    (1-\alpha)\alpha^{t-t_0}, & \text{if }\alpha_t = \alpha,\\
    \frac{1}{t} & \text{if }\alpha_t = \frac{t}{t+1},
    \end{cases}
\end{align}

\begin{align}
C_1 & \leqslant \alpha_{\zeta_D,t}^4   \sup_{W_1, W_2, \cdots, W_{\zeta_D}}\E\left[\lVert \y^{\zeta_D} - \mm \rVert^4 \right]
 \in  \begin{cases}
    \bigO\left(\sup_{W_1, W_2, \cdots, W_{\zeta_D}}\E\left[\lVert \y^{\zeta_D} - \mm \rVert^4 \right]
\alpha^{4 t}\right), & \text{if }\alpha_t = \alpha,\\
    \bigO\left(\sup_{W_1, W_2, \cdots, W_{\zeta_D}}\E\left[\lVert \y^{\zeta_D} - \mm \rVert^4 \right]
\frac{1}{t^4}\right) & \text{if }\alpha_t = \frac{t}{t+1}.
\end{cases}
\end{align}
{\allowdisplaybreaks
\begin{align}
    C_2  = \E &\left[ \left(\sum_{\tau_1=\zeta_D}^t (1-\alpha_{\tau_1}) \alpha_{\tau_1+1,t} {(W - P)}^{t-\tau_1} (\xav^{\tau_1+1} - P\xav^{\tau_1+1}) \right)^\top\right.\nonumber\\
    & \left(\sum_{\tau_2=\zeta_D}^t (1-\alpha_{\tau_2}) \alpha_{\tau_2+1,t} {(W - P)}^{t-\tau_2} (\xav^{\tau_2+1} - P\xav^{\tau_2+1}) \right)\nonumber\\
    & \left(\sum_{\tau_3=\zeta_D}^t (1-\alpha_{\tau_3}) \alpha_{\tau_3+1,t} {(W - P)}^{t-\tau_3} (\xav^{\tau_3+1} - P\xav^{\tau_3+1}) \right)^\top\nonumber\\
    & \left.\left(\sum_{\tau_4=\zeta_D}^t (1-\alpha_{\tau_4}) \alpha_{\tau_4+1,t} {(W - P)}^{t-\tau_4} (\xav^{\tau_4+1} - P\xav^{\tau_4+1}) \right)
    \right] \\
    = \E &\left[ 
    \sum_{\tau_1, \tau_2, \tau_3, \tau_4=\zeta_D}^t
    (1-\alpha_{\tau_1})  \alpha_{\tau_1+1,t} 
    (1-\alpha_{\tau_2}) \alpha_{\tau_2+1,t} \right. (1-\alpha_{\tau_3})  \alpha_{\tau_3+1,t} 
    (1-\alpha_{\tau_4}) \alpha_{\tau_4+1,t} \nonumber\\
    & \phantom{\sum_{\tau=\zeta_D}^t} (\xav^{\tau_1+1} - P\xav^{\tau_1+1})^\top {(W - P)}^{2t-\tau_1- \tau_2} (\xav^{\tau_1+1} - P\xav^{\tau_1+1}) \nonumber\\
    & \left. \phantom{\sum_{\tau=\zeta_D}^t} (\xav^{\tau_3+1} - P\xav^{\tau_3+1})^\top {(W - P)}^{2t-\tau_3- \tau_4} (\xav^{\tau_4+1} - P\xav^{\tau_4+1}) 
    \right] \nonumber\\
    & \in \begin{cases}
        \bigO \left( \E\left[\lVert \x -  \mm \rVert^4 \right] \frac{(1-\alpha)^4}{(1-\alpha \lambda_2(W))^2} \left(1 + \frac{1}{\ln{\frac{1}{\alpha \lambda_2(W)}}} \right)^2\frac{1}{(t+1) ^2}\right), & \text{if } \alpha_t=\alpha,\\
        \bigO \left( \E\left[\lVert \x-  \mm \rVert^4 \right] \frac{1}{(1- \lambda_2(W))^2} \left(1 + \frac{1}{\ln{\frac{1}{ \lambda_2(W)}}} \right)^2\frac{1}{(t+1) ^4}\right), & \text{if } \alpha_t=\frac{t}{t+1},\label{e:bound_c2}\\
    \end{cases}
\end{align}
}
where the last result follows from observing that $\rho(W-P) = \lambda_2(W)<1$ and $\lVert \x - P \x \rVert^4 \le \lVert \x - \mm\rVert^4 $ and then applying Corollary~\ref{cor:sums_stochastic_matrix_minus_projector} with $\z = \x - P\x$, $B= W-P$ and 1)~$\beta=\alpha$ for $\alpha_t=\alpha$, 2)~$\beta=1$ for $\alpha_t=\frac{t}{t+1}$.

%As $B-P$ is symmetric $\lVert B-P \rVert = \rho(B-P)$.

The calculations to bound $C_3$ are similar:
\begin{align}
    C_3  =  \E &\left[ 
    \sum_{\tau_1, \tau_2, \tau_3, \tau_4=\zeta_D}^t
    (1-\alpha_{\tau_1})  \alpha_{\tau_1+1,t} 
    (1-\alpha_{\tau_2}) \alpha_{\tau_2+1,t} \right. (1-\alpha_{\tau_3})  \alpha_{\tau_3+1,t} 
    (1-\alpha_{\tau_4}) \alpha_{\tau_4+1,t} \nonumber\\
    & \left. \phantom{\sum_{\tau=\zeta_D}^t} (P \xav^{\tau_1+1} - \mm)^\top  (P \xav^{\tau_1+1} - \mm)  \phantom{\sum_{\tau=\zeta_D}^t} (P \xav^{\tau_3+1} - \mm)^\top  (P \xav^{\tau_4+1} - \mm) 
    \right] \nonumber\\
    & \in \begin{cases}
        \bigO \left( \E\left[\lVert P \x -  \mm \rVert^4 \right] {(1-\alpha)^2} \left(1 + \frac{1}{\ln{\frac{1}{\alpha }}} \right)^2\frac{1}{(t+1) ^2}\right), & \text{if } \alpha_t=\alpha,\\
        \bigO \left( \E\left[\lVert P \x-  \mm \rVert^4 \right] 
        \left(\frac{1 + \ln t}{1+t}\right)^2 \right), & \text{if }\alpha_t=\frac{t}{t+1}.\label{e:bound_c3}\\
    \end{cases}
\end{align}
In this case, we apply 1)~Corollary~\ref{cor:sums_stochastic_matrix_minus_projector} with $\z = \x - P\x$, $\beta=\alpha$, and $B=I$, for $\alpha_t=\alpha$, and 2)~Corollary~\ref{cor:sums_stochastic_matrix} with $\z = \x - P\x$ and $A(t_1,t_2)= I,$  $\forall (t_1, t_2) \in \mathbb N^2$, for $\alpha_t=\frac{t}{t+1}$.

The result follows by simply aggregating the three bounds.

\end{proof}

\begin{remark}
\label{rem:comparison_forth_moments}
We observe that 
\begin{align}
\E\left[\left\lVert \xc - \mmc  \right\rVert^4 \right] 
& = n_c \kappa \sigma^4 + n_c (n_c -1) \sigma^4,\\
\E\left[\left\lVert \xc - \Pc \xc  \right\rVert^4 \right] 
& = \left(n_c - 2 + \frac{1}{n_c}\right) \kappa \sigma^4 + (n_c -1)\left(n_c - 2 + \frac{3}{n_c}\right) \sigma^4,\\    
\E\left[\left\lVert \Pc \xc - \mmc  \right\rVert^4 \right] 
& = \frac{1}{n_c} \kappa \sigma^4 + 3\frac{n_c -1}{n_c} \sigma^4,. 
\end{align}
where $\kappa$ is the kurtosis index (and then $\kappa \sigma^4$ is the fourth moment). Then, for $n_c \leqslant 2$
\begin{align}
\label{e:relation_fourth_moments}
\E\left[\left\lVert \Pc \xc - \mmc  \right\rVert^4 \right]
\leqslant \frac{3}{n_c^2} \E\left[\left\lVert \xc - \Pc \xc \right\rVert^4 \right] \leqslant \frac{3}{n_c^2} \E\left[\left\lVert \xc - \mmc \right\rVert^4 \right], 
\end{align}
showing the advantage of averaging the estimates across all agents in the same connected components.
\end{remark}

 \subsection{ $(\varepsilon,\delta)$-Bounds: Proof of Theorem~\ref{thm:consensus_epsilon_delta}} 

We prove this theorem whose scope is larger.
\begin{theorem}
\label{thm:consensus_epsilon_delta_long}
%Consider a connected component $\mathcal{CC}_a$ and pick uniformly at random an agent $a'$ in it, then
Consider a graph component $c$ and pick uniformly at random an agent $a$ in $c$, then 
\[
		\mathbb{P}\left(\forall t>\tau^C_a, |\hat \mu_a^t-\mu_a |<\varepsilon \right)\ge 1 -\delta
		\]
  where 
  \[\tau_a^C = \max\left\{ \zeta_D, C' \frac{ \E\left[\lVert \x -  \mm \rVert^4 \right]}{n_c \varepsilon^4 \delta} \left( \frac{(1-\alpha)^2}{(1-\alpha \lambda_2(W))^2} \left(1 + \frac{1}{\ln{\frac{1}{\alpha \lambda_2(W)}}} \right)^2 + \frac{1}{n_c^2}  \left(1 + \frac{1}{\ln{\frac{1}{\alpha }}} \right)^2\right)\right\}\] 
  for $\alpha_t= \alpha$, and 
  \[\tau_a^C \triangleq  \max\left\{ \zeta_D,C'' \frac{\E\left[\lVert \Pc \xc -  \mmc \rVert^4 \right]}{n_c \varepsilon^4 \delta} \ln^2\left(e C'' \frac{\E\left[\lVert \Pc \xc -  \mmc \rVert^4 \right]}{n_c \varepsilon^4 \delta}\right)\right\}.\] 
  for $\alpha_t= t/(t+1)$.
\end{theorem}
\begin{proof}
We start considering the auxiliary system studied in the previous sections: consensus matrices can be arbitrary until time $\zeta_D$ and then agents acquire perfect knowledge about which neighbors belong to the same class and simply rely on information arriving through these links.
    \begin{align}
       \Prob\left( |\hat{\mu}_a^{t+1} - \mu_a|   \ge \varepsilon\right) 
            & = \Prob\left( \left(\hat{\mu}_a^{t+1} - \mu_a\right)^4   \ge \varepsilon^4\right) \\
            & \leqslant \frac{\E\left[  \left(\hat{\mu}_a^{t+1} - \mu_a\right)^4 \right]}{\varepsilon^4}\\
            & = \frac{\frac{1}{n_c}\sum_{a'=1}^{n_c}\E\left[  \left(\hat{\mu}_a^{t+1} - \mu_a\right)^4 \right]}{\varepsilon^4}\\
            & = \frac{\E\left[  \left\lVert{\yc}^{t+1} - \mmc\right\rVert^4 \right]}{n_c \varepsilon^4}
    \end{align}
Applying the union bound, we obtain:
\begin{align}
       \Prob\left( \exists t \ge t' |   |\hat{\mu}_a^{t+1} - \mu_a|   \ge \varepsilon\right) 
            & \le \sum_{t=t'}^\infty \frac{\E\left[  \left\lVert{\yc}^{t+1} - \mmc\right\rVert^4 \right]}{n_c \varepsilon^4}.
\end{align}

When $\alpha_t= \alpha$, considering the dominant term in~\eqref{e:fourth_moment_alpha_const} and~\eqref{e:relation_fourth_moments} leads to 
\begin{align}
    \E\left[  \left\lVert{\yc}^{t+1} - \mmc\right\rVert^4 \right] & \leqslant \frac{C'}{2} \frac{ \E\left[\lVert \xc -  \mmc \rVert^4 \right]}{(t+1)^2} \left( \frac{(1-\alpha)^2}{(1-\alpha \lambda_2(W))^2} \left(1 + \frac{1}{\ln{\frac{1}{\alpha \lambda_2(W)}}} \right)^2 + \frac{1}{n_c^2}  \left(1 + \frac{1}{\ln{\frac{1}{\alpha }}} \right)^2\right)
\end{align}
We observe that 
\begin{align}
    \sum_{t=t'}^\infty \frac{1}{(t+1)^2} & \leqslant \int_{t'}^{\infty} \frac{1}{t^2}\diff t = \frac{1}{t'},
\end{align}
from which we conclude:
\begin{align}
\Prob\left( \exists t \ge t' :    |\hat{\mu}_a^{t+1} - \mu_a|   \ge \varepsilon\right) 
            & \le \frac{C'}{2} \frac{ \E\left[\lVert \x -  \mm \rVert^4 \right]}{n_c \varepsilon^4 t} \left( \frac{(1-\alpha)^2}{(1-\alpha \lambda_2(W))^2} \left(1 + \frac{1}{\ln{\frac{1}{\alpha \lambda_2(W)}}} \right)^2 + \frac{1}{n_c^2}  \left(1 + \frac{1}{\ln{\frac{1}{\alpha }}} \right)^2\right).
\end{align}
This probability is then smaller than $\delta/2$ for 
\begin{align}
    t'\geqslant  \tau_a^C = \max\left\{ \zeta_D, C' \frac{ \E\left[\lVert \x -  \mm \rVert^4 \right]}{n_c \varepsilon^4 \delta} \left( \frac{(1-\alpha)^2}{(1-\alpha \lambda_2(W))^2} \left(1 + \frac{1}{\ln{\frac{1}{\alpha \lambda_2(W)}}} \right)^2 + \frac{1}{n_c^2}  \left(1 + \frac{1}{\ln{\frac{1}{\alpha }}} \right)^2\right)\right\}.
\end{align}

Similarly, when $\alpha_t=\frac{t}{t+1}$, considering the dominant term in \eqref{e:fourth_moment_alpha_t} leads to
\begin{align}
    \E\left[  \left\lVert{\yc}^{t+1} - \mmc\right\rVert^4 \right] & \leqslant \frac{C''}{16}\E\left[\lVert \Pc \xc -  \mmc \rVert^4 \right] \left(\frac{\ln (1 + t)}{ 1 + t }\right)^2 .
\end{align}

We observe that 
\begin{align}
    \sum_{t=t'}^\infty \left(\frac{\ln(1+t)}{t+1}\right)^2 & \leqslant \int_{t'}^{\infty} \left(\frac{\ln t}{t}\right)^2 \diff t \\
    & = \int_{\ln t'}^{\infty} x^2 e^{-x} \diff x  \\
    & = \left(e^{-x}\left(x^2 + 2 x + 2\right)\right) \Big\rvert_{\infty}^{\ln t'} \\
    & = \frac{(\ln t')^2 + 2 \ln t' + 2}{t'},\\
    & = \frac{(\ln t' + 1)^2 +1}{t'},\\
    & \leqslant 2 \frac{(\ln t' + 1)^2 }{t'},\\
    & = 2 \frac{\ln^2{(et')} }{t'}
\end{align}
from which we conclude:
\begin{align}
\Prob\left( \exists t \ge t' :    |\hat{\mu}_a^{t+1} - \mu_a|   \ge \varepsilon\right) & \leqslant \frac{C''}{8} \frac{\E\left[\lVert \Pc \xc -  \mmc \rVert^4 \right]}{n_c \varepsilon^4} \frac{\ln^2{(et')} }{t'}.
% \\
%     & \leqslant {2 C''} \frac{\E\left[\lVert \Pc \xc -  \mmc \rVert^4 \right]}{n_c \varepsilon^4 (t')^{\frac{1}{1+\eta}}}.
\end{align}
This probability is then smaller than $\delta/2$ for
\begin{align}
    t'\geqslant \frac{C''}{4} \frac{\E\left[\lVert \Pc \xc -  \mmc \rVert^4 \right]}{n_c \varepsilon^4 \delta} \ln^2{(et')},
\end{align}
and by applying Lemma~\ref{l:nasty_inequality} with $x = t' e$, we obtain that a sufficient condition is 
\begin{align}
    t'\geqslant C'' \frac{\E\left[\lVert \Pc \xc -  \mmc \rVert^4 \right]}{n_c \varepsilon^4 \delta} \ln^2\left(e C'' \frac{\E\left[\lVert \Pc \xc -  \mmc \rVert^4 \right]}{n_c \varepsilon^4 \delta}\right).
\end{align}
Let then define 
\begin{align}
    \tau_a^C \triangleq  \max\left\{ \zeta_D,C'' \frac{\E\left[\lVert \Pc \xc -  \mmc \rVert^4 \right]}{n_c \varepsilon^4 \delta} \ln^2\left(e C'' \frac{\E\left[\lVert \Pc \xc -  \mmc \rVert^4 \right]}{n_c \varepsilon^4 \delta}\right)\right\}.
\end{align}
Finally, let us consider the system of interest. With probability $1- \delta/2$ the agents will have identified the correct links by time $\zeta_D$. The corresponding trajectories coincide with trajectories of the auxiliary system we studied. For the auxiliary system, the estimates have the required precision after time $\tau_a^C$ with probability $1-\delta/2$. It follows that the probability that the estimates in the stochastic have not the required precision after time $\tau_a^C$ is at most $\delta$.
\end{proof}

\newpage

\section{ Appendix G - On the  Structure of $G(N,r)$ and its Impact on Performance of our Algorithms}\label{app:Gnr}

\subsection{ On the Local Tree  Structure of $G(N,r)$ }
Let $\mathcal{N}= G(\mathcal{V},\mathcal{E})$ be   a network sampled from the class of 
random regular graphs $G(N,r)$ with fixed degree $r$. \footnote{observe that graphs in $G(N,r)$ are not necessarily simple} Here and in the following we fix $\mathcal{V}=\mathcal{A}$ and  $n=|\mathcal{V}|=|\mathcal{A}|$.
We are interested to investigate  the structure of the 
$d$-deep neighborhood, $\mathcal{N}^{(v)}_d$, of a generic node $v$ (i.e. the sub-graph inter-connecting nodes at a distance smaller or equal than $d$v from $v$).
Observe that as $n$ grows large, for any $d\in \mathbb{N}$, we should expect that 
$\mathcal{N}^{(v)}_d$ is likely  equal to $ \mathcal{T}_d$, a perfectly balanced   
tree of depth $d$,  in which the root has $r$ children, and all the other nodes have $r-1$. children.

To this end  using an approach inspired by Lemma 5 in  \citep{como} (which applies to directed graphs), we can claim that:
\begin{theorem} \label{tree-like-Gnr}
Whenever $v_0$ is  chosen uniformly at random,  we have 
\[
\mathbb{P}(\mathcal{N}^{(v_0)}_d\neq \mathcal{T}_d)\le \frac{(H+1)H}{2N}
%= 1- \prod_{h=1}^{H}} \left ( 1 - \frac{hr-(2h-1)}{nr-2(h-1)}\right)
\]
where $H= 1+\sum_{d'=1}^d r(r-1)^{d'-1}$.
\end{theorem}
%%{\it P{roof}}
\begin{proof}
First observe  that  realizations of  $G(N,r)$   graphs  are typically obtained through the following  standard procedure:
every node is initially connected to $r\ge 2$ stubs; then stubs then are sequentially  
randomly matched/paired to form edges, as follows: at each stage, an arbitrarily selected free/unpaired stub is selected and paired  with a different stub 
picked uniformly at random among the still unpaired stubs.

We identify every stub with different number in $[1, r N]$ (we assume $r N$ to be even). 
Now, let $\nu(i)$ for $i \in [1, r N]$ be  the function that returns the identity of the node to which the stub is
connected.  See Fig.~\ref{fig:gnr_fig1}

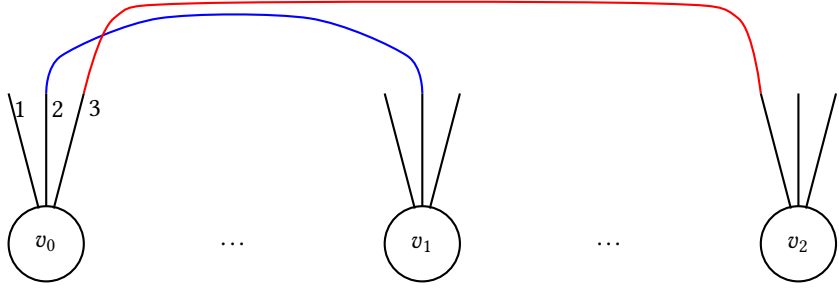
\begin{figure}[h] \setlength{\unitlength}{0.1 cm} % selecting unit length \centering      % used for centering Figure
 \begin{center}
		\begin{tikzpicture}
			
                \draw[thick](  0,0) circle(0.5);
			\draw[thick]((5,0) circle(0.5);
			\draw[thick]((10 ,0) circle(0.5);
                 \node (a) at(0, 0){$v_0$};
                 \node (b) at(5, 0){$v_1$};
                \node (b) at(10, 0){$v_2$};

                 \draw [blue,thick] plot [smooth] coordinates {(0,2) (0.2, 2.5) (1.5,3) (3.5,3) (4.8,2.5) (5,2)};
                 \draw [red,thick] plot [smooth] coordinates {(0.5, 2) (0.9,  3)  (2, 3.2) (8.0, 3.2) (9.2, 3) (9.5, 2)};
                 
                \draw[-,thick](0, 0.5)--(0,2);
                \draw[-,thick](0.1, 0.47)--(0.5,2);
                \draw[-,thick](-0.1, 0.47)--(-0.5, 2);
              
                 \node  at(0.15, 1.8){$2$};  
                 \node  at(0.65, 1.8){$3$};  
                 \node  at(-0.35, 1.8){$1$};  
       
                \draw[-,thick](5, 0.5)--(5,2);
                \draw[-,thick](5.1, 0.47)--(5.5,2);
                \draw[-,thick]( 4.9, 0.47)--(4.5, 2);
             
                \draw[-,thick](10, 0.5)--(10,2);
                \draw[-,thick](10.1, 0.47)--(10.5,2);
                \draw[-,thick]( 9.9, 0.47)--(9.5, 2);

                    \node at (2.5, 0) {\ldots};
                     \node at (7.5, 0) {\ldots};
		\end{tikzpicture}
            \end{center}
		\caption{Random  matching of stubs, $r=3$}
            \label{fig:gnr_fig1}
	\end{figure} 

Our procedure explores the $d$-neighborhood  of a node $v_0$, taken at random, by sequentially unveiling stub-pairings,  according to a bread-first approach.  At every step, the procedure checks whether the already explored portion of  $\mathcal{N}^{(v_0)}_d$, $G(\mathcal{V}, \mathcal{E})$,  has a tree structure.

The procedure initializes $\mathcal{V}^{(0)}=v_0$ and $\mathcal{E}^{(0)}= \emptyset$.

At step 1, our procedure takes a stub $k_1$ connected to $v_0$ (i.e. such that $\nu(k_1)=v_0$) and matches it, uniformly at random with another stub $r(k_1)\neq k_1$.  Let  $v_1:=\nu(r(k_1))$.
Then the procedure updates sets   $\mathcal{V}$ and $\mathcal{E}$ according to the rule:
$\mathcal{V}^{(1)}=\mathcal{V}^{(0)}\cup \{v_1\}$ and 
$\mathcal{E}^{(1)}= \mathcal{E}^{(0)}\cup \{(v_0,v_1)\}$.
At this stage  $G(\mathcal{V}^{(1)},\mathcal{E}^{(1)})$    is  a
tree  only provided that  $v_1\not \in \mathcal{V}^{(0)}$.
This happens with a probability:
\[
p_1:=\mathbb{P}(v_1:=\nu(r(k_1))\not \in  \{v_0\} )=1- \frac{r-1}{Nr-1}.
\]
In case    $v_1\neq v_0$   the algorithm  proceeds, otherwise it prematurely  terminates providing $G(\mathcal{V}^{(1)},\mathcal{E}^{(1)})$. 
In step 2,  our procedure takes a new unmatched stub $k_2$   (connected again to $v_0$).
$k_2$ is matched uniformly at random with another free (i.e. still unmatched) stub $r(k_2)\not \in  \{k_1, r(k_1), k_2\}$,  let  $v_2:=\nu(r(k_2))$.
Then  sets   $\mathcal{V}$ and $\mathcal{E}$ are updated:
$\mathcal{V}^{(2)}=\mathcal{V}^{(1)}\cup \{v_2\}$ and 
$\mathcal{E}^{(2)}= \mathcal{E}^{(2)}\cup \{(v_0,v_2)\}$. 
$G(\mathcal{V}^{(2)},\mathcal{E}^{(2)})$   is  a tree  
only  under the condition that $v_2:=\nu(r(k_2))\not \in \mathcal{V}^{(1)}$. This happens with a probability:
\[
p_2:=\mathbb{P}(v_2 :=\nu(r(k_2))\not \in \{v_0,v_1\}\mid v_0\neq v_1 )= 1- \frac{2r-2}{Nr-3}
\]
again if $G(\mathcal{V}^{(2)},\mathcal{E}^{(2)})$   is a tree the algorithm proceeds, otherwise it prematurely terminates,  providing $G(\mathcal{V}^{(2)},\mathcal{E}^{(2)})$. 
At a generic step $h$, our procedure takes a free stub $k_h$, connected to  vertex 
$v_{\lfloor (h-1)/r \rfloor } \in \{v_0, v_1, \cdots , v_{h-1} \}$,  (recall that we explore nodes/edges  according to a breadth-first  approach),
and matches it with a randomly chosen (still unmatched) stub $r(k_h)$.   Let $v_h=r(k_h)$.  Then 
Then  sets   $\mathcal{V}$ and $\mathcal{E}$ are updated as follows:
$\mathcal{V}^{(h)}=\mathcal{V}^{(h-1)}\cup \{v_h\}$ and 
$\mathcal{E}^{(h)}= \mathcal{E}^{(h-1)}\cup \{(v_\lfloor (h-1)/r \rfloor   ,v_h)\}$. 
Again  $G(\mathcal{V}^{(h)},\mathcal{E}^{(g)})$   is a tree only if  
$v_h \not \in  \mathcal{V}^{(h-1)}$, and this happens with a probability  
\begin{align*}
p_h:=& \mathbb{P}(\nu(r(k_h))\not \in \{v_0, v_1, \cdots, v_{h-1} \}\\ & \mid \{v_0=v_1=, \cdots,=v_{h-1}\})\\
= &1- \frac{hr-(2h-1)}{Nr-2(h-1)},
\end{align*}
in such a case the algorithm proceeds, otherwise it prematurely terminates, providing $G(\mathcal{V}^{(h)},\mathcal{E}^{(h)})$ in output.

The algorithm naturally  terminates (providing a tree) when all the nodes in   $\mathcal{N}^{(v_0)}_d$ have been unveiled  (i.e., placed in $\mathcal{V}$) and the corresponding
unveiled graph $G(\mathcal{V},\mathcal{E} )$  is  a tree.
This happens at step $H$.
The probability  that the algorithm terminates providing  a tree is given by: 
\begin{align*}
\mathbb{P}(\mathcal{N}^{(V)}_d\neq \mathcal{T}_d) =& 1-\prod_{h=1}^{H} p_h\le  \sum_h (1-p_h)\\
= & \sum_{h=1}^{H}\frac{hr-(2h-1)}{Nr-2(h-1)}\\
\le & \sum_{h=1}^{H} \frac{h}{N}= \frac{1}{N} \sum_1^{H} h=\frac{(H+1)H} {2N} \\ 
\end{align*} %%{\it $\square$}
\end{proof}
%Note that   Lemma 5 in  \citepp{como} applies to more general class of   configuration model  graphs, with respect to which the class of  directed  random regular graph $G(n,r)$  constitutes a particular case.

Now denoting with $M$ the number of vertices $v\in \mathcal V$ for which   ${\mathcal{N}^{(v)}_d\neq \mathcal{T}_d}$, we have that 
\[
M= \sum_{v\in \mathcal{V}} \bm{1}_{ \{ \mathcal{N}^{(v)}_d\neq \mathcal{T}_d \} }
\]
and therefore 
\begin{align*}
\mathbb{E} [ M] &= \sum_{v\in \mathcal{V}}  \mathbb{E} [\bm{1}_{ \{ \mathcal{N}{(v)}_d\neq \mathcal{T}_d \} }]
=  \sum_{v\in \mathcal{V}} \mathbb{P}(\mathcal{N}^{(v)}_d\neq \mathcal{T}_d)\\
& %=n\frac{\sum_{v\in \mathcal{V}} \mathbb{P}(\mathcal{N}^{(v)}_d\neq \mathcal{T}_d)}{n}= 
= N \mathbb{P}(\mathcal{N}^{(v_0)}_d\neq \mathcal{T}_d) \le \frac{(H+1)H}{2}
\end{align*}
where $v_0$ is uniformly taken at random.
By assuming ${(H+1)}{H}=o(N)$, and 
 applying Markov inequality we can claim that for any $\varepsilon >0$ arbitrarily slowly:
\begin{equation}\label{tree-asympt}
\mathbb{P}\left(\frac{M}{N}>\varepsilon\right)\downarrow 0 \quad   \forall \varepsilon>0.
\end{equation}
i.e. the fraction of nodes $v$ for which $\mathcal{N}^{(v_0)}_d\neq \mathcal{T}_d$ is negligible with  a probability tending to 1.

At last, we would like to highlight that \eqref{tree-asympt} can be transferred
to the class $G_0(N,r)$ of uniformly chosen {\it simple} regular graphs thanks to: 
\begin{proposition}[\citet{amini2010bootstrap}]\label{amini}
     Any sequence of event $E_n$ occurring with a probability tending to 1 (0) in $G(N,r)$ occurs well in $G_0(N,r)$ with a probability tending to 1 (0).
 \end{proposition}
 Moreover, recalling that by construction $\frac{M}{N}\le 1$, from \eqref{tree-asympt}  we can immediately deduce that  that $\mathbb{E}[M]/N \to 0$ on $G_0((N,r)$.

At last, we would like to mention the following result, from which
Proposition \ref{amini} rather immediately descends.

\begin{theorem} [\cite{janson_2009}] \label{janson}
\[
\liminf_{n\to \infty } \mathbb{P}(G((N,r) \text{ is simple} )>0
\]
\end{theorem}
Observe Theorem \ref{janson} provides a theoretical foundation  
to design a simple algorithm for the generation of 
a graph in class $G_0((N,r)$, based on the  superposition of an acceptance/rejection procedure to the generation of graphs in  $G(N,r)$. More efficient algorithms are, however, well-known in literature \cite{McKay-Wormald}.

\subsection{The Structure of $\mathcal{CC}_a^d$ and  $\mathcal{CC}_a$  }
Now we investigate on the structure of $\mathcal{CC}_a^d$.
We assume that agents $a\in \mathcal{A}$  are partitioned into a finite number $K$ of 
similarity classes.  Agents are assigned to   similarity classes
independently. We indicate with  $p_k$ the probability according to which 
agent $a$ is assigned to class $k$. Note that by construction 
$\sum_{k=1}^K p_k=1$. We denote with $k_a$ the similarity class to which $a$ is assigned.
%\textcolor{blue}{EL-ci serve una notazione per la generica similarity class.}
In this scenario the structure of  $\mathcal{CC}_a^d$  can be rather easily analyzed, it turns out that:
\begin{theorem}  
Conditionally over the event  
$\{ \mathcal{N}^d_a \text{ is a tree}\}$. $\mathcal{CC}_a^d$ has the structure of a  Branching  process originating from a unique ancestor, obeying to the following properties:
\begin{itemize}
\item the number of off-springs of different 
nodes  are independent. 
\item the number of off-spring of  the ancestor  (generation 0 node) is distributed as a $\text{Bin}(r, p_{k_a})$; 
\item while the number of offs-springs of any generation $i$ node  (with $1\le i<d$) is distributed as  $\text{Bin}(r-1, p_{k_a})$. 
\item generation-$d$ nodes have no off-springs.
\end{itemize}
\end{theorem}
\begin{proof}
The proof is rather immediate. Consider $a$ and explore $\mathcal{CC}_a^d$ according to a breath-first exploration process that stops at depth-$d$. 
The number of off-springs of $a$, $O_a$,  by construction, equals the number of nodes in $\mathcal{N}_a$ that belongs to class $k_a$. This number $O_a$ is distributed as:  $O_a\stackrel{L}{=}\text{Bin}(r, p_{k_a})$.
Consider, now,  any other explored node $a'$, at distance $i<d$ from the ancestor, this node, by construction, will have  a unique parent node $p_{a'}$, 
off-springs of $a'$, $O_{a'}$ will be given by all nodes 
in $\mathcal{N}_{a'}\setminus\{p_{a'}\}$ that belong to class $k_a$. 
Their number, $O_{a'}$, is distributed as:  $O_{a'}\stackrel{L}{=}\text{Bin}(r, p_{k_{a'}}).$ See Fig. \ref{fig:gnr_fig2}. When the exploration reaches a $d$-depth node, it stops,
therefore the number of off-springs for every  $d$-depth nodes is zero.
A last, since the sets $\mathcal{N}_{a'}\setminus\{p_{a'}\}$ are disjoint (as immediate consequence of the fact that $\mathcal{N}^d_a$ is a tree), the variables in  $\{ O_{a'}\}_{a'\in\mathcal{CC}_a^d  }$ are independent.
\end{proof}

  %\textcolor{magenta}{EL-poi forse conviene aggiungere una figurella}

	\begin{figure}[h] \setlength{\unitlength}{0.1 cm} % selecting unit length \centering      % used for centering Figure
 \begin{center}
		\begin{tikzpicture}

             \node (A)  at (0,0) [circle, thick, fill=green] {a};
              \node (B)  at (-1,2) [circle, thick, fill=green] {b};
            \node (C)  at ( 0,2) [circle, thick, fill=gray] {c};
             \node (D)  at (1,2) [circle, thick, fill=green] {d};
 %            \node[circle, draw, thick, fill=red ] (B) at (1,2) 
 %            \node[circle, draw, thick, fill=red] (C) at (0,2) 
 %            \node[circle, draw, thick, fill=blue] (D) at (-1,2) 
            \draw(0,0)[thick] (A)-- (B); 
             \draw(0,0)[thick] (A)-- (C);             
	      \draw(0,0)[thick] (A)-- (D); 
              \node (E)  at (-2,4) [circle, thick, fill=green] {e};
              \node (F)  at (-1,4) [circle, thick, fill=green] {f};
            \node (G)  at ( 1,4) [circle, thick, fill=gray] {g};
             \node (H)  at (2,4) [circle, thick, fill=green] {h};

%             \node[circle, draw, thick, fill=blue] (E) at (-1,4) 
%             \node[circle, draw, thick, fill=red ] (F) at (0,4) 
%             \node[circle, draw, thick, fill=red] (G) at (1,4) 
%             \node[circle, draw, thick, fill=blue] (H) at (2,4) 

             \draw(0,0)[thick] (B)-- (E); 
             \draw(0,0)[thick] (B)-- (F) ;            
	          \draw(0,0)[thick] (D)-- (G);
           \draw(0,0)[thick] (D)-- (H);
           \draw(1.5,0.8) node {$O_a=2$};
           \draw (0,1) arc (90:120:1.5);
           \draw (0,1) arc (90:60:1.5);
           \draw(2.4,2.8 ) node {$O_b=1$};
          \draw(-0.2 ,2.8 ) node {$O_d=2$};

                \draw (-1,3) arc (90:120:1.5);
                \draw (-1,3) arc (90:80:1.5);
                \draw (1,3) arc (90:100:1.5);
                \draw (1,3) arc (90:60:1.5);
       \end{tikzpicture}
             \end{center}
		\caption{Structure $|\mathcal{CC}_a^d|$: an example for $r=3$, $d=2$; green nodes belong to $|\mathcal{CC}_a^d|.$ }
            \label{fig:gnr_fig2}
	\end{figure}
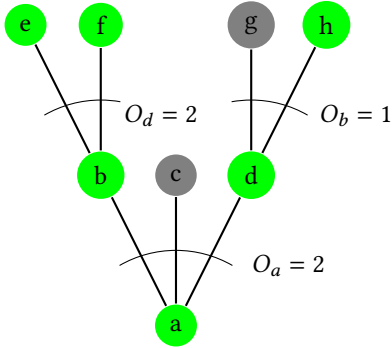

We can now recall and adapt a few standard asymptotic results on  Galton-Watson (GW) Processes    (in a standard GW process  
all nodes in the tree give origin to  number of off-springs, which are identically distributed and independent).
First observe that, in our case,  as $|\mathcal{A}|\to \infty$ we can assume that  $d\to \infty$ as well, for example choosing $d$ as in Proposition \ref{prop:gnr}.
%$d=\lfloor (\frac{1}{2}-\varepsilon )(\log_{r-1}|\mathcal{A}| )\rfloor$ with $0<\varepsilon<1/2$.

We say that a standard GW process is super-critical if the average number of off-springs of every node $\mathbb{E}[O_a]:=m>1$.
Now for a  standard supercritical GW process, denoted with  $Z_i$ the number 
of  nodes belonging to generation $i$, we have:
\begin{theorem}[extinction-explosion principle]
For every super-critical nontrivial~\footnote{ a GW process is non trivial if $\mathbb{P}(O_a=0)>0$}  GW process
$\{Z_i\}_i$   is bound to either extinction or explosion, i.e.,
\[
\mathbb{P}\left(Z_i=0 \text{ eventually}\right)+ \mathbb{P}\left(\lim \frac{Z_i}{m^i}>0\right)=1.
\]
\end{theorem}
The previous result can immediately extended to our two-stage branching process as $d\to \infty$. Denoted with $m_0= rp_{k_a} $ and 
$m:=(r-1)p_{k_a}$, under the assumption that $m>1$ we obtain that:
\begin{equation} \label{exploesti}
\mathbb{P}\left(Z_i=0 \text{ eventually}\right)+ \mathbb{P}\left(\lim \frac{Z_i}{m_0m^{i-1}}>0\right)=1.
\end{equation}

Indeed the distribution of off-springs at the root has no impact on the structural properties of the process.

Now we focus on the extinction probability 
$q_{2BP}:=\mathbb{P}\left(Z_i=0 \text{ eventually}\right)$  in a two-stages branching process, as our process.
To compute it, we can adapt classical results on GW. It turns out that:
\begin{theorem}\label{th:extinction}
Consider a two-stage branching process, in which the 
number of offspring of the root is $\text{Bin}(p_k,r)$ while the off-spring of every other node  $\text{Bin}(p_k,r-1)$.
Its extinction probability $ q_{2BP}$
is given by 
\begin{equation*}
    q_{2BP}= \sum_{h=0}^r \binom{r}{h} p_k^h (1-p_k)^{r-h} q_{GW}^h=
    [(1-p_k)+ p_k q_{GW}]^r
\end{equation*}
where $q_{GW}$ is the extinction probability of a standard GW, with distribution of off-springs given by $\text{Bin}(r-1,p_k)$.   
$q_{GW}$ can be easily computed as the only solution in $(0,1)$
of equation: 
\[
t = [(1- p_k) +p_k t]^{r-1}
\]
\end{theorem}
\begin{proof} 
The proof is immediate, considering that: i) every sub-tree originated by an offspring of the ancestor has the same structure of a
 standard GW. The event $\{Z_i=0\}$   is equivalent to event  $\{ \text{every sub-tree originated by every offspring of the ancestor}\\ \text{ is extincted within $i-1$ generations} \}$. Then conditioning on the number of off-springs of the ancestor we get the claim. 
 Of course, previously computed asymptotic extinction probability $ q_{2BP}$ provides an upper bound to the probability of early extinction of 
 a  $d$-depth truncated two-stage Branching process.
 \end{proof}

\begin{table}
\begin{center}
\begin{tabular}{l|l|l|l}
\toprule
     & $p_{k_a}=1/2$  & $p_{k_a}=1/4$  & $p_{k_a}=1/8$ \\
     \midrule
    $r=4$ &    $0.146$     &      $1$         &      $1$       \\    
    $r=8$ &    $4.17\cdot 10^{-3}$          &    $0.176$           & $1$            \\
    $r=16$ &   $1.52 \cdot 10^{-5}$          &    $1.08\cdot 10^{-2}$      &    $0.190$                \\
    $r=32$ &   $2.32 \cdot 10^{-10}$         &  $1.01  \cdot 10^{-4}$ &       $1.51   \cdot 10^{-2}$                 \\
\bottomrule
\end{tabular}
\end{center}
\caption{extinction probability, $q_{2BP}$, for different values of $r$ and $p_{k_a}$.}
\label{tab:extinction_prob}
\end{table}

 At last observe that, choosing $d$ as in Proposition \ref{prop:gnr}, we have 
 $m_0m^{d-1}= \widetilde \Theta( |\mathcal{A}|^{\frac{1}{2}(1+\log_{r-1} p_{k_a})}) $,  therefore, 
 recalling that by construction $|\mathcal{CC}_a^d|> Z_d$,  by \eqref{exploesti},
 we can always  select $f( |\mathcal{A}| )$ that satisfies jointly  $f( |\mathcal{A}| )= 
 o(m_0m^{d-1}) $ and $f( |\mathcal{A}| )= \widetilde \Theta( |\mathcal{A}|^{\frac{1}{2}(1+\log_{r-1} p_{k_a}) } )$, such that: 
\[
\lim_{|\mathcal{A}|\to \infty} \mathbb{P}(|\mathcal{CC}_a^d|>f(n) )=1-q_{2BP}.
\]
Moreover, note that the extinction probability $q_{2BP}$ is actually a function of its two parameters $(r,p_{k_a})$,  its is rather immediate to show that
\[
\lim_{r\to \infty} q_{2BP}(r,p_{k_a})=0   \qquad \forall p_{k_a}>0
\]
Therefore choosing $r$ sufficiently large we can make $q_{2BP}(r,p_{k_a})$ arbitrarily small
and at the same time guarantee $|\mathcal{A}|^{\frac{1}{2}(1+\log_{r-1} p_{k_a})}> |\mathcal{A}|^{ \frac{1}{2}-\phi }$ for an arbitrarily small  $\phi>0$.

At last observe that if we turn our attention to $\mathcal{CC}_a$, since by construction we have  $\mathcal{CC}_a^d\subseteq  \mathcal{CC}_a$ $\forall d,a$,  rather immediately we have:
\[
\lim_{|\mathcal{A}|\to \infty} \mathbb{P}(|\mathcal{CC}_a|>g(n) )=1-q_{2BP}.
\]
for any $g(n)=o(|\mathcal{A}|^{\frac{1}{2}(1+\log_r p_{k_a}) })$.

\newpage

\section{Appendix H - Proof of the results in Table~\ref{t:comparison}}
\label{app:order_sense}

The results in Table~\ref{t:comparison} follow immediately, once we derive the asymptotics for $n^\star_\gamma(x)$   in the sub-Gaussian setting and in bounded 4-th moment setting. The asymptotics for $\tilde{n}_{\gamma}(x)$ are immediate to derive from the bounds in Theorem~\ref{thm:Bcolme_thm2}.

\subsection{$n^\star_\gamma(x)$, sub-Gaussian setting}
Remember that 
 $n^\star_\gamma(x) \coloneqq \lfloor \beta^{-1}_\gamma(x) \rfloor$, i.e., $n^\star_\gamma(x)$ is the smallest integer $n$ such that
 \[\beta_\gamma(n):= \sigma \sqrt{\frac{2}{n}\left(1+ \frac{1}{n}\right)\ln(\sqrt{(n+1)}/\gamma )}\le x.\]
As we are interested in upper bounds for $n^\star_\gamma(x)$, we can start by upperbounding the left-hand side expression.
\begin{align}
     \sigma \sqrt{\frac{2}{n}\left(1+ \frac{1}{n}\right)\ln\left(\frac{\sqrt{(n+1)}}{\gamma} \right)}
     & \le \sigma \sqrt{\frac{4}{n}\ln\left(\frac{\sqrt{2 n}}{\gamma} \right)}
     = \sigma \sqrt{\frac{2}{n}\ln\left(\frac{2n}{ \gamma^2} \right)},
\end{align}
and imposing that the right-hand side is smaller than $x$, we obtain:
\begin{align}
    \sigma^2 \frac{2}{n}\ln\left(\frac{2n}{ \gamma^2} \right) & \le x^2\\
    \frac{4 \sigma^2}{\gamma^2 x^2} \ln\left(\frac{2n}{ \gamma^2} \right) & \le \frac{2 n }{\gamma^2}.
\end{align}
From \cite[Lemma A.1]{shalev_understanding} a sufficient condition for this inequality to hold is 
\begin{align}
    \frac{8 \sigma^2}{\gamma^2 x^2} \ln\left(\frac{4 \sigma^2}{\gamma^2 x^2} \right) \le  \frac{2 n }{\gamma^2},
\end{align}
and then 
\begin{align}
    \frac{4 \sigma^2}{ x^2} \ln\left(\frac{4 \sigma^2}{\gamma^2 x^2} \right) \le  n.
\end{align}
We can conclude that 
\begin{align}
    n^\star_\gamma(x) \in \bigO\left(  \frac{\sigma^2}{ x^2} \ln\left(\frac{ \sigma}{\gamma x} \right)\right), 
\end{align}
from which the asymptotics for $\zeta_a$ and $\tau_a$ can be derived opportunely replacing $x$ and $\gamma$.

\subsection{$n^\star_\gamma(x)$, Bounded 4-th Moment Setting}
The reasoning is analogous, but we start from 
 \begin{align}
 \beta_\gamma(n)= \left(2\frac{\kappa + 3 \sigma^4}{\gamma}\right)^{\frac{1}{4}} \left(\frac{1 + \ln^2 n}{n}\right)^{\frac{1}{4}}.
 \end{align}
 For $n \ge 3$
 \begin{align}
    \left(2\frac{\kappa + 3 \sigma^4}{\gamma}\right)^{\frac{1}{4}} \left(\frac{1 + \ln^2 n}{n}\right)^{\frac{1}{4}}   
    & \le  \left(2\frac{\kappa + 3 \sigma^4}{\gamma}\right)^{\frac{1}{4}} \left(\frac{2 \ln^2 n}{n}\right)^{\frac{1}{4}}, 
 \end{align}
 and imposing the RHS to be smaller than $x$:
  \begin{align}
    4 \frac{\kappa + 3 \sigma^4}{\gamma} \frac{\ln^2 n}{n} & \le x^4\\
    4 \frac{\kappa + 3 \sigma^4}{\gamma x^4} {\ln^2 n} & \le {n},
 \end{align}
 and from Lemma~\ref{l:nasty_inequality} a sufficient condition for this inequality to hold is
 \begin{align}
     \max\left\{16 \frac{\kappa + 3 \sigma^4}{\gamma x^4} {\ln^2 \left( 16 \frac{\kappa + 3 \sigma^4}{\gamma x^4} \right)}, 1 \right\} & \le {n}.
 \end{align}
We can conclude that 
\begin{align}
    n^\star_\gamma(x) \in \tilde{\bigO}\left( \frac{\kappa + 3 \sigma^4}{\gamma x^4}\right).
\end{align}
\newpage

\section{Appendix I - Additional Performance Evaluation of the \belief{} and \consensus{}}
\label{sec:add_performance}

In this Appendix, we report further results  on the  performance of \belief{} and \consensus{}, as a function of the underlying graph structure and characterizing parameters.

\begin{figure}[h!]
    %\centerline{\includegraphics[width=0.65\textwidth]{figs/updated_figs/comparison_r_5_2.pdf}}
	\centerline{\includegraphics[width=0.65\textwidth]{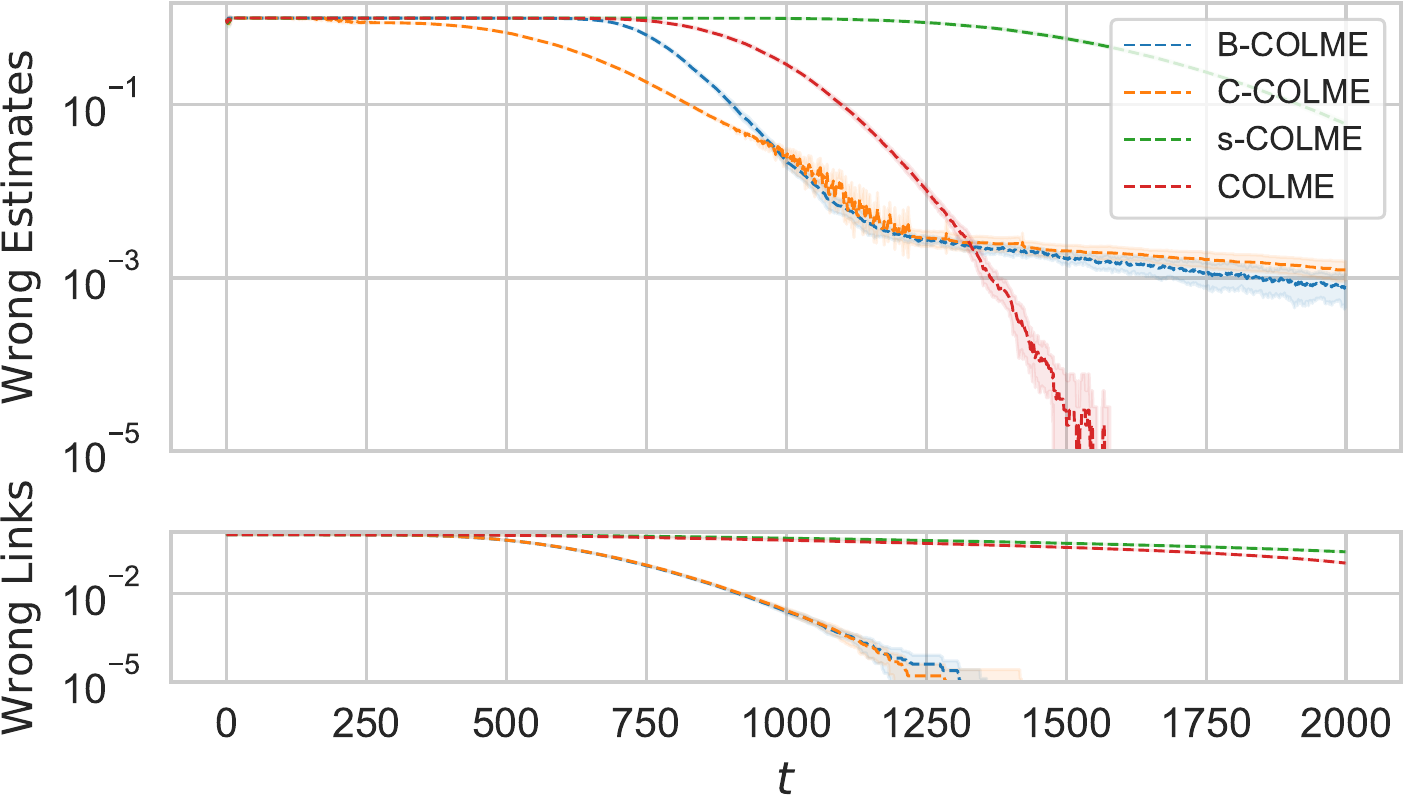}}
    \caption{All approaches compared considering a smaller $r$ compared to Figure 2 ($r=5$), this shows that $r$ needs to be chosen appropriately (large enough).} \label{fig:comparison_r5}
\end{figure}

\subsection{Over a $G_0(N,r)$ Varying $r$}

Here, we explore the performance of the two proposed \textit{scalable} algorithms as a function of the number of neighbors they are allowed to contact during the dynamics, i.e., the parameter $r$ of the $G_0(N,r)$ graph.

\begin{figure}[h!]
    % \centerline{\includegraphics[width=0.75\textwidth]{figs/updated_figs/scalable_as_function_of_r_no_local_new_2.pdf}}
	\centerline{\includegraphics[width=0.7\textwidth]{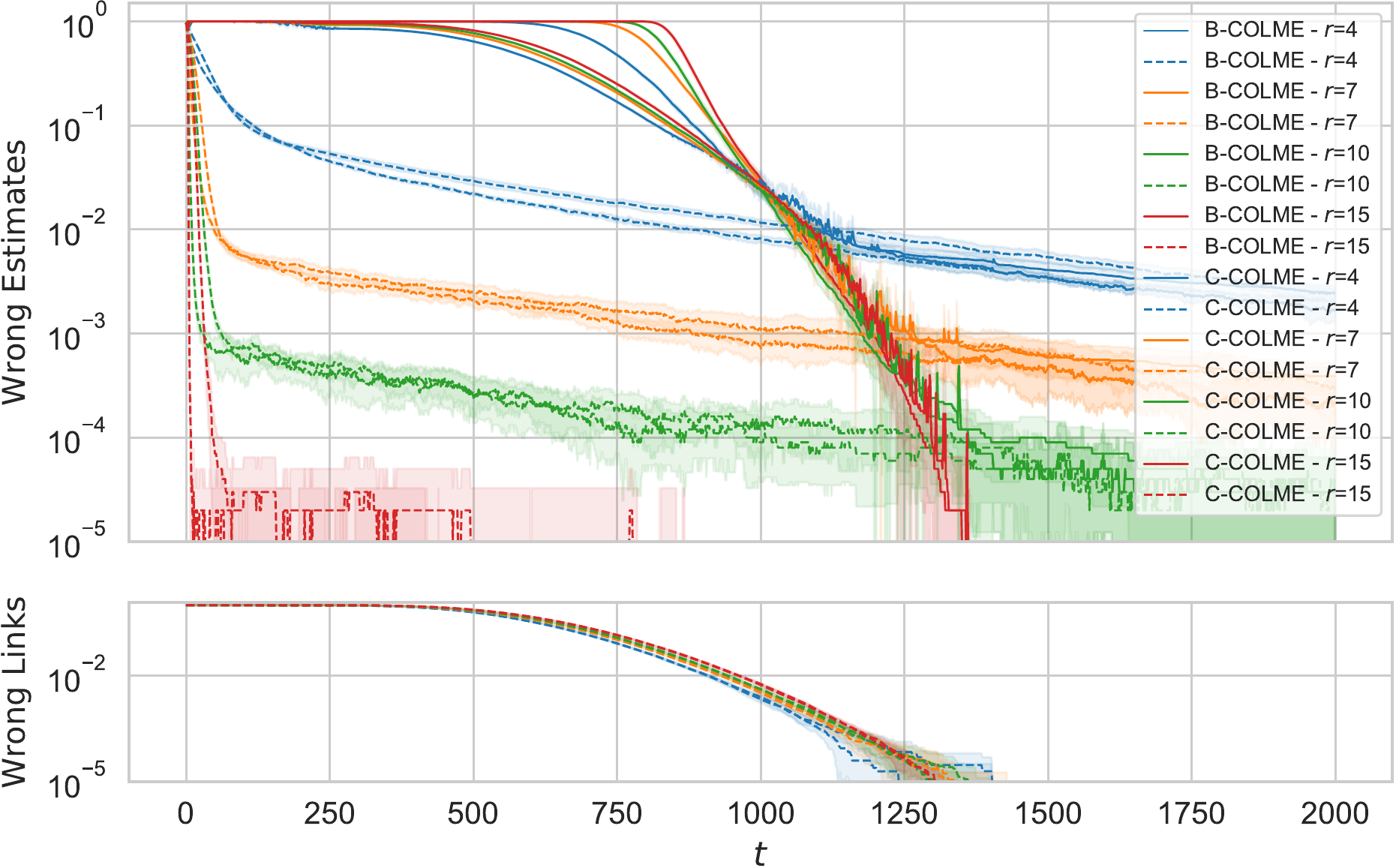}}
    \caption{Performance comparison as a function of the parameter $r$ of $G_0(N,r)$ for \belief{} and \consensus{}. Dashed line is used for the \textit{oracle} while solid line for the algorithm.} \label{fig:all_vary_r}
\end{figure}

% \subsubsection{Over a Watts-Strogatz Graph}
% \vspace{4cm}

% \subsection{Scalable Algorithms Using the Generalized Chebychev Confidence Interval}
% \vspace{4cm}

\subsection{\belief{} as a Function of the Depth of Information Kept}
\label{sec:belief_depth}

In the \belief{} algorithm, many cycles in the graph can degrade the performance of the algorithm (hence the need for a tree-like local structure). Here we study the performance of the algorithm as a function of the depth $d$ of the neighborhood that receives the estimate of a given node, and we find that a high value of this parameter eventually degrades the performance of the algorithm.

\begin{figure}[h!]
    % \centerline{\includegraphics[width=0.7\textwidth]{figs/updated_figs/b_colme_depth_no_local_2.pdf}}
	\centerline{\includegraphics[width=0.7\textwidth]{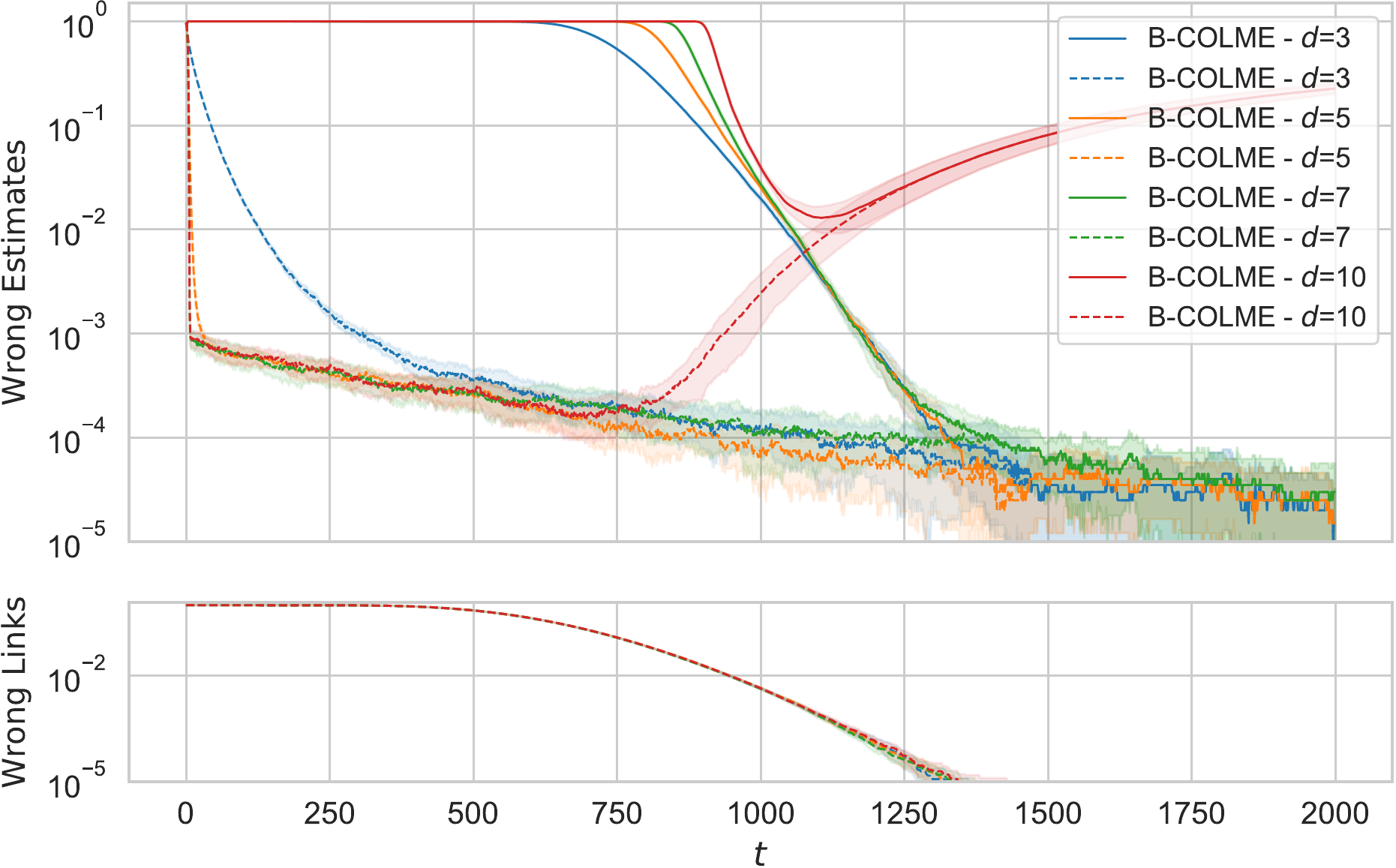}}
    \caption{Performance comparison of \belief{} as a function of the \textit{depth} $\kappa$ of the info kept. Dashed line is used for the \textit{oracle} while solid line for the algorithm.} \label{fig:belief_vary_d}
\end{figure}

\subsection{\consensus{} (Constant $\alpha$) as a Function of the Weight $\alpha$}

We report some experiments considering $\alpha$ constant, as opposed as $\alpha=\frac{t}{t+1}$ (refreshed at each topology modification), and explore the impact of the parameter on the probability of error.

\begin{figure}[h!]
    % \centerline{\includegraphics[width=0.7\textwidth]{figs/updated_figs/c_colme_CONST_alpha_no_local_2.pdf}}
	\centerline{\includegraphics[width=0.7\textwidth]{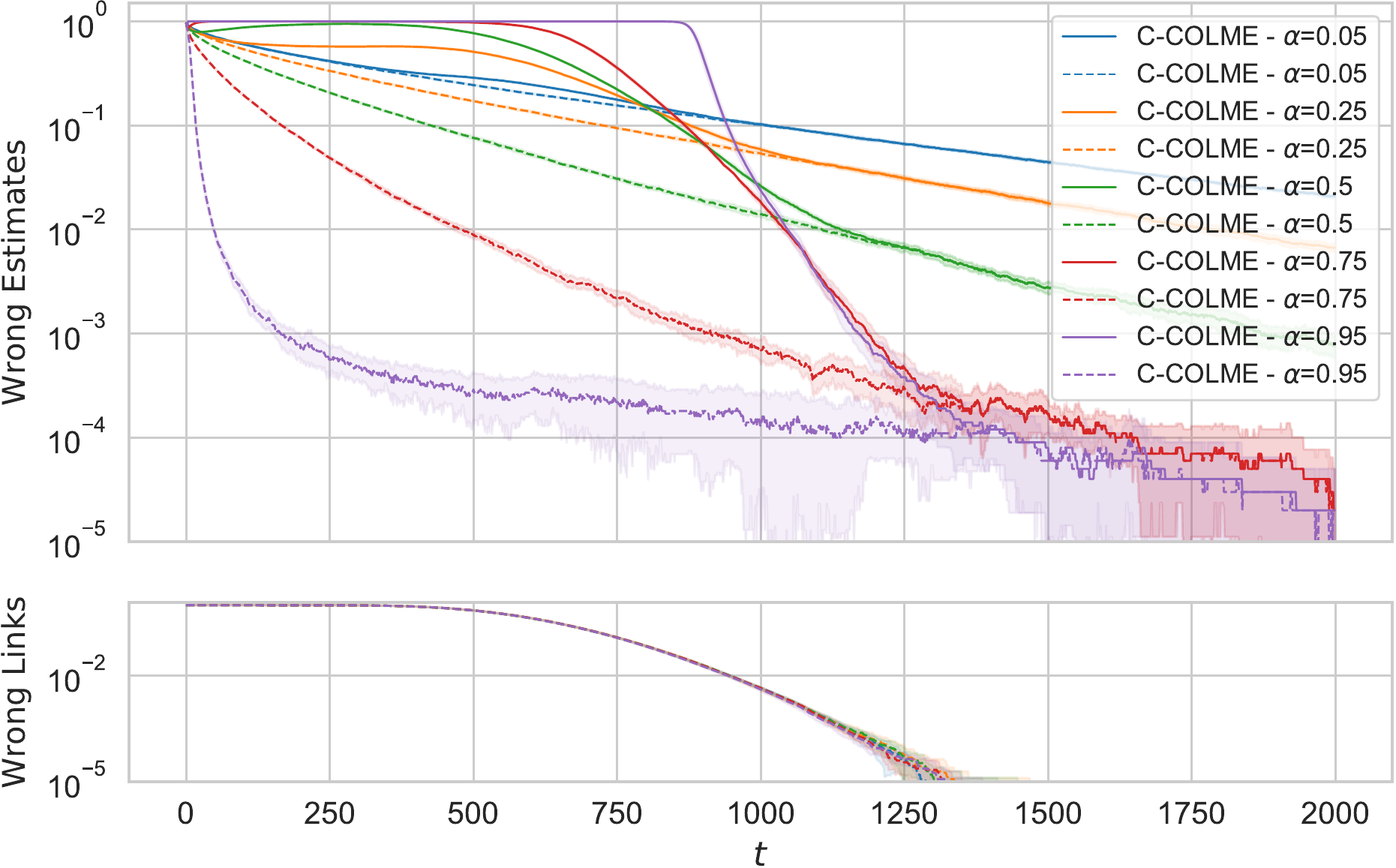}}
    \caption{Performance comparison of \consensus{} as a function of the weight $\alpha$ when it is considered as constant. Dashed line is used for the \textit{oracle} while solid line for the algorithm.} \label{fig:consensus_const_alpha}
\end{figure}

\newpage

\section{Appendix J - Additional Details on the Decentralized Federated Learning Approach FL-DG}
    \label{sec:add_ML_exp}

    \belief{} and \consensus{} are useful in contexts where nodes want to learn the preferences of users by using information from other peers in addition to local data, or when sensors jointly try to estimate certain quantities in a very heterogeneous environment (e.g., smart farming).
    Nevertheless, it is of great interest to apply these techniques in the context of decentralized personalized federated learning, where the goal of the nodes is to learn a machine learning model on a given local dataset $D_a$ while having the possibility to collaborate with other agents (assuming that they can be classified into one of $C$ possible classes, each characterized by a distribution $D_c$). We propose FL-DG, a decentralized FL algorithm inspired by \belief{} and \consensus{}, where agents decide which peers to collaborate with based on the cosine similarity between the weights updates of their models. We compare it to a classical decentralized FL algorithm over a static graph (FL-SG).

    \begin{center}
    \centering
        \begin{minipage}{.55\linewidth}
            \begin{algorithm}[H]
            \centering
            \caption{FL-SG Training}
            \label{alg:ML_training}
            \begin{algorithmic}
            \STATE {\textbf{Input:}} $\,\,\,\, \mathcal{G} = \left(\mathcal{A}, \mathcal{E}\right)$, $\, D_a \in \left\{ D_c \right\}_{c=1}^{C} \forall a \in \mathcal{A}$, $\, \theta_0$
            \STATE {\textbf{Output:}} collaborative FL-SG model $\theta_a, \, \forall a \in \mathcal{A}$
            \STATE $\mathcal{C}_a^0 \leftarrow \mathcal{N}_a^t \; \forall a \in \mathcal{A}$, $\quad \theta^0_a \leftarrow \theta_0$, $\quad \theta^0_l \leftarrow \theta_0 \; \forall l \in \mathcal{E}$
            \WHILE{new sample $s_a^t$ arrives at time $t$} \STATE{\textit{// Training Phase}}
                \FOR{node $a$ in $\mathcal{A}$}
                    \FOR{epoch $e$ in $\{1,..,E\}$}
                        \FOR{minibatch $M_{a}^t$ in $\{M_{a,0}^t, M_{a,1}^t\}$}
                            \STATE $\theta^{t+1}_a \leftarrow \theta_a^t + \text{SGD}(\theta_a^t, D_a^t)$
                            \FOR{neighbor $a'$ in $\mathcal{N}_a \cap \mathcal{C}_a^t$}
                                \STATE $\theta_{a,a'}^{t+1} \leftarrow \theta^{t}_{\{a,a'\}} + \text{SDG}(\theta^t_{\{a,a'\}}, M_a^t)$
                            \ENDFOR
                        \ENDFOR
                    \ENDFOR
                \ENDFOR                      
                \STATE{\textit{// Discovery Phase}}
                \FOR{undirected link $\{ a,a'\}$ in $\mathcal{E}$}
                    \STATE $\Delta \theta_{a,a'}^{t+1} \leftarrow \theta_{a,a'}^{t+1} - \theta_{\{ a,a' \}}^t$
                    \STATE $\Delta \theta_{a',a}^{t+1} \leftarrow \theta_{a',a}^{t+1} - \theta^t_{\{ a,a' \}} $
                    \STATE $\omega_{\{a,a'\}}^{t+1} \leftarrow \frac{1}{t+1} \frac{\langle \Delta \theta_{a,a'}^{t+1}, \Delta \theta_{a',a}^{t+1}\rangle}{||\Delta \theta_{a,a'}^{t+1}||\cdot||\Delta \theta_{a',a}^{t+1}||} + \frac{t}{t+1} \omega_{\{a,a'\}}^{t}$
                    \IF{$\omega_{a,a'}^{t+1} < \varepsilon_1$}
                        \STATE $\mathcal{C}_a^{t+1} \leftarrow \mathcal{C}_a^{t} \setminus a',\,$ and $\,\, \mathcal{C}_{a'}^{t+1} \leftarrow \mathcal{C}_{a'}^{t} \setminus a$
                    \ENDIF
                \ENDFOR
                \STATE{\textit{// Model Updating Phase}}
                \FOR{node $a$ in $\mathcal{A}$}
                    \STATE $\theta_a^{t+1} \leftarrow \frac{1}{|\mathcal{C}_a^{t+1}|+1} \theta_a^{t+1}$ 
                    \FOR{opt neighbor $a'$ in $\mathcal{C}_a^{t+1}$}
                        \STATE $\theta_a^{t+1} \leftarrow \theta_a^{t+1} + \frac{1}{|\mathcal{C}_a^{t+1}|+1} \theta_{a'}^{t+1}$ 
                    \ENDFOR
                \ENDFOR
                \FOR{undirected link $\{a,a'\} $ in $ \mathcal{E}$}
                    \STATE {$\theta_{\{a,a'\}}^{t+1} \leftarrow \frac{\theta_{a,a'}^{t+1} + \theta_{a',a}^{t+1}}{2} \quad$ \textit{// Update Link Model}}
                \ENDFOR
                \STATE $t \leftarrow t + 1$
            \ENDWHILE
            \end{algorithmic}
            \end{algorithm}
        \end{minipage}
    \end{center}

    We focus on the case $C=2$ and use the MNIST dataset \cite{mnist}. To obtain two different distributions from the MNIST dataset, we simply swap two labels (\lq\lq 3\rq\rq\ and \lq\lq 5\rq\rq\ as well as \lq\lq 1\rq\rq\ and \lq\lq 7\rq\rq). Each node has the task of recognizing handwritten digits using its local data and collaborating with neighboring nodes over $\mathcal{G}$, which is a complete graph in our scenario.
    We use a very simple feedforward neural network model for all nodes. It consists of the input layer, a hidden layer with $100$ nodes, and the output layer. We indicate the parameters of the NN as $\theta^t_a$, for agent $a$ at time $t$. 

    Again, we consider an online setting in which agents receive new samples over time. In particular, each agent $a \in \mathcal{A}$ receives a new sample $s_a^t$ at every time instant $t$. Agents are initially assigned a local database of $M_{a,0}^0$ samples (in our example $|M_{a,0}^0|=30$), and with the new samples they construct two overlapping minibatch $M_{a,0}^t$ and $M_{a,1}^t$. We consider a time horizon leading to two non-overlapping minibatch, i.e., $t=|M_{a,0}^0|$. The agents train their model for $E$ epochs ($E=15$) over the two minibatch at each time instant.

    Differently from the \belief{} and \consensus{} mean estimation algorithms we have to modify the \textit{discovery} phase to a large extent since the task of mean estimation and model training are structurally different. In \cite{sattler} it was shown that it is possible to partition the agents in a federated learning framework by using the cosine similarity of the gradient updates (or the \textit{parameters} updates) of the considered agents. Note that cosine similarity values close to $1$ indicate similar models/agents, while lower values indicate increasingly different agents. This can be intuitively understood by observing that two nodes with different data distributions are optimizing different loss functions, and, if we constrain the starting point of the optimization to be the same for both agents, we will observe an increase in the angle between the vectors corresponding to the gradient updates (see Figure 2 in \cite{sattler} for an illustrative example).
    Subject to some regularity assumptions, it is indeed possible to use the cosine similarity of the parameter updates instead of gradients. Let us denote the updates as $\Delta \theta^t = \theta ^{t+1} -\theta ^t$.
    
    To allow nodes to discover their \textit{similar} neighbors, we define a \textit{link} model $\theta^t_{\{a,a'\}}$ (for each unordered pair $(a,a')$, \lq\lq shared\rq\rq\ between the nodes) and a node-link model $\theta^t_{a,a'}$ associated with a certain (ordered) neighbors pair $(a,a')$.
    Thus, every node $a \in \mathcal{A}$ keeps a model for each of its neighbors $a' \in \mathcal{N}_a$, i.e., $\theta^t_{a,a'}$. Then, at each training round, node $a$ retrieves the shared model $\theta^t_{\{a,a'\}}$ and, starting from those parameters, trains the node-link model $\theta_{a,a'}^t$ on its local data.
    
    After all nodes have performed the \textit{training} phase, they compute the \textit{similarity} metric between the models, i.e., the cosine similarity $\omega^t_{a,a'}$, which allows them to determine whether to collaborate with a neighbor or not. We can compute $\omega^t_{a,a'}$ as:
    
    \begin{equation}\label{eq:cos_sim}
        \omega_{\{ a,a' \}} ^ {t} = \frac{\langle \Delta \theta_{a,a'}^{t}, \Delta \theta_{a',a}^{t}\rangle}{||\Delta \theta_{a,a'}^{t}||\cdot||\Delta \theta_{a',a}^{t}||}
    \end{equation}

    This metric is updated at each iteration by making an average with the previous value (see Algorithm \ref{alg:ML_training}). Whenever $\omega^t_{a,a'}$ goes below a certain threshold $\varepsilon_1$, link $\{a,a'\}$ is deemed to be connecting nodes of different classes and is removed from $\mathcal{E}$.

    Lastly, agents update their collaborative models $\theta^t_a$, averaging the parameters of the agents $a'$ in their estimated similarity class $\mathcal{C}^{t}_a$.
    Moreover, all the link models $\theta^{t}_{\{ a,a' \}}$ are updated averaging the two node-link models of the nodes at the ends of the link, i.e, $\theta^{t+1}_{\{ a,a' \}} = \theta^t_{\{ a,a' \}} + \frac{ \Delta \theta_{a,a'}^{t} +  \Delta \theta_{a',a}^{t}}{2} $. A detailed explanation of the model is provided in Algorithm \ref{alg:ML_training}.

    We report the results (Fig \ref{fig:ML-exp} in the main article) obtained over 30 communication rounds comparing the Local model, which uses only the local dataset of each node, FL-SG a decentralized FL approach that averages the model parameters of all the neighbors over a static graph, and our FL-DG approach, again an averaging model where the nodes dynamically remove connections on the basis of the cosine similarity.

    \begin{figure}
      \centering
      \includegraphics[width=0.6\linewidth]{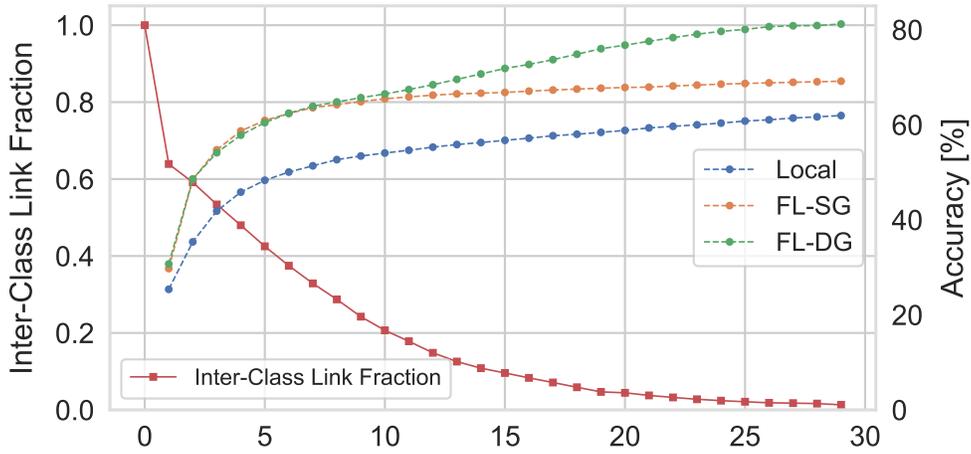}
      \caption{Accuracy of a local model (Local), a decentralized FL over a static graph (FL-SG), and our approach over a dynamic graph (FL-DG). We show also the fraction of links between communities over time for FL-DG. }
      \label{fig:ML-exp_app}
  \end{figure}

\end{document}